\RequirePackage{etoolbox}
\newtoggle{neurips}

\togglefalse{neurips}     %

\documentclass{article}

\newcommand{\neurips}[1]{\iftoggle{neurips}{#1}{}}
\newcommand{\arxiv}[1]{\iftoggle{neurips}{}{#1}}

\neurips{
  \usepackage[nonatbib]{neurips_2026}   %
  \usepackage{natbib}
  \bibpunct{(}{)}{;}{a}{,}{,}
}
  
\arxiv{   
	\usepackage{natbib}
	\bibpunct{(}{)}{;}{a}{,}{,}
}

\arxiv{
	\usepackage[letterpaper, left=1in, right=1in, top=1in,
        bottom=1in]{geometry}
        \usepackage[hypertexnames=false]{hyperref}
        \hypersetup{colorlinks,citecolor=blue!70!black,linkcolor =
          blue!70!black}       %
        \usepackage{parskip}
      }

\arxiv{
  \usepackage[dvipsnames]{xcolor}         %
}
\neurips{
  \usepackage[dvipsnames]{xcolor}
  \usepackage[hypertexnames=false,colorlinks=true,
              citecolor=blue!70!black,linkcolor=blue!70!black,urlcolor=black,breaklinks=true]{hyperref}
}

\usepackage[utf8]{inputenc} %
\usepackage[T1]{fontenc}    %
\usepackage{url}            %
\usepackage{booktabs}       %
\usepackage{amsfonts}       %
\usepackage{nicefrac}       %
\usepackage{microtype}      %
\renewcommand{\paragraph}[1]{\par\noindent\textbf{#1}}
\arxiv{\usepackage{amsthm}}
\neurips{\usepackage{amsthm}}
\usepackage{mathtools} 
\usepackage{amsmath}
\usepackage{bm}
\usepackage{bbm}
\usepackage{amssymb} 
\usepackage{MnSymbol} %
\DeclareSymbolFont{upgreek}{U}{eur}{m}{n}
\DeclareMathSymbol{\uppi}{0}{upgreek}{"19}
\usepackage{makecell}
\usepackage{enumitem}
\usepackage{comment}
\usepackage{wrapfig}
\usepackage{colortbl}
\usepackage{setspace}
\usepackage{transparent}
\usepackage{inconsolata}
\usepackage[scaled=.90]{helvet}
\usepackage{xspace}
\usepackage{verbatim}
\usepackage{titletoc}
\usepackage{algorithm}
\usepackage[noend]{algpseudocode}

\makeatletter
\AtBeginDocument{%
  \let\oldALG@step\ALG@step
  \renewcommand{\ALG@step}{%
    \oldALG@step
    \xdef\ALG@currentHref{ALG@line.\thealgorithm.\arabic{ALG@line}}%
    \global\let\@currentHref\ALG@currentHref
    \xdef\@currentlabel{\arabic{ALG@line}}%
  }%
  \patchcmd{\ALG@doentity}%
    {\noindent\hskip\ALG@tlm}%
    {\noindent\hskip\ALG@tlm\Hy@raisedlink{\hypertarget{\ALG@currentHref}{}}}%
    {}{}%
}
\makeatother

\newcommand{\multiline}[1]{\parbox[t]{\dimexpr\linewidth-\algorithmicindent}{#1}}
\usepackage{tablefootnote}
\usepackage{pifont}
\usepackage[nameinlink,capitalize]{cleveref}
\neurips{\usepackage[nameinlink,capitalize]{cleveref}}

\makeatletter
\newcommand{\neutralize}[1]{\expandafter\let\csname c@#1\endcsname\count@}
\makeatother

\arxiv{
\usepackage{thmtools}
\declaretheoremstyle[
  spaceabove=8pt, %
  spacebelow=8pt, %
  headfont=\bfseries,
  bodyfont=\normalfont\itshape,
  headpunct=., %
  notefont=\normalfont\bfseries, 
  postheadspace=0.5em,
]{myspacingstyle}

\declaretheorem[style=myspacingstyle,name=Theorem,parent=section]{theorem}
\declaretheorem[style=myspacingstyle,name=Lemma,parent=section]{lemma}
\declaretheorem[style=myspacingstyle,name=Corollary,parent=section]{corollary}
\declaretheorem[style=myspacingstyle,name=Claim,parent=section]{claim}
\declaretheorem[style=myspacingstyle,name=Assumption, parent=section]{assumption}
\declaretheorem[style=myspacingstyle,name=Condition, parent=section]{condition}

\declaretheorem[style=myspacingstyle,name=Remark,parent=section]{remark}
\declaretheorem[style=myspacingstyle,name=Proposition, parent=section]{proposition}
\declaretheorem[style=myspacingstyle,name=Fact, parent=section]{fact}
\declaretheorem[style=myspacingstyle,name=Definition, parent=section]{definition}
}

\usepackage{aliascnt}
\neurips{
\theoremstyle{plain}
\newtheorem{theorem}{Theorem}[section]

\newaliascnt{lemma}{theorem}
\newtheorem{lemma}[lemma]{Lemma}
\aliascntresetthe{lemma}

\newaliascnt{corollary}{theorem}
\newtheorem{corollary}[corollary]{Corollary}
\aliascntresetthe{corollary}

\newaliascnt{proposition}{theorem}
\newtheorem{proposition}[proposition]{Proposition}
\aliascntresetthe{proposition}

\newaliascnt{remark}{theorem}
\theoremstyle{remark}

\aliascntresetthe{remark}

\theoremstyle{definition}

\newaliascnt{assumption}{theorem}
\newtheorem{assumption}[assumption]{Assumption}
\aliascntresetthe{assumption}

\newaliascnt{definition}{theorem}
\newtheorem{definition}[definition]{Definition}
\aliascntresetthe{definition}

\newaliascnt{condition}{theorem}
\newtheorem{condition}[condition]{Condition}
\aliascntresetthe{condition}

\newaliascnt{claim}{theorem}

\aliascntresetthe{claim}

\newaliascnt{fact}{theorem}

\aliascntresetthe{fact}
}

\usepackage{crossreftools}
\makeatletter
\renewenvironment{proof}[1][Proof]%
{%
	\par\noindent{\bfseries\upshape {#1.}\ }%
}%
{\qed\newline}
\makeatother

\theoremstyle{plain}

\usepackage{xpatch}
\xpatchcmd{\proof}{\itshape}{\normalfont\proofnameformat}{}{}
\newcommand{\proofnameformat}{\bfseries}

\newcommand{\pref}[1]{\cref{#1}}
\newcommand{\pfref}[1]{Proof of \pref{#1}}

\renewcommand{\eqref}[1]{\texorpdfstring{\hyperref[#1]{(\ref*{#1})}}{(\ref*{#1})}}
\crefformat{equation}{#2Eq.\,(#1)#3}
\Crefformat{equation}{#2Eq.\,(#1)#3}

\Crefformat{figure}{#2Figure~#1#3}
\Crefformat{assumption}{#2Assumption~#1#3}

\Crefname{assumption}{Assumption}{Assumptions}
\crefname{assumption}{Assumption}{Assumptions}
\Crefname{condition}{Condition}{Conditions}
\crefname{condition}{Condition}{Conditions}
\Crefname{claim}{Claim}{Claims}
\crefname{claim}{Claim}{Claims}
\Crefname{fact}{Fact}{Facts}
\crefname{fact}{Fact}{Facts}
\Crefname{definition}{Definition}{Definitions}
\crefname{definition}{Definition}{Definitions}
\Crefname{line}{Line}{Lines}
\crefname{line}{Line}{Lines}

\def\ddefloop#1{\ifx\ddefloop#1\else\ddef{#1}\expandafter\ddefloop\fi}
\def\ddef#1{\expandafter\def\csname bb#1\endcsname{\ensuremath{\mathbb{#1}}}}
\ddefloop ABCDEFGHIJKLMNOPQRSTUVWXYZ\ddefloop
\def\ddefloop#1{\ifx\ddefloop#1\else\ddef{#1}\expandafter\ddefloop\fi}
\def\ddef#1{\expandafter\def\csname b#1\endcsname{\ensuremath{\mathbf{#1}}}}
\ddefloop ABCDEFGHIJKLMNOPQRSTUVWXYZ\ddefloop
\def\ddef#1{\expandafter\def\csname sf#1\endcsname{\ensuremath{\mathsf{#1}}}}
\ddefloop ABCDEFGHIJKLMNOPQRSTUVWXYZ\ddefloop
\def\ddef#1{\expandafter\def\csname c#1\endcsname{\ensuremath{\mathcal{#1}}}}
\ddefloop ABCDEFGHIJKLMNOPQRSTUVWXYZ\ddefloop
\def\ddef#1{\expandafter\def\csname h#1\endcsname{\ensuremath{\widehat{#1}}}}
\ddefloop ABCDEFGHIJKLMNOPQRSTUVWXYZ\ddefloop
\def\ddef#1{\expandafter\def\csname hc#1\endcsname{\ensuremath{\widehat{\mathcal{#1}}}}}
\ddefloop ABCDEFGHIJKLMNOPQRSTUVWXYZ\ddefloop
\def\ddef#1{\expandafter\def\csname t#1\endcsname{\ensuremath{\widetilde{#1}}}}
\ddefloop ABCDEFGHIJKLMNOPQRSTUVWXYZ\ddefloop
\def\ddef#1{\expandafter\def\csname tc#1\endcsname{\ensuremath{\widetilde{\mathcal{#1}}}}}
\ddefloop ABCDEFGHIJKLMNOPQRSTUVWXYZ\ddefloop
\def\ddefloop#1{\ifx\ddefloop#1\else\ddef{#1}\expandafter\ddefloop\fi}
\def\ddef#1{\expandafter\def\csname scr#1\endcsname{\ensuremath{\mathscr{#1}}}}
\ddefloop ABCDEFGHIJKLMNOPQRSTUVWXYZ\ddefloop

\newcommand{\eps}{\epsilon}
\newcommand{\veps}{\varepsilon}

\DeclareMathOperator*{\argmax}{arg\,max}

\def\ddef#1{\expandafter\def\csname b#1\endcsname{\ensuremath{\mb{#1}}}}
\ddefloop abcdeghijklmnopqrstuvwxyz\ddefloop

\newcommand{\ind}[1]{^{\sss{(#1)}}}

\DeclarePairedDelimiter{\nrm}{\|}{\|}

\DeclarePairedDelimiter{\ceil}{\lceil}{\rceil}

\let\P\undefined

\DeclareMathOperator{\P}{P}

\newcommand{\mb}[1]{\boldsymbol{#1}}
\renewcommand{\bm}[1]{\boldsymbol{#1}}
\newcommand{\wt}[1]{\widetilde{#1}}
\newcommand{\wh}[1]{\widehat{#1}}
\newcommand{\wb}[1]{\widebar{#1}}

\let\underbar\undefined

\makeatletter
\let\save@mathaccent\mathaccent
\newcommand*\if@single[3]{%
	\setbox0\hbox{${\mathaccent"0362{#1}}^H$}%
	\setbox2\hbox{${\mathaccent"0362{\kern0pt#1}}^H$}%
	\ifdim\ht0=\ht2 #3\else #2\fi
}
\newcommand*\rel@kern[1]{\kern#1\dimexpr\macc@kerna}
\newcommand*\widebar[1]{\@ifnextchar^{{\wide@bar{#1}{0}}}{\wide@bar{#1}{1}}}
\newcommand*\underbar[1]{\@ifnextchar_{{\under@bar{#1}{0}}}{\under@bar{#1}{1}}}
\newcommand*\wide@bar[2]{\if@single{#1}{\wide@bar@{#1}{#2}{1}}{\wide@bar@{#1}{#2}{2}}}
\newcommand*\under@bar[2]{\if@single{#1}{\under@bar@{#1}{#2}{1}}{\under@bar@{#1}{#2}{2}}}
\newcommand*\wide@bar@[3]{%
	\begingroup
	\def\mathaccent##1##2{%
		\let\mathaccent\save@mathaccent
		\if#32 \let\macc@nucleus\first@char \fi
		\setbox\z@\hbox{$\macc@style{\macc@nucleus}_{}$}%
		\setbox\tw@\hbox{$\macc@style{\macc@nucleus}{}_{}$}%
		\dimen@\wd\tw@
		\advance\dimen@-\wd\z@
		\divide\dimen@ 3
		\@tempdima\wd\tw@
		\advance\@tempdima-\scriptspace
		\divide\@tempdima 10
		\advance\dimen@-\@tempdima
		\ifdim\dimen@>\z@ \dimen@0pt\fi
		\rel@kern{0.6}\kern-\dimen@
		\if#31
		\overline{\rel@kern{-0.6}\kern\dimen@\macc@nucleus\rel@kern{0.4}\kern\dimen@}%
		\advance\dimen@0.4\dimexpr\macc@kerna
		\let\final@kern#2%
		\ifdim\dimen@<\z@ \let\final@kern1\fi
		\if\final@kern1 \kern-\dimen@\fi
		\else
		\overline{\rel@kern{-0.6}\kern\dimen@#1}%
		\fi
	}%
	\macc@depth\@ne
	\let\math@bgroup\@empty \let\math@egroup\macc@set@skewchar
	\mathsurround\z@ \frozen@everymath{\mathgroup\macc@group\relax}%
	\macc@set@skewchar\relax
	\let\mathaccentV\macc@nested@a
	\if#31
	\macc@nested@a\relax111{#1}%
	\else
	\def\gobble@till@marker##1\endmarker{}%
	\futurelet\first@char\gobble@till@marker#1\endmarker
	\ifcat\noexpand\first@char A\else
	\def\first@char{}%
	\fi
	\macc@nested@a\relax111{\first@char}%
	\fi
	\endgroup
}
\newcommand*\under@bar@[3]{%
	\begingroup
	\def\mathaccent##1##2{%
		\let\mathaccent\save@mathaccent
		\if#32 \let\macc@nucleus\first@char \fi
		\setbox\z@\hbox{$\macc@style{\macc@nucleus}_{}$}%
		\setbox\tw@\hbox{$\macc@style{\macc@nucleus}{}_{}$}%
		\dimen@\wd\tw@
		\advance\dimen@-\wd\z@
		\divide\dimen@ 3
		\@tempdima\wd\tw@
		\advance\@tempdima-\scriptspace
		\divide\@tempdima 10
		\advance\dimen@-\@tempdima
		\ifdim\dimen@>\z@ \dimen@0pt\fi
		\rel@kern{0.6}\kern-\dimen@
		\if#31
		\underline{\rel@kern{-0.6}\kern\dimen@\macc@nucleus\rel@kern{0.4}\kern\dimen@}%
		\advance\dimen@0.4\dimexpr\macc@kerna
		\let\final@kern#2%
		\ifdim\dimen@<\z@ \let\final@kern1\fi
		\if\final@kern1 \kern-\dimen@\fi
		\else
		\underline{\rel@kern{-0.6}\kern\dimen@#1}%
		\fi
	}%
	\macc@depth\@ne
	\let\math@bgroup\@empty \let\math@egroup\macc@set@skewchar
	\mathsurround\z@ \frozen@everymath{\mathgroup\macc@group\relax}%
	\macc@set@skewchar\relax
	\let\mathaccentV\macc@nested@a
	\if#31
	\macc@nested@a\relax111{#1}%
	\else
	\def\gobble@till@marker##1\endmarker{}%
	\futurelet\first@char\gobble@till@marker#1\endmarker
	\ifcat\noexpand\first@char A\else
	\def\first@char{}%
	\fi
	\macc@nested@a\relax111{\first@char}%
	\fi
	\endgroup
}
\makeatother

\newcommand{\R}{\mathbb{R}}

\DeclarePairedDelimiterX\ip[2]{\langle}{\rangle}{#1,#2}

\let\P\undefined
\newcommand{\P}{\mathbb{P}}

\newcommand{\E}{\mathbb{E}}

\DeclarePairedDelimiterXPP\Es[2]{\mathbb{E}_{#1}}[]{}{
	
	#2
}

\renewcommand{\succeq}{\geq}

\DeclarePairedDelimiterXPP\Ipt[2]{{#1}\transpose}(){}{#2}
\DeclarePairedDelimiterXPP\Iptr[2]{}(){\transpose{#2}}{#1}

\usepackage{xspace}

\newcommand{\Pim}{\Pi}

 \newcommand{\A}{\mathcal{A}}

\def\argmax{\mathop{\mbox{ arg\,max}}}

\newcommand{\iprod}[2]{\left\langle#1,#2\right\rangle}

\newcommand{\norm}[1]{\left\|#1\right\|}

\newcommand{\transpose}{^\mathsf{\scriptscriptstyle T}}

\usepackage{todonotes}
\definecolor{PalePurp}{rgb}{0.66,0.57,0.66}

\newcommand{\veceval}{\mathrm{vec}}

\newcommand{\apx}{\texttt{LinOpt}}

\newcommand{\est}{\texttt{Vec}}

\newcommand{\wtilde}[1]{\widetilde{#1}}

\def\argmax{\mathop{\textup{ arg\,max}}}

\renewcommand{\tilde}{\wt}

\definecolor{darkpastelgreen}{rgb}{0.01, 0.75, 0.24}

\usepackage[normalem]{ulem}

\newcommand{\ntest}{n_{\texttt{test}}}
\newcommand{\mboost}{m_{\texttt{boost}}}
\newcommand{\nfqi}{n_{\texttt{fqi}}}
\newcommand{\nout}{n_{\texttt{outlier}}}
\newcommand{\ndir}{n_{\texttt{samp}}}
\newcommand{\nspan}{n_{\texttt{span}}}
\newcommand{\ncover}{n_{\texttt{cover}}}
\newcommand{\nrej}{n_{\texttt{reject}}}
\newcommand{\StateSpace}{\mathsf{S}}

\newcommand{\rob}{\texttt{rob}}
\newcommand{\apxspanner}{\textsc{RobustSpanner}\xspace}

\renewcommand{\veceval}{\textsc{EstVec}\xspace}
\newcommand{\ee}{\E}

\newcommand{\FQI}{\textsc{FQI}\xspace}
\newcommand{\nveceval}{n_{\texttt{veceval}\xspace}}
\renewcommand{\apx}{\textsc{LinOpt}\xspace}
\renewcommand{\est}{\textsc{Vec}\xspace}

\newcommand{\pirefh}[1][h]{\pi_{h,\texttt{ref}}}

\newcommand{\RS}{\textsc{RejectionSampling}\xspace}
\newcommand{\Out}{\mathrm{Out}}

\newcommand{\phibar}{\wb{\phi}}

\newcommand{\pihat}{\wh{\pi}}
\newcommand{\Span}{\operatorname{Span}}

\newcommand{\lazyspanner}{\textsc{SubspaceCover}\xspace}

\renewcommand{\argmax}{\operatorname{argmax}}

\renewcommand{\emptyset}{\varnothing}

\newcommand{\algcommentlight}[1]{\textcolor{blue!70!black}{\transparent{0.5}\footnotesize{\texttt{\textbf{//\hspace{2pt}#1}}}}}

\renewcommand{\ind}[1]{^{{\scriptscriptstyle#1}}}

\newcommand{\poly}{\mathrm{poly}}

\newcommand{\Qf}{Q}

\newcommand{\Qhat}{\wh{\Qf}}

\def\multiset#1#2{\ensuremath{\left(\kern-.3em\left(\genfrac{}{}{0pt}{}{#1}{#2}\right)\kern-.3em\right)}}

\renewcommand{\emptyset}{\varnothing}

\newcommand{\Alg}{\mathtt{Alg}}

\newcommand{\algcommentbiglight}[1]{\textcolor{blue!70!black}{\transparent{0.5}\footnotesize{\texttt{\textbf{/* #1~*/}}}}}

\renewcommand{\a}{\bm{a}}

\renewcommand{\br}{\bm{r}}

\newcommand{\reals}{\mathbb{R}}

\newcommand{\w}{\bm{w}}

\renewcommand{\P}{\mathbb{P}}

\DeclareMathOperator{\Law}{Law}

\newcommand{\x}{\bm{x}}
\renewcommand{\a}{\bm{a}}
\newcommand{\nn}{\nonumber}

\Crefformat{line}{#2Line #1#3}

\makeatletter
\let\OldStatex\Statex
\renewcommand{\Statex}[1][3]{%
  \setlength\@tempdima{\algorithmicindent}%
  \OldStatex\hskip\dimexpr#1\@tempdima\relax}
\makeatother

\newcommand{\paragraphi}[1]{\par\noindent\emph{#1.}}

\usepackage[suppress]{color-edits}

\addauthor{no}{ForestGreen}
\addauthor{zm}{red}
\addauthor{sr}{magenta}

\let\oldtextsc\textsc
\renewcommand{\textsc}[1]{\textup{\oldtextsc{#1}}}

\neurips{
\title{End-to-End Efficient RL for Linear Bellman Complete MDPs with Deterministic Transitions}
\author{%
  Anonymous Author(s)
}
}

\arxiv{
\title{End-to-End Efficient RL for Linear Bellman Complete MDPs with Deterministic Transitions}
\author{%
  Zakaria Mhammedi\thanks{Correspondence to: \texttt{mhammedi@google.com}} \\
  Google Research, NYC \\
  \and
  Alexander Rakhlin \\
  MIT \\
  \and
  Nneka Okolo \\
  MIT \\
}
}

\begin{document}

\maketitle

\begin{abstract}
We study reinforcement learning (RL) with linear function approximation in Markov Decision Processes (MDPs) satisfying \emph{linear Bellman completeness}---a fundamental setting where the Bellman backup of any linear value function remains linear. While statistically tractable, prior computationally efficient algorithms are either limited to small action spaces or require strong oracle assumptions over the feature space.
We provide a computationally efficient algorithm for linear Bellman complete MDPs with \emph{deterministic transitions}, stochastic initial states, and stochastic rewards. For finite action spaces, our algorithm is end-to-end efficient; for large or infinite action spaces, we require only a standard argmax oracle over actions. Our algorithm learns an $\veps$-optimal policy with sample and computational complexity polynomial in the horizon, feature dimension, and $1/\veps$.
\end{abstract}

\renewcommand{\u}{\bm{u}}
\renewcommand{\w}{\bm{w}}

\section{Introduction}

Reinforcement Learning (RL) is a standard framework for sequential decision making \citep{sutton2018reinforcement}, with applications ranging from game playing \citep{mnih2015human,silver2016mastering} and robotics \citep{andrychowicz2020learning} to training and fine-tuning large language models \citep{ouyang2022training,schulman2017proximal}.
In many of these applications, the state and action spaces are large or infinite, making function approximation essential to generalize across states and actions \citep{bertsekas1996neuro,sutton2018reinforcement}.
A classical and widely studied approach is \emph{linear value function approximation}, where one represents value functions as linear functions of pre-specified feature vectors.
This approach dates back decades \citep{bradtke1996linear,sutton2018reinforcement} and continues to underpin many modern empirical methods \citep{mnih2015human,schulman2017proximal}.

\paragraph{Linear Bellman completeness.}
This paper focuses on \emph{linear Bellman complete} MDPs, a fundamental setting in RL with function approximation.
In this setting, we are given feature maps $\phi_h:\cX\times\cA\to\reals^d$ for each step $h$ in the horizon (where $\cX$ and $\cA$ are the state and action spaces, respectively), and we assume that the \emph{Bellman backup} of any linear function (in these features) remains linear. The Bellman backup at step $h$ maps a function $f_{h+1}:\cX\times\cA\to\reals$ to the function $(x,a)\mapsto \E[\br_h + \max_{a'}f_{h+1}(\x_{h+1},a')\mid \x_h = x, \a_h =a]$; it is the core operator underlying value iteration.
Formally, linear Bellman completeness means that for any linear function $\langle \phi_{h+1}(x,a), \theta \rangle$, its Bellman backup at step $h$ can be expressed as $\langle \phi_h(x,a), \theta' \rangle$ for some coefficient vector $\theta'$.
This structural assumption is attractive for several reasons:
\begin{itemize}[leftmargin=*,itemsep=0.5pt]
  \item \emph{Generality:} Linear Bellman completeness strictly generalizes several well-studied RL models, including linear MDPs \citep{jin2020provably,yang2019sample}, and also captures the Linear Quadratic Regulator (LQR), a cornerstone model in control theory \citep{bertsekas1996neuro,wu2024computationally}.
  \item \emph{Necessity for value iteration:} It is a necessary condition for the success of least-squares value iteration (LSVI), one of the most classical algorithms in RL \citep{bradtke1996linear}. Without Bellman completeness, LSVI can diverge arbitrarily \citep{tsitsiklis1996feature}. %
  \item \emph{Statistical tractability:} When the learner has access to sufficiently exploratory data, linear Bellman completeness is sufficient for statistically efficient learning \citep{munos2005error,munos2008finite,zanette2020learning}. %
\end{itemize}

\paragraph{The computational challenge.}
Despite its simplicity and statistical tractability, linear Bellman completeness appears to exhibit a \emph{statistical-computational gap}.
While we know that polynomial sample complexity is achievable, the proposed algorithms that attain it are computationally intractable \citep{zanette2020learning,jin2021bellman}.
This gap is particularly notable given the centrality of linear Bellman completeness to the theory of value iteration.

The core difficulty is \emph{exploration}: algorithms must collect data with sufficient coverage to learn a near-optimal policy.
Classical approaches that rely on data from a fixed exploratory policy can fail if that policy does not provide adequate coverage of the state-action space, leading to poor sample efficiency or suboptimal policies \citep{munos2005error}.
Overcoming this requires careful exploration, but the two principal exploration paradigms that have proven successful in simpler RL settings both break down for linear Bellman complete MDPs:
\begin{itemize}[leftmargin=*,itemsep=0.5pt]
  \item \emph{Optimistic bonuses.} In tabular MDPs and linear MDPs, algorithms such as UCBVI \citep{azar2017minimax,jin2018q} and LSVI-UCB \citep{jin2020provably} add exploration bonuses to the reward that scale inversely with the frequency of visits to each state (or feature direction). This ``local optimism'' approach is computationally efficient but requires that the value function class be complete with respect to these exploration bonuses. Unfortunately, for linear Bellman completeness, standard bonus functions (e.g., quadratic functions of the features) do not preserve linearity under Bellman backups, breaking this approach \citep{zanette2020learning}.
  \item \emph{Global optimism and version spaces.} An alternative approach is to maintain a version space of plausible value functions and construct policies that are optimistic with respect to this set \citep{jiang2017contextual,zanette2020learning,jin2021bellman,du2021bilinear}. While statistically efficient, these methods require solving complex nonconvex optimization problems at each step, making them computationally intractable even in the linear case.
\end{itemize}

As a result, the following fundamental question has remained open:
\begin{center}
\emph{Is there an algorithm that learns a near-optimal policy in a linear Bellman complete MDP using $\poly(d, H, 1/\veps)$ samples and computation time?}
\end{center}
Even the special case where the action space is a small constant was only recently resolved \citep{golowich2024linear}, but their approach has complexity exponential in $|\cA|$. Moreover, their algorithm requires a bounded $Q$-function parameter assumption. As \cite{wu2024computationally} shows, this assumption is restrictive and limits generality (we discuss this limitation in more detail in \cref{sec:related}).
Prior to this, \citet{zanette2020provably} gave the first computationally efficient algorithm (FRANCIS), but under a strong reachability assumption requiring that every feature direction be reachable.

\paragraph{Our setting and contributions.}
We study the setting where the \emph{state transition dynamics are deterministic} but the initial state and rewards are stochastic \citep{wen2016efficient,du2020agnostic,wu2024computationally}.
This setting is practically relevant (e.g., in simulators, game environments, or robotic systems with known physics but stochastic rewards), and it remains theoretically challenging: the learner must still explore to estimate value functions from noisy reward feedback and handle stochastic initial states.
Importantly, even this restricted setting has remained open for large action spaces. While \citet{wu2024computationally} studied this setting, their approach is not end-to-end computationally efficient: it relies on a linear optimization oracle over the feature space (one that solves problems of the form $\sup_{(x,a)\in \cX\times \cA}\phi_h(x,a)^\top \theta$ for various $\theta$). Such an oracle is not needed in, e.g., the linear MDP setting, and its requirement represents a significant gap. Resolving this discrepancy is an important step toward understanding the general problem.

Our main contribution is a computationally efficient algorithm for linear Bellman complete MDPs with deterministic transitions, large action spaces, and stochastic rewards and initial states. The algorithm requires only an $\argmax$ oracle over actions (i.e., $\argmax_{a\in\cA}\phi_h(x,a)^\top\theta$ for a given $x$ and $\theta$), which is the standard oracle assumption used even in computationally efficient linear MDP algorithms \citep{jin2020provably}. For finite action spaces, this oracle is trivially implementable by enumeration, yielding an end-to-end efficient algorithm.

At a high level, the algorithm proceeds layer-by-layer, constructing a small set of policies that induce feature vectors spanning the relevant subspaces at each layer. The key insight is that deterministic dynamics enable efficient discovery of all relevant feature directions without requiring the linear optimization oracle that prior work relies upon.

While the paradigm of first learning a policy cover and then planning is well-established \citep{du2019provably,mhammedi2023efficient,mhammedi2023representation}, adapting it to linear Bellman completeness requires new ideas. The novelty in our approach lies in: (i)~exploiting deterministic dynamics to efficiently construct the cover without solving complex optimization problems, and (ii)~reducing stochastic-reward planning to deterministic-reward planning by first estimating the instantaneous reward parameters---this reduction allows us to bypass the bounded $Q$-function parameter assumption required by prior work.

\paragraph{Organization.}
\Cref{sec:prelims} formalizes the setting and introduces notation.
\Cref{sec:framework} introduces the exploration-exploitation framework underlying our approach and discusses the challenges of applying it to linear Bellman completeness. \Cref{sec:main_guarantee} presents our algorithm and guarantees.
The appendix contains related work (\cref{sec:related}), proofs of the main theorems, and supporting technical lemmas.

\section{Preliminaries}
\label{sec:prelims}
We model the environment as a Markov Decision Process (MDP), formally given by a tuple $\cM = (\cX, \cA, H, P, r^\star)$. Here $\cX$ and $\cA$ denote the state and action spaces, both of which may be large or infinite; when $\cA$ is finite, we write $A \coloneqq |\cA|$. The horizon is $H \in \mathbb{N}$. The reward function is specified by $r^\star = \{r_h^\star\}_{h=1}^H$ where each $r_h^\star : \cX \times \cA \to [0,1]$. The transitions are governed by $P = \{P_h\}_{h=0}^H$ where $P_h : \cX \times \cA \to \Delta(\cX)$ for $h \geq 1$, and $P_0(\cdot \mid \varnothing)$ specifies the initial state distribution.

A deterministic policy is a sequence $\pi = \{\pi_h : \cX \to \cA\}_{h=1}^H$. We denote by $\Pi$ the set of all deterministic policies and focus on such policies throughout. Under policy $\pi \in \Pi$, a trajectory $(\x_1, \a_1, \bm{r}_1), \ldots, (\x_H, \a_H, \bm{r}_H)$ is generated via $\a_h = \pi_h(\x_h)$, $\bm{r}_h \sim r_h^\star(\x_h, \a_h)$, $\x_{h+1} \sim P_h(\cdot \mid \x_h, \a_h)$, starting from $\x_1 \sim P_0(\cdot \mid \varnothing)$. We use $\P^\pi[\cdot]$ and $\E^\pi[\cdot]$ to denote probability and expectation with respect to this sampling process.

The state-action value function ($Q$-function) and state value function are defined as $Q_h^\pi(\cdot,\cdot) \coloneqq \E^\pi\bigl[\sum_{\tau=h}^H \bm{r}_\tau \mid \x_h=\cdot, \a_h=\cdot\bigr]$ and $V_h^\pi(\cdot) \coloneqq Q^\pi_h(\cdot,\pi_h(\cdot))$, respectively. We write $V_h^\star(\cdot)\coloneqq \sup_\pi V_h^\pi(\cdot)$\arxiv{ for the optimal value}.
More generally, for an arbitrary collection of reward functions $r_{1:H}$ (where each $r_h:\cX\times\cA\to\reals$), we define the value functions under $r_{1:H}$ as
\[
Q_h^\pi(x,a;r_{1:H}) \coloneqq \E^\pi\left[\sum_{\tau=h}^H r_\tau(\x_\tau,\a_\tau) \mid \x_h=x,\a_h=a\right],
\quad
V_h^\pi(x;r_{1:H}) \coloneqq Q_h^\pi(x,\pi_h(x);r_{1:H}).
\]
When $r_{1:H}=r^\star_{1:H}$ (the MDP's reward functions), this recovers the standard notation: $V_h^\pi(\cdot)=V_h^\pi(\cdot;r^\star_{1:H})$.

\begin{definition}[Linear Bellman completeness]
\label{def:lbc}
For each layer $h\in[H]$, let $\phi_h:\cX\times\cA\to\reals^d$ be a feature map and define $
  \cB_h \coloneqq \left\{\theta \in \reals^d : |\langle \phi_h(x,a), \theta \rangle| \le 1 \ \forall (x,a) \in \cX \times \cA\right\}.$

The MDP $\cM$ is defined to be \emph{linear Bellman complete} with respect to the feature mappings $(\phi_h)_{h\in[H]}$ if, for each $h \in [H]$, there is a mapping $\cT_h : \cB_{h+1} \to \cB_h$ so that,
\[
 \forall \theta \in \cB_{h+1}, \forall (x,a) \in \cX \times \cA, \quad \langle \phi_h(x,a), \cT_h \theta \rangle = \E \left[ \max_{a' \in \cA} \langle \phi_{h+1}(\x_{h+1}, a'), \theta \rangle \mid \x_h = x, \a_h = a \right].
\]
\end{definition}

\begin{assumption}[Deterministic transitions and Linear Bellman completeness]
  \label{ass:bellman_complete}
  We assume:
  \begin{enumerate}[leftmargin=*,itemsep=0.5pt]
      \item \emph{Deterministic transitions:} For each $h \in [H]$, the transition $P_h(\cdot \mid x, a)$ is supported on a single state for all $(x,a) \in \cX \times \cA$. (The initial state $\x_1$ and rewards $\bm{r}_h$ remain random.) \label{item:deterministic}
      \item \emph{Linear Bellman completeness:} The MDP $\cM$ satisfies \cref{def:lbc} with respect to $\phi_{1:H}$. \label{item:lbc}
  \end{enumerate}
  \end{assumption}

\begin{assumption}[Boundedness and linearity]
\label{ass:features}
We assume the feature maps $\phi_{1:H}$ are known and\arxiv{ that}:
\begin{enumerate}[leftmargin=*,itemsep=0.5pt]
    \item \emph{Boundedness:} For all $h \in [H]$ and $(x,a) \in \cX \times \cA$, $\|\phi_h(x,a)\|_2 \le 1$. \label{item:feature_boundedness}
    \item \emph{Reward linearity:} For all $h \in [H]$, $\E[\bm{r}_h\mid \x_h=x,\a_h=a] = \langle w^\star_h, \phi_h(x,a) \rangle$ for some unknown $w^\star_h \in \reals^d$ with $\|w^\star_h\|\leq \sqrt{d}$. \label{item:reward_linearity}
\end{enumerate}
\end{assumption}

\paragraph{Boundedness assumptions.}
Several related definitions of linear Bellman completeness have been considered in the literature, some of which impose additional $\ell_2$-norm constraints that can be overly restrictive.
We briefly discuss two such assumptions and explain why they are problematic:
\begin{enumerate}[leftmargin=*,itemsep=0.5pt]
  \item \emph{Bounded $Q$-function parameters.} Some works \citep{golowich2024linear,zanette2020learning,zanette2020provably} assume that the parameter vector $\theta$ of any $Q$-function satisfies $\|\theta\|_2 \le R$ for some polynomial $R$ in $d$. As shown in \citep[Section 3.1]{wu2024computationally}, this additional assumption is restrictive and weakens the generality of linear Bellman completeness.
  \item \emph{Non-expansiveness of Bellman backup.} \citet{song2022hybrid} assume that Bellman backups do not increase the $\ell_2$-norm of parameters. \neurips{This fails even in simple tabular MDPs \citep{wu2024computationally}.}\arxiv{This is even more restrictive and fails in simple tabular MDPs \citep{wu2024computationally}.}
\end{enumerate}
In contrast, \cref{item:reward_linearity} of \cref{ass:features} only assumes bounded parameters $\|w^\star_h\| \le \sqrt{d}$ for the \emph{instantaneous reward} functions.
This is a standard assumption in the linear MDP literature \citep{jin2020provably} and does not constrain the $Q$-function parameters.\footnote{Importantly, even with bounded reward parameters and Bellman completeness, the $Q$-function parameters can still grow unboundedly---but this does not pose a problem for our approach, which circumvents the need to bound $Q$-function parameters (see \cref{sec:main_guarantee}).}

\paragraph{Action argmax oracle.}
Throughout, we assume that ties in $\argmax_{a\in \cA}$ operations are broken in an arbitrary but fixed manner.
For large or infinite action spaces, we assume access to an oracle that returns this argmax, i.e., computes $\argmax_{a\in \cA} \phi_h(x,a)^\top \theta$ for any given state $x$ and vector $\theta$.

\arxiv{\paragraph{Goal.}
Our goal is to derive an algorithm that returns a policy
$\widehat{\pi}$ with $\E[V_1^\star(\x_1) - V_1^{\widehat{\pi}}(\x_1)] \le \veps$ with probability $1-\delta$,
with sample complexity polynomial in $(d,H,1/\veps,\log(1/\delta))$ and a polynomial number of calls to the action $\argmax$ oracle (for finite actions, this oracle is trivially implementable by enumeration).}
\neurips{\paragraph{Goal.} Under \cref{ass:bellman_complete} and \cref{ass:features}, we seek an algorithm returning a policy $\widehat{\pi}$ with $\E[V_1^\star(\x_1) - V_1^{\widehat{\pi}}(\x_1)] \le \veps$ with probability $1-\delta$, sample complexity $\poly(d,H,1/\veps,\log(1/\delta))$, and $\poly$ calls to an action $\argmax$ oracle (trivially implementable for finite actions).}

\paragraph{Additional notation.}
Since transitions and policies are deterministic, the trajectory from any $\x_1=x$ under $\pi$ is unique. We write $\phi^\pi_h(x)$ for the feature $\phi_h(\x_h, \a_h)$ along this trajectory.

\section{Algorithmic Framework and Challenges}
\label{sec:framework}

In this section, we introduce the exploration-exploitation paradigm \citep{kearns2002near,du2019provably,zanette2020provably,jin2020reward} that underlies our algorithm, discuss the key challenges that arise when applying it to the linear Bellman complete setting, and motivate the approach we develop in \cref{sec:main_guarantee}.
 
\paragraph{The exploration-exploitation paradigm.}
Beyond the tabular and linear MDP settings, a successful algorithmic paradigm in the theoretical RL literature is to \emph{separate exploration from exploitation}.
In this paradigm, the algorithm proceeds in two phases:
\begin{itemize}[leftmargin=*,itemsep=0.5pt]
  \item \emph{Phase I (Exploration):} Build a \emph{policy cover} layer-by-layer.
  Intuitively, a policy cover is a small set of policies such that for any state, there exists a policy in the set that reaches that state with near-optimal probability (or density).
  \item \emph{Phase II (Exploitation):} Use the policy cover to learn a near-optimal policy, typically via dynamic programming.
\end{itemize}
This paradigm has been applied successfully to several structured MDP settings, including block MDPs \citep{misra2020kinematic,mhammedi2023representation} (which generalize the tabular setting by incorporating rich observations), sparse linear MDPs \citep{golowich2024exploring}, low-rank MDPs \citep{mhammedi2023efficient}, and observable POMDPs \citep{golowich2022learning}.

\paragraph{Handling error compounding.}
A key challenge in provable RL beyond tabular settings is preventing approximation errors from \emph{compounding across layers} of the MDP.
\citet{mhammedi2023representation,mhammedi2023efficient,golowich2024exploring} address this challenge by decomposing the state space into two types of states:
\begin{itemize}[leftmargin=*,itemsep=0.5pt]
  \item \emph{High-reachability states:} States that are reachable with sufficient probability (or density) by some policy. The algorithms target covering these states and achieve a \emph{multiplicative guarantee}: for each such state, the covering policy reaches it with probability comparable to the optimal probability.
  \item \emph{Low-reachability states:} States that are not reachable by any policy with meaningful probability. These states can be \emph{safely ignored}.
\end{itemize}
The ``safe ignoring'' argument proceeds as follows.
Consider an imaginary MDP where each low-reachability state leads deterministically to a terminal state.
With non-negative rewards, any policy has higher value in the original MDP than in this imaginary MDP.
The key insight from \citep{mhammedi2023representation,mhammedi2023efficient,golowich2024exploring} is that, precisely because the ignored states have low reachability, the value of any policy differs only negligibly between the original and imaginary MDPs. Consequently, an optimal policy in the imaginary MDP is near-optimal in the original MDP.
Together, the multiplicative guarantee and the safe-ignoring argument imply that it suffices to focus on high-reachability states, where coverage errors can be controlled layer-by-layer without compounding.
Once a valid policy cover for these states is computed, exploitation via dynamic programming is straightforward.

\paragraph{Can this paradigm be applied to linear Bellman completeness?}
A natural question is whether this paradigm can be applied to learn a near-optimal policy in the linear Bellman complete setting.
To assess this, we need to verify three things:
\begin{enumerate}[leftmargin=*,itemsep=0.5pt]
  \item \label{item:q1} \emph{Intrinsic objective for cover construction.} Can one define an objective (or a set of objectives) that, when optimized, yields a notion of policy cover useful for learning a near-optimal policy in a second phase (e.g., via dynamic programming)?
  \item \label{item:q2} \emph{Layer-by-layer computation without compounding.} Can such a cover be approximately computed layer-by-layer without errors compounding between layers?
  \item \label{item:q3} \emph{Exploitation from the cover.} Given such a cover, can we perform dynamic programming to compute a near-optimal policy?
\end{enumerate}
In the remainder of this section, we partially address these questions and identify the main obstacles. In \cref{sec:main_guarantee}, we present our algorithm that overcomes all of them.

\paragraph{Barycentric spanners as policy covers.}
We start by addressing Question~\ref{item:q1}.
In the linear Bellman complete setting, there is a natural class of policies (argmax policies with a fixed tie-breaking rule) that contains the optimal policy and has the following property: for any policy $\pi$ and any $\pi'$ in this class, the expected value $\E^{\pi}[Q^{\pi'}_h(\x_h,\a_h)]$ is linear in the expected feature vector $\E^{\pi}[\phi_h(\x_h, \a_h)]$ \citep{golowich2024linear}; that is, there exists $\theta_{h}^{\pi'} \in \reals^d$ such that $\E^{\pi}[Q^{\pi'}_h(\x_h,\a_h)] = \E^{\pi}[\phi_h(\x_h,\a_h)]^\top \theta_{h}^{\pi'}$. Thus, policy evaluation reduces to estimating the parameter $\theta_h^{\pi'}$.
By rolling in with a policy $\pi$ up to layer $h$ and rolling out with $\pi'$, one can estimate $\theta_h^{\pi'}$ along the direction of $\E^{\pi}[\phi_h(\x_h, \a_h)]$ via regression. This motivates finding a small set of policies whose expected features span those of all other policies, so that data from the covering set suffices to estimate $\theta_h^{\pi'}$ in all relevant directions. This is precisely the notion of a \emph{barycentric spanner} \citep{awerbuch2008online}, also used in the low-rank MDP setting by \citet{mhammedi2023efficient}.
 
Formally, a $C$-barycentric spanner for $\cW \subseteq \reals^d$ is a subset $\{w_1, \ldots, w_d\} \subseteq \cW$ such that every $w \in \cW$ can be written as $w = \sum_{i=1}^d \beta_i w_i$ with $|\beta_i| \le C$.
We use the relaxed $(C, \veps)$-\emph{approximate} spanner \citep{mhammedi2023efficient}, allowing additive error: $\|w - \sum_{i=1}^d \beta_i w_i\| \le \veps$ with $|\beta_i| \le C$. In our setting, $\cW=\{\E^\pi[\phi_h(\x_h, \a_h)]: \pi \in \Pi\}$.
A spanner for $\cW$ guarantees that the corresponding policies cover all relevant feature directions, which is precisely what is needed to estimate $\theta_h^{\pi'}$ as described above.

To compute a spanner, we use the \apxspanner algorithm from \citep{mhammedi2023efficient} (see \cref{alg:spanner}), which tolerates approximate oracles---unlike the original construction of \citep{awerbuch2008online}---at the cost of additive error $\veps$.
The algorithm requires two oracles:
\begin{enumerate}[leftmargin=*,itemsep=0.5pt]
  \item \emph{Approximate linear optimization} ($\apx$): Given $\theta \in \reals^d$ and $h\in[H]$, return $\pihat$ approximately maximizing the functional $\pi \mapsto \theta^\top \E^\pi[\phi_h(\x_h, \a_h)]$ over $\pi \in \Pi$.
  \item \emph{Vector estimation} ($\est$): Given $\pi$, estimate $\E^\pi[\phi_h(\x_h, \a_h)]= \E[\phi_h^\pi(\x_1)]$.
\end{enumerate}
The vector estimation oracle $\est$ is implemented via Monte Carlo sampling (\cref{alg:veceval} in the appendix). The more challenging task is implementing the approximate linear optimization oracle $\apx$. We discuss how to construct such a subroutine in \cref{sec:main_guarantee}.

\paragraph{The challenge of exploitation: unbounded $Q$-function parameters.}
We now partially address Question~\ref{item:q3}.
As argued above, a barycentric spanner provides exactly the coverage needed to estimate value function parameters via regression. Given such a spanner, a natural approach to exploitation is \emph{Policy Search by Dynamic Programming} (PSDP) \citep{bagnell2003policy}, which constructs a policy layer by layer: at each layer starting from $H$, it uses the spanner's coverage to evaluate the $Q$-function of the current policy and selects the greedy action with respect to this $Q$-function.
Formally, PSDP proceeds via a backward loop over layers $\ell = H, \ldots, 1$: at each layer, one rolls in with policies from the cover to reach layer $\ell$, then rolls out with the partially constructed policy $\pihat_{\ell+1:H}$, regressing onto the sum of future rewards to learn a $Q$-function for layer $\ell$.
The policy $\pihat_\ell$ is then constructed as the argmax of this $Q$-function, which remains linear by the result of \citet{golowich2024linear}.
However, obtaining tractable guarantees for the regression problems requires bounded $Q$-function parameters---a restrictive assumption as shown by \citet{wu2024computationally} (see also \cref{sec:prelims}). 

Our approach bypasses this limitation by estimating the instantaneous reward parameters $w^\star_h$ \emph{first} (which are bounded by \cref{ass:features}), constructing deterministic proxy rewards $\wh r_h(x,a) = \phi_h(x,a)^\top \wh w_h$, and then planning over these deterministic rewards.
This decouples the regression problem (which requires only bounded reward parameters) from the planning problem.
Two challenges remain: first, accurate reward estimation requires a policy cover with good feature coverage---bringing us back to Questions~\ref{item:q1} and~\ref{item:q2}; second, even with deterministic proxy rewards, one must still solve the planning problem efficiently. We address both challenges in \cref{sec:main_guarantee}.

\paragraph{The compounding error challenge in linear Bellman completeness.}
Finally, we address Question~\ref{item:q2}.
Unlike block MDPs and low-rank MDPs, the linear Bellman complete setting does \emph{not} provide the structural decomposition into high- and low-reachability states that enables the safe-ignoring argument of \citet{mhammedi2023representation,mhammedi2023efficient}; see \cite{golowich2024linear}.
As a result, coverage errors from approximate spanners can compound across layers: if the spanner at layer $h$ fails to span some relevant feature directions, these gaps propagate and amplify at subsequent layers.
Since \apxspanner can only compute \emph{approximate} spanners, this compounding is unavoidable with the approach of prior work.
Thus, a fundamentally different approach is needed.

\section{An Efficient Algorithm}
\label{sec:main_guarantee}

Our algorithm follows the two-phase paradigm of \cref{sec:framework} and assumes access to an $\argmax$ oracle over actions (\cref{sec:prelims}).
Recall from \cref{sec:framework} that three challenges must be addressed: computing a policy cover for each layer, doing so without error compounding, and solving the planning problem efficiently after estimating reward parameters.
Our key ideas are:
(i)~we introduce \lazyspanner, a subroutine that computes a new type of coverage policies (distinct from the approximate spanners of \cref{sec:framework}) and can be used to \emph{exactly} identify all relevant feature directions at each layer, avoiding error compounding entirely;
and (ii)~we use Fitted $Q$-Iteration (\textsc{FQI}) with these coverage policies to perform exact linear optimization within each layer, solving both the exploration and the deterministic planning problems efficiently. We now describe how FQI with a new type of coverage policies is used in our algorithm for exact linear optimization.

\begin{algorithm}[h]
  \caption{$\lazyspanner(h,\Psi_{1:h-1},
      \veps,\delta)$.}
  \label{alg:notspanner}
  \begin{algorithmic}[1]
    \Require{$h,\Psi_{1:h-1}, \veps, \delta$.}
    \State Initialize $\cS\ind{1}_h \gets \emptyset$ and $\Psi\ind{1}_h \gets \emptyset$.
    \State Let $\ntest$, $\ndir$, $\mboost$, $\nrej$, $\nspan$, $\nfqi$ be as in \eqref{eq:parameters}.
    \For{$t=1,\dots, d+1$}
    \State Set $K\ind{t} \gets d - \mathrm{dim}(\cS\ind{t}_h)$.
    \State {\bf if} $K\ind{t}=0$ {\bf return} $(\cS\ind{t}_h, \Psi\ind{t}_{h})$.
    \State Let $\{\theta\ind{t,j}\}_{j\in [K\ind{t}]}$ be an orthonormal basis of $(\cS\ind{t}_h)^\perp$.
\Statex[1] \algcommentbiglight{Search for reachable directions outside the current subspace $\cS\ind{t}_h$.}
    \For{$i=1,\dots, K\ind{t}$}
    \State Set $r_{1}\ind{t, i, \pm}\equiv \dots r\ind{t,i, \pm}_{h-1}\equiv 0$ and $r\ind{t,i,-}_h(\cdot,\cdot) \equiv -\phi_h(\cdot, \cdot)^\top \theta\ind{t,i}$ and $r\ind{t,i,+}_h(\cdot,\cdot) \equiv \phi_h(\cdot, \cdot)^\top \theta\ind{t,i}$.
    \State Set $\pihat\ind{t,i,-} \gets \textsc{FQI}(h, r\ind{t,i,-}_{1:h}, \Psi_{1:h-1},\nfqi)$ and $\pihat\ind{t,i,+} \gets \textsc{FQI}(h, r\ind{t,i,+}_{1:h}, \Psi_{1:h-1},\nfqi)$. \label{line:fqi_call}
    \State Initialize $\nout^\pm \gets 0$.
    \For{$\ntest$ times}
    \State Sample $(\x^-_1, \a^-_1, \dots, \x^-_h, \a^-_h) \sim \P^{\pihat\ind{t,i,-}}$ and $(\x^+_1, \a^+_1, \dots, \x^+_h, \a^+_h) \sim \P^{\pihat\ind{t,i,+}}$.
    \State {\bf if} $\phi_h(\x^-_h,\a^-_h)\notin \cS\ind{t}_h$ {\bf then} $\nout^- \gets \nout^- + 1$.
    \State {\bf if} $\phi_h(\x^+_h,\a^+_h)\notin \cS\ind{t}_h$ {\bf then} $\nout^+ \gets \nout^+ + 1$.
    \EndFor
    \If{$\nout^+ \geq \nout^-$} \label{line:outlier_compare}
    \State Set $\pihat\ind{t,i} \gets \pihat\ind{t,i,+}$ and $\nout \gets \nout^+$. \label{line:first_outlier}
    \Else
    \State Set $\pihat\ind{t,i} \gets \pihat\ind{t,i,-}$ and $\nout \gets \nout^-$. \label{line:second_outlier}
    \EndIf
    \If{$\nout > \ntest \cdot \frac{\veps}{4 Hd}$} \hfill \algcommentlight{Outlier direction detected} \label{line:outlier} 
    \State $\Psi\ind{t+1}_{h}\gets \Psi\ind{t}_{h}\cup\{\pihat\ind{t,i}\}$.
    \State \multiline{Let $P\ind{t,i}$ denote the distribution of $\phi^\perp_h(\x_h,\a_h) = \mathrm{proj}_{(\cS\ind{t}_h)^\perp} \phi_h(\x_h,\a_h)$, where\\ $(\x_h,\a_h) \sim \P^{\pihat\ind{t,i}}$.}
    \State Define $Q\ind{t,i} = \textsc{RejectionSampling}(P\ind{t,i}, \reals^{d} \setminus \{0\} ,\nrej)$. \hfill \algcommentlight{\cref{alg:rejection-sampling-budgeted}}\label{line:rejection}
    \State Set $v\ind{t,i}\gets \textsc{ComputeOutlierDirection}(Q\ind{t,i}, \ndir, \mboost)$. \hfill \algcommentlight{\cref{alg:getdirection}} \label{line:computeoutlierdirection}
    \State $\cS\ind{t+1}_h\gets \Span\{\cS\ind{t}_h\cup \{v\ind{t,i}\}\}$. \label{line:adddir}
    \State {\bf break}
    \ElsIf{$i=K\ind{t}$}
    \State {\bf return } $(\cS\ind{t}_h, \Psi\ind{t}_{h})$.
    \EndIf
    \EndFor
    \EndFor
    \State \textbf{Return:} $(\cS\ind{t}_h, \Psi\ind{t}_{h})$.
  \end{algorithmic}
\end{algorithm}

\paragraph{FQI with exact linear optimization.}
To avoid the bounded $Q$-function parameter requirement of PSDP, we use FQI \citep{munos2008finite} and perform linear optimization \emph{exactly} within each layer.
To enable this exact linear optimization, we build an additional set of policies $\Psi_h$ at each layer $h$ that satisfy a \emph{coverage property}: for i.i.d.~initial states $\x_1,\x_1\ind{1}, \dots, \x_1\ind{\ncover}$,
\begin{align}
  \P\left[\cS_h \subseteq \Span\left\{\phi^\pi_h(\x_{1}\ind{m}): m\in[\ncover], \pi \in \Psi_{h}\right\}\right] \geq 1- \delta/H^2,
  \label{eq:coverage_property_intro}
\end{align}
where, with high probability over the initial state, $\cS_h$ is a subspace spanned by $\phi_h^{\pi}(\x_1)$ for all $\pi \in \Pi$.
Such a coverage property ensures that sampling enough initial states and rolling in with policies in $\Psi_h$ covers all relevant feature directions.
Thanks to deterministic transitions (\cref{ass:bellman_complete}), the regression problems in \textsc{FQI} can be solved \emph{exactly} via minimum-norm least-squares; coverage then guarantees these solutions are correct for most sampled initial states.
Unlike \citet{wu2024computationally}, our least-squares problems are \emph{unconstrained}, avoiding the need for a linear optimization oracle over features.
At the same time, we also compute an approximate spanner $\Gamma_h$ via \apxspanner.
The reason we need both $\Psi_h$ and $\Gamma_h$ is that while the coverage policies $\Psi_h$ suffice for deterministic regression in \textsc{FQI}, they lack the spanner property needed for stochastic regression.
In Phase II, estimating the reward parameters $w^\star_h$ via ridge regression requires the spanner property of $\Gamma_h$ for accurate estimation across relevant feature directions.

\paragraph{The \lazyspanner\ subroutine.}
The subroutine that computes the coverage policies $\Psi_h$ (along with the subspace $\cS_h$) is \lazyspanner\ (\cref{alg:notspanner}).
Unlike \apxspanner\ or the original barycentric spanner algorithm, \lazyspanner\ does not attempt to maximize the determinant (which would yield a true spanner).
Instead, its goal is simply to: (i)~identify the subspace $\cS_h$ where feature maps land with high probability, and (ii)~find a set of policies $\Psi_h$ that satisfy the coverage property \eqref{eq:coverage_property_intro}.
The formal guarantee (\cref{lem:notspanner_guarantee}) shows that given coverage policies for layers $1,\dots,h-1$, \lazyspanner\ outputs a subspace $\cS_h$ containing the span of all reachable feature vectors at layer $h$ with high probability, along with policies $\Psi_h$ that provide coverage for $\cS_h$.

\lazyspanner\ identifies the subspace $\cS_h$ in a greedy fashion.
It runs an iterative process that searches for policies leading to features that fall \emph{outside} the current subspace $\cS\ind{t}_h$.
Once such a policy is identified (detected by the test in \cref{line:outlier}), it is added to $\Psi_h$.
Next, this policy is used to identify a single ``persistent'' outlier direction $v$ via the subroutine \textsc{ComputeOutlierDirection} in \cref{line:computeoutlierdirection}.
Once this direction is found, the subspace is updated: $\cS\ind{t+1}_h \gets \Span(\cS\ind{t}_h \cup \{v\})$; see \cref{line:adddir}.
Since the ambient dimension is $d$, this process terminates after at most $d$ updates.

\paragraph{The \textsc{ComputeOutlierDirection} subroutine.}
The goal of \textsc{ComputeOutlierDirection} (\cref{alg:getdirection}) is to solve the following abstract problem (formalized in \cref{lem:compute_outlier_direction}): given access to samples from a distribution $P$ over $\reals^d$, identify a direction $v$ such that $v$ lies in the span of $n$ fresh i.i.d.\ samples from $P$ with high probability.
In our concrete setting, we are trying to identify a direction in the orthogonal complement of $\cS\ind{t}_h$.
The relevant distribution $P$ is that of the projected features $\phi^\perp_h(\x_h,\a_h) = \mathrm{proj}_{(\cS\ind{t}_h)^\perp} \phi_h(\x_h,\a_h)$ where $(\x_h,\a_h) \sim \P^{\pihat\ind{t,i}}$, \emph{conditioned on} $\phi_h(\x_h,\a_h) \notin \cS\ind{t}_h$ (i.e., conditioned on the feature falling outside the current subspace).
Since we do not have direct sampling access to this conditional distribution, we approximate it using rejection sampling (\cref{alg:rejection-sampling-budgeted}) in \cref{line:rejection} of \cref{alg:notspanner}.
We only call \textsc{ComputeOutlierDirection} when the current candidate policy has passed the test in \cref{line:outlier}, which detects that an outlier direction is discoverable with sufficient probability.

\begin{algorithm}[!htbp]
  \caption{$\textsc{ComputeSpanner}_{\texttt{fqi}}(\veps,\delta)$: Phase I (Exploration).}
  \label{alg:cover}
  \begin{algorithmic}[1]
    \Require{$\veps, \delta$.}
    \State Let $\nfqi$, $\nveceval$, and $\veps_{\rob}$ be as in \eqref{eq:parameters}.
    \For{$h=1,\dots, H$}
    \State $(\cS_h, \Psi_{h}) \gets \lazyspanner\big(h,\Psi_{1:h-1}, \veps, \frac{\delta}{4H}\big)$.
    \State For $\theta \in \reals^d$, $h\in[H]$, define $r_{h}(\cdot,\cdot;\theta)= \phi_h(\cdot,\cdot)^\top{\theta}$ and $r_{1:h-1}\equiv 0$.
    \State For $\theta\in\reals^d$, let $\apx(\theta)=\FQI(h, r_{1:h}(\cdot,\cdot;\theta),\Psi_{1:h-1},\nfqi)$. \label{line:fqi}\label{line:linopt} \hfill\algcommentlight{\cref{alg:fqi}.}
    \State For $\pi\in\Pim$, define $\est(\pi)=\veceval(h,\phi_h, \pi, \nveceval)$. \label{line:est} \hfill\algcommentlight{\cref{alg:veceval}.}
    \State \label{line:spanner}Set $\Gamma_h =\{\pi_{1}, \dots, \pi_d\} =  \apxspanner(\apx(\cdot), \est(\cdot), 2,  \veps_{\rob})$. \hfill\algcommentlight{\cref{alg:spanner}.}
    \EndFor
    \State \textbf{Return:} $(\cS_{1:H}, \Psi_{1:H},\Gamma_{1:H})$.
  \end{algorithmic}
\end{algorithm}

\renewcommand{\w}{\bm{w}}
\renewcommand{\u}{\bm{u}}

The following lemma summarizes the guarantee of \cref{alg:cover}: it outputs coverage policies $\Psi_{1:H}$ and spanner policies $\Gamma_{1:H}$ with the properties needed for Phase II.
The proof is in \cref{sec:spanner_guarantee_main}.
\begin{lemma}[Guarantee of $\textsc{ComputeSpanner}_{\texttt{fqi}}$]
  \label{thm:cover_guarantee} 
  Consider a call to \cref{alg:cover} with inputs $\veps,\delta$.
  Suppose \cref{ass:features,ass:bellman_complete} hold.
  Then, with probability at least $1-\delta$, the outputs $(\cS_{1:H}, \Psi_{1:H}, \Gamma_{1:H})$ satisfy the following properties:
  \begin{enumerate}[leftmargin=*,itemsep=0.5pt]
    \item \label{item:thm_coverage} \emph{(Coverage)} For all $h\in[H]$ and i.i.d.~$\x_1, \x_1\ind{1},\dots,\x_1\ind{\ncover}$ initial states:
        \begin{align}
          \label{eq:thm_highcoverage}
          \P[\cS_h \subseteq \Span\left\{\phi^\pi_h(\x_{1}\ind{m}): m\in[\ncover], \pi \in \Psi_{h}\right\}] \geq 1- \delta/H^2.
        \end{align}
    \item \label{item:thm_concentration} \emph{(State concentration)} For all $h\in[H]$, $\P[\x_1\in \cX\ind{h}]\geq 1-(h-1) \veps/H$, where
            \begin{align}
              \cX\ind{h} =\bigcap_{\ell \in[h-1]}\left\{x\in \cX \mid  \Span \left\{ \phi^{\pi}_\ell(x) : \pi \in \Pi\right\} \subseteq \cS_{\ell}\right\}. \label{eq:thm_Xdef}
            \end{align}
    \item \label{item:thm_spanner} \emph{(Spanner property)} For all $h\in[H]$, the policy set $\Gamma_h=\{\pi_{1},\dots,\pi_d\}$ satisfies: for any $\pi \in \Pi$, there exist coefficients $\beta_{1},\dots,\beta_d\in[-2,2]$ such that
      \begin{align}
        \nrm*{\E^{\pi}[\phi_h(\x_h,\a_h)]- \sum_{j=1}^d \beta_j\cdot \E^{\pi_j}[\phi_h(\x_h,\a_h)]} \leq 25 d \veps.
        \label{eq:cover_spanner_prop}
      \end{align}
  \end{enumerate}
\end{lemma}

The proof (in \cref{sec:spanner_guarantee_main}) proceeds by induction over layers. Below, we sketch how \lazyspanner achieves the coverage (\cref{item:thm_coverage}) and state concentration (\cref{item:thm_concentration}) properties at a single layer $h$, assuming they hold for layers $1,\dots,h-1$.

\paragraph{Proof sketch for the \lazyspanner guarantee (\cref{lem:notspanner_guarantee}).}
Fix a layer $h$ and suppose coverage policies $\Psi_{1:h-1}$ for previous layers are available.
\lazyspanner iterates over an orthonormal basis $\{\theta\ind{t,j}\}$ of $(\cS\ind{t}_h)^\perp$, where $\cS\ind{t}_h$ is the subspace of feature directions discovered so far at iteration $t$ (initialized to $\emptyset$), and uses \textsc{FQI} with proxy rewards $r_h(\cdot,\cdot) = \pm\phi_h(\cdot,\cdot)^\top \theta\ind{t,j}$ (and $r_{1:h-1} \equiv 0$) to find policies that maximize the features' absolute projection outside the current subspace.
A crucial property, enabled by deterministic transitions and the coverage of $\Psi_{1:h-1}$, is that the resulting policies $\pihat$ satisfy $\pm\phi_h^{\pihat}(x)^\top \theta\ind{t,j} \ge \pm\phi_h^{\pi}(x)^\top \theta\ind{t,j}$ for all $\pi \in \Pi$ and all $x \in \cX\ind{h}$ \emph{simultaneously}---not just in expectation (\cref{cor:fqi}).
The mechanism is as follows: with deterministic transitions, the regression targets in \textsc{FQI} (\cref{alg:fqi}) are noiseless; the minimum-norm least-squares estimator $\wh\theta_\ell$ then satisfies the normal equations $\Sigma_\ell(\wh\theta_\ell - \tilde\theta_\ell)=0$ (\cref{lem:noiseless-regression-orthogonality}), where $\tilde\theta_\ell$ is the parameter of the Bellman backup of \textsc{FQI}'s current value function estimate at layer $\ell+1$ (which exists by linear Bellman completeness) and $\Sigma_\ell = \sum_{(x,a,y) \in \cD_\ell} \phi_\ell(x,a)\phi_\ell(x,a)^\top$ is the empirical covariance built from the \textsc{FQI} dataset $\cD_\ell$ (which consists of rollouts from $\Psi_\ell$). This means the Bellman residual is orthogonal to the column space of $\Sigma_\ell$.
Because $\Psi_{1:h-1}$ covers $\cS_{1:h-1}$, the column space of $\Sigma_\ell$ contains all feature vectors $\phi_\ell^\pi(x)$ for $x \in \cX\ind{h}$, and thus the Bellman residual vanishes exactly on these states.

This pointwise guarantee is what makes the outlier detection test valid:
if $\pihat$ achieves the pointwise maximum of $\pm\phi_h^\pi(x)^\top \theta\ind{t,j}$ over $\pi$ and produces features outside $\cS\ind{t}_h$ with probability below $\veps/(4Hd)$, then \emph{no} policy can produce features outside $\cS\ind{t}_h$ with much higher probability.
Once the test fails for all basis directions of $(\cS\ind{t}_h)^\perp$, the algorithm can safely terminate: $\cS\ind{t}_h$ contains the span of all reachable features.

When an outlier \emph{is} detected, \textsc{ComputeOutlierDirection} identifies a persistent direction $v$ from the conditional distribution of features outside $\cS\ind{t}_h$.
The key tool is a finite-sample exchangeability result (\cref{thm:span}): among $n+1$ i.i.d.\ samples from any distribution over $\reals^d$, the probability that the last sample falls outside the span of the first $n$ is at most $d/(n+1)$.
Setting $n \gg d$ ensures that $v$ lies in the span of future rollouts with high probability.
As the subspace grows from $\cS\ind{t}_h$ to $\cS\ind{t+1}_h = \Span(\cS\ind{t}_h \cup \{v\})$, the coverage guarantee extends by induction:
any $u \in \cS\ind{t+1}_h$ decomposes into a $\cS\ind{t}_h$-component (covered by the induction hypothesis) and a $v$-component (covered by the new policy $\pihat$); a union bound shows that the failure probability increases by at most $1/(2d)$ per expansion.
Since $\dim(\cS_h) \le d$, the outer loop terminates after at most $d$ expansions, and the total failure probability remains below $1/2$.

\begin{algorithm}[!htbp]
  \caption{$\textsc{PolicyOpt}_{\texttt{fqi}}(\Psi_{1:H-1},\Gamma_{1:H},n,\lambda)$: Phase II (Policy optimization)}
  \label{alg:fqi-stochastic-appendix}
  \begin{algorithmic}[1]
    \Require{$\Psi_{1:H-1},\Gamma_{1:H},n,\lambda>0$.}
    \Statex[0] \algcommentbiglight{Estimating the parameters of the reward functions at every layer}
    \For{$\ell = 1, \ldots, H$}
    \State Set $\cD_\ell\gets \emptyset$.
    \For{$\pi\in\Gamma_\ell$}
    \For{$n$ times}
    \State Sample $(\x_1, \a_1,\br_1,\dots,\x_\ell,\a_\ell, \bm{r}_\ell) \sim \P^{\pi}$.  \hfill \algcommentlight{Sample from policy $\pi$ at layer $\ell$}
    \State Update $\cD_{\ell}\gets \cD_{\ell}\cup \{(\x_\ell,\a_\ell, \bm{r}_\ell)\}$.
    \EndFor
    \EndFor
    \State Set $\Sigma_\ell \gets \sum_{(x,a,r)\in \cD_\ell} \phi_\ell(x,a)\phi_\ell(x,a)^\top$.
    \State Set $\wh w_\ell \gets \left(\lambda I + \Sigma_\ell\right)^{-1} \cdot \sum_{(x,a,r)\in \cD_{\ell}} \phi_\ell(x, a) \cdot r$.
    \State Define deterministic proxy reward: $\wh r_\ell(\cdot,\cdot) \coloneqq \phi_\ell(\cdot,\cdot)^\top \wh w_\ell$.
    \EndFor
    \Statex[0] \algcommentbiglight{Planning with deterministic FQI}
    \State Call deterministic FQI with proxy rewards: $\pihat \gets \textsc{FQI}(H, \wh r_{1:H}, \Psi_{1:H-1}, n)$. \hfill \algcommentlight{\cref{alg:fqi}}
    \State \textbf{Return:} $\pihat$.
  \end{algorithmic}
\end{algorithm}

\paragraph{Phase II: Policy optimization with stochastic rewards.}
Given the coverage policies $\Psi_{1:H}$ and spanner policies $\Gamma_{1:H}$ computed in Phase~I, we now turn to exploitation.
As discussed in \cref{sec:framework}, directly applying PSDP to stochastic rewards requires bounded $Q$-function parameters, which may not hold since Bellman backups can amplify parameter norms arbitrarily. Our approach bypasses this by estimating reward parameters first and reducing to deterministic planning.
Phase II (\cref{alg:fqi-stochastic-appendix}) implements this strategy: it estimates the reward parameters $w^\star_h$ via regression on rollouts from the spanner policies $\Gamma_{1:H}$, then constructs deterministic proxy rewards and optimizes them using \textsc{FQI}.

The coverage property of $\Psi_{1:H-1}$ ensures that \textsc{FQI} can solve the value function regression problems exactly (in closed form), while the spanner property of $\Gamma_{1:H}$ ensures accurate reward parameter estimation.
Together, these properties guarantee that \cref{alg:fqi-stochastic-appendix} outputs a near-optimal policy for the original stochastic reward MDP; see \cref{sec:policyopt-fqi} for the formal statement and proof.

Combining Phase I (\cref{alg:cover}) and Phase II (\cref{alg:fqi-stochastic-appendix}), we obtain the following end-to-end guarantee, showing that the complexity is independent of the size of $\cA$ (the proof is in \cref{sec:mainproof}).

\begin{theorem}[End-to-end suboptimality guarantee]
  \label{thm:main}
  Fix $\veps, \delta \in (0,1)$ and horizon $H \in \mathbb{N}$. Suppose \cref{ass:features,ass:bellman_complete} hold.
  Define $L \coloneqq \max\{1, \log(dH/(\delta\veps))\}$.
  Consider the \arxiv{following }procedure:
  \begin{enumerate}[leftmargin=*,itemsep=0.5pt]
    \item Run \cref{alg:cover} with inputs $(\veps', \delta/4)$ where $\veps' \coloneqq \veps/(c_0 d^2 H^2 L)$ for a sufficiently large universal constant $c_0$, obtaining outputs $(\cS_{1:H}, \Psi_{1:H}, \Gamma_{1:H})$.
    \item Run $\textsc{PolicyOpt}_{\texttt{fqi}}(\Psi_{1:H-1}, \Gamma_{1:H}, n, 1)$ (\cref{alg:fqi-stochastic-appendix}) with \neurips{$n \coloneqq \max\left\{\ncover,\, c_2 \cdot \frac{d^4 H^2 L^2}{\veps^2}\right\}$, \label{eq:n-policyopt}}\arxiv{
    \begin{equation}
      n \coloneqq \max\left\{\ncover,\, c_2 \cdot \frac{d^4 H^2 L^2}{\veps^2}\right\},
      \label{eq:n-policyopt}
    \end{equation}
    }
    where $\ncover$ is defined in \eqref{eq:parameters} instantiated with $\veps'$, and $c_2$ is a sufficiently large universal constant, obtaining output policy $\pihat$.
  \end{enumerate}
  Then, with probability at least $1-\delta$, we have \neurips{$ \E\left[V_1^\star(\x_1) - V_1^{\pihat}(\x_1)\right] \le \veps.$} \arxiv{
  \[
    \E\left[V_1^\star(\x_1) - V_1^{\pihat}(\x_1)\right] \le \veps.
  \]}
\end{theorem}

\clearpage
\neurips{
\begin{ack}
Alexander Rakhlin acknowledges support from ARO through award W911NF-21-1-0328, Simons Foundation, and the NSF through awards DMS-2031883 and PHY-2019786, the DARPA AIQ program, and AFOSR FA9550-25-1-0375. Nneka Okolo acknowledges support from the Institute for Data, Systems, and Society (IDSS) via the Norbert Wiener Fellowship.
\end{ack}
}
\arxiv{
\section*{Acknowledgments}
Alexander Rakhlin acknowledges support from ARO through award W911NF-21-1-0328, Simons Foundation, and the NSF through awards DMS-2031883 and PHY-2019786, the DARPA AIQ program, and AFOSR FA9550-25-1-0375. Nneka Okolo acknowledges support from the Institute for Data, Systems, and Society (IDSS) via the Norbert Wiener Fellowship.
}
\neurips{\bibliographystyle{plainnat}}
\arxiv{\bibliographystyle{plainnat}}
\clearpage
\appendix
\section*{Appendix}\label{supp_material}
\addcontentsline{toc}{section}{Appendix}

\startcontents[appendix]
\printcontents[appendix]{}{-1}{\setcounter{tocdepth}{2}}
\vspace{1em}
 
\clearpage
\part{Additional Related Work}
\label{part:addrelated}
\section{Related Works}
\label{sec:related}
\paragraph{Tabular and linear MDPs.}
The simplest RL setting is the \emph{tabular} case, where both the state and action spaces are finite and small enough that every state-action pair can be visited multiple times.
Here, computationally efficient and near-optimal algorithms have been known for decades, including classical methods such as $Q$-learning \citep{watkins1992q} and value iteration \citep{bertsekas1996neuro}, as well as modern optimistic exploration algorithms such as UCBVI \citep{azar2017minimax}, $Q$-learning-UCB \citep{jin2018q}, and variants achieving optimal sample complexity \citep{zhang2020almost}.

The \emph{linear MDP} setting \citep{jin2020provably,yang2019sample} extends the tabular case to large state and action spaces by assuming that the transition dynamics and rewards are linear in known feature maps $\phi_h(x,a) \in \reals^d$.
Specifically, there exist an unknown signed measure $\mu_h$ and an unknown vector $\theta_h \in \reals^d$ such that for all bounded measurable functions $f: \cX \to \reals$ and all $(x,a) \in \cX \times \cA$, $
\mathbb{E}[f(\x_{h+1}) \mid \x_h=x,\a_h=a] = \langle \phi_h(x,a), \mu_h(f) \rangle$, where $\mu_h(f) \coloneqq \int f(x')  d\mu_h(x')$, and $r_h(x,a) = \langle \phi_h(x,a), \theta_h \rangle,$ with $\|\theta_h\|\leq \sqrt{d}$.
This structure allows efficient exploration via optimistic bonuses, leading to computationally efficient algorithms such as LSVI-UCB \citep{jin2020provably}, with recent refinements achieving optimal or near-optimal rates \citep{he2023nearly,zanette2020frequentist}.

\paragraph{Statistically efficient approaches.}
Linear Bellman completeness strictly generalizes linear MDPs but does not impose any structure on the transition function---it only requires that Bellman backups of linear value functions remain linear \citep{zanette2020learning}.
This makes the setting significantly more challenging computationally, since techniques that work for linear MDPs (e.g., adding quadratic exploration bonuses) no longer apply: quadratic functions do not remain linear under Bellman backups in this setting.

Early work established that linear Bellman completeness is \emph{statistically} tractable.
\citet{munos2005error,munos2008finite} showed that with sufficiently exploratory data, fitted $Q$-iteration succeeds under Bellman completeness.
\citet{zanette2020learning} proposed ELEANOR, an algorithm that achieves polynomial sample complexity in the online setting by maintaining a version space and performing global optimistic planning.
Similarly, \citet{jin2021bellman} proposed GOLF using Bellman eluder dimension, and \citet{du2021bilinear} studied bilinear classes.
However, all these algorithms are \emph{computationally intractable}: they require solving nonconvex optimization problems to implement optimistic planning over complex version spaces, even when the value function class is linear.

\paragraph{Reachability and constant actions.}
The first computationally efficient algorithm for linear Bellman completeness was FRANCIS \citep{zanette2020provably}, which works under a strong \emph{reachability} assumption that for every direction in feature space, there exists a policy whose expected features have non-trivial projection onto that direction.
The question of computationally efficient learning without reachability remained open until very recently.
\citet{golowich2024linear} made a breakthrough by introducing a \emph{local optimism} approach that constructs exploration bonuses preserving Bellman linearity.
This was the first polynomial-time algorithm for the general setting.
However, their approach has a critical limitation: the sample and computational complexity scale as $\mathrm{poly}(d, H, 1/\veps)^{O(|\cA|)}$, exponential in the number of actions.
As they note, this restricts their method to settings where $|\cA|$ is a small constant (e.g., $|\cA| = 2$ or $3$). Additionally, their algorithm requires a bounded $Q$-function parameter assumption: that the parameter $\theta$ satisfying $Q_h(\cdot,\cdot) = \phi_h(\cdot,\cdot)^\top \theta$ has $\|\theta\|_2 \le R$ for some polynomial $R$ in $d$. As shown by \citet{wu2024computationally}, this assumption is restrictive and limits generality.

\paragraph{Deterministic dynamics setting.}
The special case of deterministic transition dynamics has been studied as a stepping stone toward the general problem, and is practically relevant for applications such as game environments and robotic systems with known kinematics.
\citet{wen2016efficient} provided a polynomial-time algorithm when both transitions and rewards are deterministic. For large action spaces, we develop an approach based on exact identification of relevant feature directions (see \cref{sec:main_guarantee}).
\citet{du2020agnostic} extended this to stochastic rewards under a \emph{positive suboptimality gap} assumption, meaning that suboptimal policies have value strictly bounded away from optimal.
Most recently, \citet{wu2024computationally} studied the same setting as ours: deterministic transitions with stochastic rewards and initial states, without assuming a suboptimality gap.
Their algorithm uses a randomization-based approach called ``null space randomization,'' which adds exploration noise only in the null space of previously collected data.
However, their algorithm requires solving constrained regression problems that rely on a \emph{linear optimization oracle} over the feature space, specifically the ability to compute $\sup_{(x,a)\in \cX \times \cA} \phi_h(x,a)^\top \theta$ for any $\theta$.
Our work removes this oracle assumption for finite action spaces, achieving end-to-end computational efficiency using only standard regression and value function evaluation. For infinite action spaces, we require only an $\argmax$ oracle over actions (i.e., $\argmax_{a\in \cA} \phi_h(x, a)^\top \theta$ for a given $x$ and $\theta$), which optimizes only over actions rather than state-action pairs. This is a strictly weaker requirement and is the standard oracle assumption in computationally efficient algorithms for linear MDPs \citep{jin2020provably,yang2019sample}.

\paragraph{Policy covers and spanning arguments.}
Our algorithm is inspired by the policy cover framework \citep{du2019provably,misra2020kinematic}, which constructs a small set of policies that collectively ``cover'' the state space.
This approach has been successful in related problems such as representation learning in RL \citep{mhammedi2023representation,mhammedi2023efficient,golowich2023exploring} and learning in POMDPs \citep{golowich2022learning}.
A key challenge in applying policy covers to linear Bellman completeness is that coverage errors compound across layers: if the policy cover at layer $h$ fails to span some relevant feature directions, these gaps propagate and amplify at subsequent layers, causing the cover to miss increasingly more of the feature space as the horizon grows.
Our contribution is showing how deterministic dynamics help address this challenge. We develop machinery (\lazyspanner) that exactly identifies all relevant feature directions at each layer, avoiding the error compounding entirely (see \cref{sec:main_guarantee}).

\paragraph{Exploration via randomization.}
Randomization has proven to be a powerful alternative to bonus-based exploration in RL.
A prominent example is \emph{Randomized Least-Squares Value Iteration} (RLSVI) \citep{osband2016generalization}, which adds Gaussian noise to least-squares estimates and has been shown to achieve near-optimal worst-case regret for linear MDPs \citep{agrawal2021improved,zanette2020frequentist}.
\citet{ishfaq2021randomized} developed randomization algorithms for general function approximation under bounded eluder dimension and Bellman completeness.
In preference-based RL, randomization has led to the first computationally efficient algorithm with near-optimal regret for linear MDPs \citep{wu2024making}.

However, these randomization approaches face challenges when applied to linear Bellman completeness:
many inject noise larger than the estimation error to ensure optimism, which can cause exponential growth of parameter values over the horizon.
Truncating these values works for low-rank MDPs but is problematic under linear Bellman completeness, since truncation may break the linearity property.
Our work and the concurrent work of \citet{wu2024computationally} address this by restricting noise to the null space of collected data, though we differ in our approach to ensuring computational efficiency.

\paragraph{Broader structural conditions.}
Linear Bellman completeness is captured by several broader structural conditions studied in the RL literature, including \emph{Bellman rank} \citep{jiang2017contextual}, \emph{witness rank} \citep{sun2019model}, \emph{Bellman eluder dimension} \citep{jin2021bellman}, the \emph{decision-estimation coefficient} \citep{foster2021statistical}, and \emph{bilinear classes} \citep{du2021bilinear}.
These frameworks provide different lenses for understanding when RL is statistically tractable.
While statistically efficient algorithms exist for these settings, computationally efficient algorithms remain unknown in general.

\paragraph{Offline setting} In the offline setting, \citet{golowich2024role} recently established a computationally efficient algorithm for linear Bellman completeness, leveraging different techniques suitable for the batch setting.
 
\clearpage
\part{Proofs}
\label{part:proofs}
\section{Choice of Parameters for $\textsc{ComputeSpanner}_{\texttt{fqi}}$ (\cref{alg:cover})}
\label{sec:parameters}
For a sufficiently large universal constant $c_1>0$, we set 
\begin{equation}
  \begin{aligned}
    \veps_{\rob} &\coloneqq \frac{\veps}{192 d H^2}, \\
    N_{\texttt{iter}} &\coloneqq d + \left\lceil\frac{d}{2}\log_2\frac{100d}{\veps_{\rob}^2}\right\rceil, \\
    \delta' &\coloneqq \frac{\delta}{8 H d^2 N_{\texttt{iter}}}, \\
    \nveceval &\coloneqq c_1 H^4 d^2 \veps^{-2} \log(1/\delta'), \\
    \ntest &\coloneqq \frac{128 H^2 d^2\log(4/\delta')}{\veps^2}, \\
    \ndir &\coloneqq 512 d^2 \log(256d/\delta'), \\
    \mboost &\coloneqq 2048 d^2 \log(4/\delta'), \\
    \nrej &\coloneqq \frac{8 H d \log \left(\left(\ndir + \ndir^2 \mboost\right)/\delta'\right)}{\veps}, \\
    \nspan &\coloneqq \left\lceil \frac{16Hd \sqrt{\log (4d) \ndir}} {\veps} + \frac{32 Hd \log(4d)}{\veps} + \frac{8Hd\ndir}{\veps}\right\rceil, \\
    \ncover = \nfqi &\coloneqq \left\lceil  \log_2(H^2 d/\delta')  \right\rceil \cdot \nspan.
  \end{aligned}
  \label{eq:parameters}
\end{equation}

\section{Guarantee of \lazyspanner}
\label{sec:lazyspanner}
\begin{lemma}[Discovery of a new direction] \label{lem:notspanner}
  Consider a call to \lazyspanner with inputs $h$, $\Psi_{1:h-1}$, $\veps$, and $\delta \in (0,1)$, and let $\delta'$ be as in \eqref{eq:parameters}. Suppose \lazyspanner reaches step $i\geq 1$ in iteration $t\geq 1$, and let $\pihat\ind{t,i,\pm}$ be the policies returned by the \textsc{FQI} calls on \cref{line:fqi_call} of \cref{alg:notspanner}.
  Then, conditioned on $\pihat\ind{t,i, \pm}$, with probability at least $1-\delta'$ (over the randomness of the $\ntest$ test samples used to compute $\nout^\pm$ in \cref{alg:notspanner}):
  \arxiv{
  \begin{align}
    \left|\frac{\nout^-}{\ntest} - \P\left[\phi_h^{\pihat\ind{t,i,-}}(\x_1)\notin \cS\ind{t}_h \right] \right| \leq \sqrt{\frac{2\log(4/\delta')}{\ntest}} \quad \text{and} \quad \left|\frac{\nout^+}{\ntest} - \P\left[\phi_h^{\pihat\ind{t,i,+}}(\x_1)\notin \cS\ind{t}_h \right] \right| \leq \sqrt{\frac{2\log(4/\delta')}{\ntest}}. \label{eq:notspanner}
  \end{align}
  }
    \neurips{
  \begin{equation}
  \label{eq:notspanner}
  \begin{gathered}
    \left|\frac{\nout^-}{\ntest} - \P\left[\phi_h^{\pihat\ind{t,i,-}}(\x_1)\notin \cS\ind{t}_h \right] \right| \leq \sqrt{\frac{2\log(4/\delta')}{\ntest}} \\
    \text{and} \qquad \qquad \qquad \qquad \qquad \qquad \qquad \qquad \qquad \qquad \qquad \qquad \qquad \qquad \qquad \qquad \\
     \left|\frac{\nout^+}{\ntest} - \P\left[\phi_h^{\pihat\ind{t,i,+}}(\x_1)\notin \cS\ind{t}_h \right] \right| \leq \sqrt{\frac{2\log(4/\delta')}{\ntest}}.
  \end{gathered}
  \end{equation}
  }
\end{lemma}

\begin{proof}
  The result follows by applying Hoeffding's inequality to the random variables $\mathbb{I}\{\phi_h^{\pihat\ind{t,i,\pm}}(\x_1)\notin \cS\ind{t}_h\}$, which are i.i.d.\ Bernoulli random variables with mean $\P\left[\phi_h^{\pihat\ind{t,i,\pm}}(\x_1)\notin \cS\ind{t}_h \right]$.
\end{proof}

\begin{lemma}[Computing outlier direction] \label{lem:compute_outlier}
  Consider a call to \lazyspanner with inputs $h$, $\Psi_{1:h-1}$, $\veps$, and $\delta\in (0,1)$, and let $\delta'$ be as in \eqref{eq:parameters}. Suppose that \lazyspanner reaches step $i\geq 1$ of iteration $t\geq 1$, and let $\pihat\ind{t,i}$ be the policy computed on \cref{line:first_outlier,line:second_outlier} of \cref{alg:notspanner}. Furthermore, assume that the if condition on \cref{line:outlier} of the algorithm is true, and let $v\ind{t,i}$ be the output of \textsc{ComputeOutlierDirection} in \cref{line:computeoutlierdirection}.
  Then, conditioned on $\pihat\ind{t,i}$ and $\nout$, with probability at least $1-2\delta'$ (over the randomness of $v\ind{t,i}$), the following holds for i.i.d.\ random variables $\x_1, \x_1\ind{1}, \dots, \x_1\ind{\ndir}$:
  \begin{align}
     & \P\left[v\ind{t,i} \in \Span\{\phi_h^{\pihat\ind{t,i}}(\x\ind{1}_1)^\perp, \dots, \phi_h^{\pihat\ind{t,i}}(\x\ind{\ndir}_1)^\perp\} \mid \phi_h^{\pihat\ind{t,i}}(\x\ind{m}_1) \notin \cS\ind{t}_{h}, \forall m \in[\ndir] \right] \nn \\
     & \quad \ge  \mathbb{I}\left\{p \geq \frac{\veps}{8 Hd}\right\} \cdot\left(1 - \frac{d}{\ndir+1} - 2\sqrt{\frac{\log(2\ndir/\delta')}{2\mboost}} - \sqrt{\frac{\log(2/\delta')}{2\ndir}}\right),
    \label{eq:confidence_compute_outlier}
  \end{align}
  where $p = \P\left[\phi_h^{\pihat\ind{t,i}}(\x_1) \notin \cS\ind{t}_h\right]$ and $\phi_h^{\pihat\ind{t,i}}(\cdot)^\perp =\mathrm{Proj}_{(\cS\ind{t}_h)^\perp}(\phi_h^{\pihat\ind{t,i}}(\cdot))$ denotes the orthogonal projection of $\phi_h^{\pihat\ind{t,i}}(\cdot)$ onto $(\cS\ind{t}_h)^\perp$ (note that $\cS\ind{t}_h$ is a subspace).
\end{lemma}

\begin{proof}
  By \cref{lem:compute_outlier_direction} (the generic guarantee of \textsc{ComputeOutlierDirection}) and \cref{lem:transfer-rejection-Rd} (which accounts for the approximate sampling from $\P^{\pihat\ind{t,i}}[\cdot \mid \phi_h^{\pihat\ind{t,i}}(\x_1) \notin \cS\ind{t}_h]$ via rejection sampling), we get that with probability at least $1 - \delta' - (\ndir+\mboost \ndir^2) \cdot (1- p)^{\nrej}$ ($p$ is as in the lemma's statement),
  \begin{align}
     & \P\left[v\ind{t,i} \in \Span\{\phi_h^{\pihat\ind{t,i}}(\x\ind{1}_1)^\perp, \dots, \phi_h^{\pihat\ind{t,i}}(\x\ind{\ndir}_1)^\perp\} \mid \phi_h^{\pihat\ind{t,i}}(\x\ind{m}_1) \notin \cS\ind{t}_{h}, \forall m \in[\ndir] \right] \nn \\
     & \quad \ge  1 - \frac{d}{\ndir+1} - 2\sqrt{\frac{\log(2\ndir/\delta')}{2\mboost}} - \sqrt{\frac{\log(2/\delta')}{2\ndir}}.
  \end{align}
  When $p \geq \frac{\veps}{8 Hd}$, the choice of $\nrej$ in \eqref{eq:parameters} ensures that $(\ndir+\mboost \ndir^2) (1- p)^{\nrej} \le \delta'$, which implies \eqref{eq:confidence_compute_outlier} with the desired confidence level. When $p < \frac{\veps}{8 Hd}$, \eqref{eq:confidence_compute_outlier} holds trivially with probability 1.
\end{proof}

\begin{lemma}[Guarantee of \cref{alg:notspanner}]
  \label{lem:notspanner_guarantee}
  Consider a call to \lazyspanner with inputs $h$, $\Psi_{1:h-1}$, and $\veps, \delta$. Suppose the following:
  \begin{itemize}[leftmargin=*]
    \item \cref{ass:features,ass:bellman_complete} hold and $\cS_{1:h-1}$ are subspaces of $\reals^d$.
    \item For all $\ell\in[h-1]$, $u \in \cS_{\ell}$, and i.i.d.~$\x_1, \x_1\ind{1},\dots,\x_1\ind{\nspan}$ ($\nspan$ as in \eqref{eq:parameters}) initial states
          \begin{align}
            \label{eq:highcoverage_new}
            \P[u \in \Span\left\{\phi^\pi_\ell(\x_{1}\ind{m}): m\in[\nspan], \pi \in \Psi_{\ell}\right\}] \geq 1/2.
          \end{align}
    \item $\P[\x_1\in \cX\ind{h}]\geq 1- (h-1) \cdot \veps/H$, where $\cX\ind{h} \coloneqq \bigcap_{\ell\in[h-1]}\{x\in \cX \mid \Span\{\phi^\pi_\ell(x):\pi \in \Pi \} \subseteq \cS_\ell \}$.
  \end{itemize}
  Then, with probability at least $1-\delta$, \lazyspanner returns $(\cS\ind{t}_h,\Psi\ind{t}_h)$ such that:
  \begin{align}
    \P\left[ \Span \left\{ \phi^{\pi}_h(\x_1) : \pi \in \Pi\right\} \subseteq \cS_{h}\ind{t},\   \x_1 \in \cX\ind{h}  \right]  \geq  1-\frac{h\veps}{H}, \label{eq:lowsacrifice_nospanner}
  \end{align}
  and for all $u \in \cS_{h}\ind{t}$, and $\x_1, \x_1\ind{1},\dots,\x_1\ind{\nspan}$ i.i.d., we have
  \begin{align}
    \label{eq:highcoverage_nospanner}
    \P[u \in \Span\left\{\phi^\pi_h(\x_{1}\ind{m}): m \in [\nspan],\  \pi \in  \Psi\ind{t}_h \right\}] \geq \frac{1}{2}.
  \end{align}
\end{lemma}
\begin{proof}
  Suppose that \lazyspanner reaches step $i\geq 1$ of iteration $t\geq 1$, and let $\pihat\ind{t,i, \pm}$ be the policies returned by \textsc{FQI} on \cref{line:fqi_call} of \cref{alg:notspanner}. Let $\nout^{\pm}$ be as in \cref{alg:notspanner} and $(\ntest,\delta')$ be as in \eqref{eq:parameters}. We have the following:
  \begin{itemize}[leftmargin=*]
    \item
          From \eqref{eq:highcoverage_new}, the definition of $\nfqi$ in \eqref{eq:parameters}, and \cref{lem:uniform-amplify-subspace},
          we obtain the following amplified coverage statement: for each fixed $\ell\in[h-1]$,
          if $\x_1,\x_1\ind{1},\dots,\x_1\ind{\nfqi}$ are i.i.d., then
          \begin{align}
            \P\left[\cS_\ell \subseteq \Span\left\{\phi^\pi_\ell(\x_{1}\ind{m}): m\in[\nfqi], \pi \in \Psi_{\ell}\right\}\right] \ge 1-\frac{\delta'}{H^2}.
            \label{eq:reinforced_coverage}
          \end{align}
    \item
          Fix a step $(t,i)$ and recall that $\theta\ind{t,i}\in(\cS_h\ind{t})^\perp$ is the $i$-th basis vector chosen in that step.
          Let $\cE^{\texttt{fqi}}$ be the event that for all $x\in \cX\ind{h}$,
          \begin{align}
            \phi_h^{\pihat\ind{t,i,-}}(x)^\top \theta\ind{t,i}
            \le \min_{\pi \in \Pi} \phi_h^\pi(x)^\top \theta\ind{t,i}
            \quad \text{ and } \quad
            \max_{\pi \in \Pi} \phi_h^\pi(x)^\top \theta\ind{t,i}
            \le \phi_h^{\pihat\ind{t,i,+}}(x)^\top \theta\ind{t,i}.
            \label{eq:corfqi_new}
          \end{align}
          By \cref{cor:fqi} (applied to rewards $r_h(\cdot,\cdot)=\pm\phi_h(\cdot,\cdot)^\top \theta\ind{t,i}$ and $r_{1:h-1}\equiv 0$) and the coverage condition \eqref{eq:reinforced_coverage}, we have
          \begin{align}
            \P\big[\cE^{\texttt{fqi}}\big]\ge 1-\delta'.\label{eq:secondbullet}
          \end{align}
    \item
          Fix a step $(t,i)$ and condition on the realized outputs $\pihat\ind{t,i,-}$ and $\pihat\ind{t,i,+}$ of the two \textsc{FQI} calls on \cref{line:fqi_call}.
          Let $\cE^{\texttt{newdir}}$ be the event that
          \begin{equation}
\label{eq:notspanner_new}
\begin{gathered}
            \left|\frac{\nout^-}{\ntest} - \P\left[\phi_h^{\pihat\ind{t,i,-}}(\x_1)\notin \cS\ind{t}_h \right] \right|
            \le \sqrt{\frac{2\log(4/\delta')}{\ntest}}\\ 
            \text{and}      \qquad \qquad \qquad \qquad \qquad \qquad \qquad \qquad \qquad \qquad \qquad \qquad \qquad \qquad \qquad \qquad \\
            \left|\frac{\nout^+}{\ntest} - \P\left[\phi_h^{\pihat\ind{t,i,+}}(\x_1)\notin \cS\ind{t}_h \right] \right|
            \le \sqrt{\frac{2\log(4/\delta')}{\ntest}}.
          \end{gathered}
          \end{equation}
        
          By \cref{lem:notspanner}, for every fixed realization of $\pihat\ind{t,i,\pm}$,
          \begin{align}
            \P\big[\cE^{\texttt{newdir}} \mid \pihat\ind{t,i,-},\pihat\ind{t,i,+}\big]\ge 1-\delta'. \label{eq:thirdbullet}
          \end{align}
    \item
          Fix a step $(t,i)$ and condition on the realized policy $\pihat\ind{t,i}$ selected on \cref{line:first_outlier,line:second_outlier} (which is a deterministic function of $\pihat\ind{t,i,\pm}$ and of the $\ntest$ test trajectories used to compute $\nout^\pm$).
          If the condition on \cref{line:outlier} of \cref{alg:notspanner} is \emph{not} satisfied (i.e., the outlier branch does not execute), define $\cE^{\texttt{outlier}}$ to be the sure event.
          Otherwise, let $\cE^{\texttt{outlier}}$ be the event that
          \begin{align}
             & \P\left[v\ind{t,i} \in \Span\left\{\phi_h^{\pihat\ind{t,i}}(\x\ind{m}_1)^\perp: m \in [\ndir] \right\} \mid \phi_h^{\pihat\ind{t,i}}(\x\ind{m}_1) \notin \cS\ind{t}_{h}, \forall m \in[\ndir] \right] \nn \\
             & \quad  \ge  \mathbb{I}\left\{\P\left[\phi_h^{\pihat\ind{t,i}}(\x_1) \notin \cS\ind{t}_h\right] \ge \frac{\veps}{8Hd}\right\}
            \cdot \left(1-\frac{1}{4d}\right),
            \label{eq:camera}
          \end{align}
          where $v\ind{t,i}$ is the output of \textsc{ComputeOutlierDirection} in \cref{alg:notspanner} at step $(t,i)$.
          By \cref{lem:compute_outlier}, for every fixed realization of $\pihat\ind{t,i}$ and $\nout$,
          \begin{align}
            \P\big[\cE^{\texttt{outlier}} \mid \pihat\ind{t,i},\nout\big]\ge 1-2\delta'. \label{eq:fourthbullet}
          \end{align}
  \end{itemize}

  \paragraphi{Local consolidation within a fixed step $(t,i)$} Define the step success event
  \[
    \cE\ind{t,i}\coloneqq \cE^{\texttt{fqi}} \cap \cE^{\texttt{newdir}} \cap \cE^{\texttt{outlier}}.
  \]
  By a union bound,
  \[
    \P[(\cE\ind{t,i})^c]
    \le \P[(\cE^{\texttt{fqi}})^c] + \P[(\cE^{\texttt{newdir}})^c] + \P[(\cE^{\texttt{outlier}})^c].
  \]

  We bound the three terms on the right-hand side using the conditional guarantees stated above.
  First, by \eqref{eq:secondbullet}, we have $\P[(\cE^{\texttt{fqi}})^c]\le \delta'$.
  Second, by \eqref{eq:thirdbullet}, for every realization of $\pihat\ind{t,i,\pm}$ we have
  $\P[(\cE^{\texttt{newdir}})^c \mid \pihat\ind{t,i,\pm}] \le \delta'$.
  Third, by \eqref{eq:fourthbullet}, for every realization of $\pihat\ind{t,i}$ and $\nout$, we have
  $\P[(\cE^{\texttt{outlier}})^c \mid \pihat\ind{t,i},\nout] \le 2\delta'$ (and this bound holds trivially when the outlier branch does not execute, since $\cE^{\texttt{outlier}}$ is then the sure event).
  By the tower property,
  \begin{align}
    \P[(\cE^{\texttt{newdir}})^c]
    = \E\left[\P[(\cE^{\texttt{newdir}})^c \mid \pihat\ind{t,i,\pm}]\right]
   & \le \delta' \nn
   \shortintertext{and}
    \P[(\cE^{\texttt{outlier}})^c]
    = \E\left[\P[(\cE^{\texttt{outlier}})^c \mid \pihat\ind{t,i},\nout]\right]
    &\le 2\delta'.\nn 
  \end{align}
  Combining the bounds $\P[(\cE\ind{t,i})^c] \le \delta' + \delta' + 2\delta' = 4\delta'$, and hence
  \begin{align}
    \P[\cE\ind{t,i}] & \ge 1-4\delta'.
  \end{align}

  \paragraphi{Global consolidation across all steps of \cref{alg:notspanner}}
  We now combine the (local) success $\cE\ind{t,i}$ events across the entire execution of \cref{alg:notspanner}.
  In this context, a ``step'' means a pair $(t,i)$ corresponding to the $i$-th direction tested in outer iteration $t$.
  Recall that $t\in\{1,\dots,d+1\}$ and $K\ind{t}=d-\dim(\cS_h\ind{t})\le d$, hence the total number of potential steps is
  \begin{align}
    T \coloneqq \sum_{t=1}^{d+1} K\ind{t} \le (d+1)d \le 2d^2.
    \label{eq:T_bound}
  \end{align}
  Order the potential steps in the order they would be executed, and let $\tau\in[T]$ be the (random) number of steps actually executed before return; then $\tau$ is a stopping time for the natural filtration of the algorithm.
  Let $\cE$ be the event that $\cE^{(t,i)}$ holds for every step $(t,i)$ executed by \cref{alg:notspanner}.
  Since the bounds \eqref{eq:secondbullet}, \eqref{eq:thirdbullet}, and \eqref{eq:fourthbullet} are conditional on the algorithm state before each step, the step-level failure bound $\P[(\cE^{(t,i)})^c] \le 4\delta'$ holds conditioned on the prior history.
  By \cref{lem:sequnionbound-stopping},
  \[
    \P[\cE] \ge 1-4\delta' T.
  \]
  With $\delta'=\delta/(8 H d^2 N_{\texttt{iter}})$ (from \eqref{eq:parameters}) and $T\le 2d^2$, we have
  \[
    4\delta'T \le 4 \cdot \frac{\delta}{8 H d^2 N_{\texttt{iter}}} \cdot 2d^2 = \frac{\delta}{H N_{\texttt{iter}}} \le \frac{\delta}{H}.
  \]

  Moving forward, let $t_\star \in [d+1]$ denote the (random) iteration at which the algorithm returns, and condition on the event $\cE$.

  \paragraph{Proving \eqref{eq:highcoverage_nospanner} via induction.} We show by induction that for all $t\in [t_\star]$, $u \in \cS\ind{t}_h$, and i.i.d.~initial states $\x_1, \x_1\ind{1},\dots,\x_1\ind{\nspan}$, we have
  \begin{align}
    \P[u \in \Span\left\{ \phi^\pi_h(\x\ind{m}_1) : m \in [\nspan], \pi \in   \Psi\ind{t}_h \right\}]	\geq 1-\frac{(t-1)}{2d}.
  \end{align}

  \paragraphi{Base case} The base case $t=1$ holds trivially since $\cS\ind{1}_h = \Psi\ind{1}_h = \emptyset$ (see \cref{alg:notspanner}).

  \paragraphi{General case} Suppose the claim holds for $t\in [t_\star-1]$. We show it holds for $t+1$. Since $t < t_\star$, the algorithm does not return at iteration $t$, so case II (the if condition on \cref{line:outlier} is false for all $i \in [K\ind{t}]$) cannot hold.
  Therefore, the if condition on \cref{line:outlier} is satisfied for some $i\in [K\ind{t}]$ (recall that $K\ind{t} = d - \mathrm{dim}(\cS_h\ind{t})$). Fix this $i$. From \eqref{eq:notspanner_new}, the definition of $\pihat\ind{t,i}$ in \cref{line:first_outlier,line:second_outlier}, and the fact that the condition on \cref{line:outlier} is satisfied, we have that
  \begin{align}
    \P\left[\phi_h^{\pihat\ind{t,i}}(\x_1)\notin \cS\ind{t}_h \right] & \geq \frac{\veps}{4Hd} - \sqrt{\frac{2 \log (4/\delta')}{\ntest}}, \nn \\
                                                                      & \geq \frac{\veps}{8 Hd},  \label{eq:nature}
  \end{align}
  where the last inequality follows from the choice of $\ntest$ in \eqref{eq:parameters}.
  By \eqref{eq:nature}, \eqref{eq:camera}, \cref{lem:boost-conditional-span}, and the choice of $\nspan$ in \eqref{eq:parameters}, we have that
  \begin{align}
    \P\left[ v\ind{t,i} \in \Span\left\{\phi_h^{\pihat\ind{t,i}}(\x\ind{m}_1)^\perp: m \in [\nspan] \right\}  \right] \geq 1 - \frac{1}{2d}. \label{eq:sensor}
  \end{align}
  Now, fix $u \in \cS\ind{t+1}_h$. Because $\cS\ind{t+1}_h =\Span\{ \cS\ind{t}_h \cup\{ v\ind{t,i} \} \},$ we have that either $u \in \cS\ind{t}_h$ or $u$ has a non-zero component along $v\ind{t,i}$. In the former case, the induction hypothesis applies. In the latter case, we can write $u = v + w$, where $v \in \Span\{v\ind{t,i}\}$ and $w \in \cS\ind{t}_h$. By the induction hypothesis, we have that
  \begin{align}
    \P[w\in \Span\left\{ \phi^\pi_h(\x\ind{m}_1) : m \in [\nspan], \pi \in   \Psi\ind{t}_h \right\}]	\geq 1 -\frac{t-1}{2d}. \label{eq:breath}
  \end{align}
  Now, because $\Psi\ind{t+1}_h = \Psi\ind{t}_h \cup \{\pihat\ind{t,i}\}$, we have by \eqref{eq:sensor}, \eqref{eq:breath}, and a union bound that
  \begin{align}
    \P[u \in \Span\left\{ \phi^\pi_h(\x\ind{m}_1) : m \in [\nspan], \pi \in  \Psi\ind{t+1}_h \right\}]	\geq 1 - \frac{t}{2d}.
  \end{align}
  This completes the induction.

  \paragraph{Proving \eqref{eq:lowsacrifice_nospanner}.} It remains to prove that if the algorithm returns at iteration $t$, then
  \begin{align}
    \P\left[\Span\left\{ \phi^{\pi}_h(\x_1) : \pi \in \Pi\right\} \subseteq \cS\ind{t}_{h}, \  \x_1 \in \cX\ind{h}  \right]\geq  1-\frac{h\veps}{H}. \label{eq:lowsacrifice_new}
  \end{align}
  When the algorithm returns at iteration $t$, we either have that I) $\mathrm{dim}(\cS\ind{t}_h) = d$; or II) for all $i\in [K\ind{t}]$, the if condition on \cref{line:outlier} is not satisfied. In case I), \eqref{eq:lowsacrifice_new} holds trivially. We now consider the second case. The fact that the if condition on \cref{line:outlier} is not satisfied for all $i\in [K\ind{t}]$ implies
  \begin{align}
    \frac{\veps}{4 H d}  \geq 	\frac{\nout}{\ntest} \geq \frac{\nout^-}{\ntest}   \vee \frac{\nout^+}{\ntest}, \label{eq:inequalities}
  \end{align}
  where the right-most inequality follows from the definition of $\nout$ in \cref{alg:notspanner}.

  Combining \eqref{eq:inequalities} with \eqref{eq:notspanner_new}, we get that for all $i\in [K\ind{t}]$,
  \begin{align}
    \P\left[\phi_h^{\pihat\ind{t,i,-}}(\x_1)\notin \cS\ind{t}_h \right] \vee \P\left[\phi_h^{\pihat\ind{t,i,+}}(\x_1)\notin \cS\ind{t}_h \right] \leq  \frac{\veps}{4 d H} + \sqrt{\frac{2 \log (4/\delta')}{\ntest}}.
  \end{align}
  Note that because $(\theta\ind{t,j})_{j\in [K\ind{t}]}$ are an orthogonal basis of $(\cS\ind{t}_h)^\perp$, we have that for all $\pi \in \Pi$,
  \begin{align}
    \phi_h^{\pi}(x) \notin \cS\ind{t}_h  \iff   \max_{i\in[K\ind{t}]} \left| \phi_h^{\pi}(x)^\top \theta\ind{t,i} \right| \neq 0.
  \end{align}
  Now, by \eqref{eq:corfqi_new}, we have that for all $x\in \cX\ind{h}$ and $\pi \in \Pi$,
  \begin{align}
    \max_{i \in [K\ind{t}]} \left|\phi_h^{\pi}(x)^\top \theta\ind{t,i}\right| \neq 0  \implies \left( \min_{i \in [K\ind{t}]} \phi_h^{\pihat\ind{t,i,-}}(x)^\top \theta\ind{t,i}< 0  \text{ or }  \max_{i \in [K\ind{t}]} \phi_h^{\pihat\ind{t,i,+}}(x)^\top \theta\ind{t,i}>0 \right).
  \end{align}
  Therefore, we have that
  \begin{align}
     & \P\left[ \Span\{\phi_h^\pi(\x_1):\pi \in \Pi\} \not\subseteq \cS\ind{t}_h\text{ and } \x_1  \in \cX\ind{h}\right]\nn                                                                                                                                                         \\ &  = \P\left[\max_{i\in [K\ind{t}]} \sup_{\pi \in \Pi} \left|\phi_h^{\pi}(\x_1)^\top  \theta\ind{t,i} \right|\neq 0 \text{ and } \x_1 \in \cX\ind{h} \right],\nn \\
     & \leq \P\left[\max_{i\in[K\ind{t}]} \max\left(\phi_h^{\pihat\ind{t,i,+}}(\x_1)^\top \theta\ind{t,i}, -\phi_h^{\pihat\ind{t,i,-}}(\x_1)^\top \theta\ind{t,i} \right) >0 \text{ and } \x_1  \in \cX\ind{h} \right], \nn                                                      \\
     & \leq \sum_{i\in[K\ind{t}]} \P\left[\phi_h^{\pihat\ind{t,i,+}}(\x_1)^\top \theta\ind{t,i}>0 \text{ and } \x_1  \in \cX\ind{h}\right] + \sum_{i\in[K\ind{t}]} \P\left[\phi_h^{\pihat\ind{t,i,-}}(\x_1)^\top \theta\ind{t,i}<0 \text{ and } \x_1  \in \cX\ind{h}\right] ,\nn \\
     & \leq \sum_{i\in[K\ind{t}]} \P\left[\phi_h^{\pihat\ind{t,i,+}}(\x_1) \notin \cS\ind{t}_h \text{ and } \x_1  \in \cX\ind{h}\right] + \sum_{i\in[K\ind{t}]} \P\left[\phi_h^{\pihat\ind{t,i,-}}(\x_1) \notin \cS\ind{t}_h \text{ and } \x_1  \in \cX\ind{h}\right],\nn        \\
    \shortintertext{and using that $K\ind{t}\leq d$, we get}
     & \leq  \frac{\veps}{2H} + 2 d \sqrt{\frac{2 \log (4/\delta')}{\ntest}}, \nn                                                                                                                                                                                                \\
     & \leq \frac{\veps}{H}, \label{eq:notincluded}
  \end{align}
  where the last inequality follows from the choice of $\ntest$ in \eqref{eq:parameters}.
  Therefore, we have that
  \begin{align}
   & \P\left[ \Span\{\phi_h^\pi(\x_1):\pi \in \Pi\} \subseteq \cS\ind{t}_h\text{ and } \x_1  \in \cX\ind{h}\right] \nn \\& = \P[\x_1 \in \cX\ind{h}] - \P\left[ \Span\{\phi_h^\pi(\x_1):\pi \in \Pi\} \not\subseteq \cS\ind{t}_h\text{ and } \x_1  \in \cX\ind{h}\right], \nn \\
                                                                                                                  & \geq 1 - (h-1) \frac{\veps}{H} - \frac{\veps}{H},
  \end{align}
  where the last inequality follows from \eqref{eq:notincluded}. This completes the proof.
\end{proof}

\section{Guarantee of $\textsc{ComputeSpanner}_{\texttt{fqi}}$ (\cref{alg:cover})}
\label{sec:spanner_guarantee_main}

This section provides the proof of \cref{thm:cover_guarantee}, which establishes that \cref{alg:cover} computes valid policy covers with both the coverage property (for $\Psi_{1:H}$) and the approximate barycentric spanner property (for $\Gamma_{1:H}$). The proof proceeds in two steps: first, we show that \textsc{FQI} with coverage policies serves as a valid approximate linear optimization oracle (\cref{lem:fqi_as_apx}), enabling large action space algorithms to avoid the polynomial-in-$|\cA|$ cost of \textsc{OCP}; then, we verify that the oracles used by \apxspanner satisfy the preconditions of \cref{prop:spanner}, yielding the spanner guarantee (\cref{lem:robust_spanner_guarantee}). The final result follows by induction over layers.

We first show that the \FQI and \veceval{} subroutines used in \cref{alg:cover} satisfy the requirements of \cref{ass:spanner} (guarantee of \apxspanner), with a small error influenced by the state concentration property.

For subspaces $\cS_{1:h-1}\subseteq \reals^d$ and policy collections $\Psi_{1:h-1}$, define the \emph{coverage event}
\begin{align}
  \cE^{\texttt{cov}}_{h} \coloneqq \bigcap_{\ell=1}^{h-1} \left\{ \cS_\ell \subseteq \Span\left\{\phi^\pi_\ell(\x_{1}\ind{m}): m\in[\ncover], \pi \in \Psi_{\ell}\right\} \right\}, \label{eq:coverage_event}
\end{align}
where $\x_1\ind{1},\dots,\x_1\ind{\ncover}$ are i.i.d.\ samples from the initial state distribution.

\begin{lemma}[Validity of \FQI as $\apx$]
	\label{lem:fqi_as_apx}
	Let $\veps > 0$ and $\delta \in (0,1)$ be the accuracy and confidence parameters input to \cref{alg:cover}, and consider Step $h$ of the algorithm.
	Let $\cS_{1:h-1}$ and $\Psi_{1:h-1}$ be the subspaces and policy collections computed by \cref{alg:cover} prior to Step $h$.
	For $\theta \in \reals^d \setminus\{0\}$, let $\apx(\theta)$ correspond to the call $\FQI(h, r_{1:h},\Psi_{1:h-1}, \nfqi)$ with rewards $r_{1:h-1}\equiv 0$ and $r_h(\cdot,\cdot) = \theta^\top \phi_h(\cdot,\cdot)$.
	Let $\delta'$ be as in \eqref{eq:parameters}, and let $\cX\ind{h}$ be defined as in \eqref{eq:thm_Xdef} (in terms of $\cS_{1:h-1}$).
	Suppose \cref{ass:features,ass:bellman_complete} hold.
	
	Conditioned on the coverage event $\cE^{\texttt{cov}}_{h}$ (defined in \eqref{eq:coverage_event}), for each fixed $\theta \in \reals^d\setminus\{0\}$, with probability at least $1-\delta'$ over the samples used by \FQI, the policy $\pihat = \apx(\theta)$ satisfies:
	\begin{align}
		\E[\phi_h^{\pihat}(\x_1)^\top \theta] \ge \max_{\pi \in \Pi} \E[\phi_h^{\pi}(\x_1)^\top \theta] - 2 \|\theta\| \cdot \P[\x_1 \notin \cX\ind{h}].
	\end{align}
\end{lemma}
\begin{proof}
    Condition on $\cE^{\texttt{cov}}_{h}$. Then for all $\ell \in [h-1]$, we have $\cS_\ell \subseteq \Span\{\phi^\pi_\ell(\x_{1}\ind{m}): m\in[\ncover], \pi \in \Psi_{\ell}\}$. This is precisely the subspace coverage condition required by \cref{lem:fqi} (the guarantee of \FQI; specifically \cref{item:span} with $n=\ncover$ samples).
    Therefore, \cref{lem:fqi} applies, and on its success event (which holds with probability at least $1-\delta'$ over the FQI samples), we have for all $x \in \cX\ind{h}$:
    $\phi_h^{\pihat}(x)^\top \theta \ge \max_{\pi \in \Pi} \phi_h^{\pi}(x)^\top \theta$.
    Condition on this success event for the remainder of the proof.
	Taking expectations over $\x_1$ and splitting based on $\{\x_1 \in \cX\ind{h}\}$:
	\begin{align}
		&\E[\phi_h^{\pihat}(\x_1)^\top \theta]\nn \\
     & = \E[\phi_h^{\pihat}(\x_1)^\top \theta \cdot \mathbb{I}\{\x_1 \in \cX\ind{h}\}] + \E[\phi_h^{\pihat}(\x_1)^\top \theta \cdot \mathbb{I}\{\x_1 \notin \cX\ind{h}\}] ,                                         \nn \\
		                                       & \ge \E\left[\left(\max_{\pi \in \Pi} \phi_h^{\pi}(\x_1)^\top \theta\right) \cdot \mathbb{I}\{\x_1 \in \cX\ind{h}\}\right] - \|\theta\| \cdot \P[\x_1 \notin \cX\ind{h}],                                                    \nn    \\
		                                       & = \E\left[\max_{\pi \in \Pi} \phi_h^{\pi}(\x_1)^\top \theta\right] - \E\left[\left(\max_{\pi \in \Pi} \phi_h^{\pi}(\x_1)^\top \theta\right) \cdot \mathbb{I}\{\x_1 \notin \cX\ind{h}\}\right] - \|\theta\| \cdot \P[\x_1 \notin \cX\ind{h}],\nn \\
		                                       & \ge \max_{\pi \in \Pi} \E\left[\phi_h^{\pi}(\x_1)^\top \theta\right] - 2 \|\theta\| \cdot \P[\x_1 \notin \cX\ind{h}],\nn 
	\end{align}
	where we used $\sup_{(x,a)}\|\phi_h(x,a)\| \le 1$ (from \cref{ass:features}).
\end{proof}

\begin{lemma}[Validity of \veceval{} as $\est$]
	\label{lem:veceval_as_est}
	Consider the setting of Step $h$ in \cref{alg:cover} with accuracy parameter $\veps>0$ and confidence parameter $\delta \in (0,1)$. Let $\delta'$ be as in \eqref{eq:parameters}.
	Let $\est(\pi)$ correspond to the call $\veceval(h, \phi_h, \pi, \nveceval)$ with $\nveceval$ as in \cref{alg:cover}.
	Then, for each fixed $\pi \in \Pi$, with probability at least $1-\delta'$ over the samples used by \veceval:
	\begin{align}
		\| \est(\pi) - \E[\phi_h^\pi(\x_1)] \| \le \veps_{\rob}.
	\end{align}
\end{lemma}
\begin{proof}
	This follows directly from \cref{lem:veceval} (guarantee of \veceval) and the choice of $\nveceval$ in \eqref{eq:parameters}, which ensures the error is at most $\veps_{\rob}$ with probability $1-\delta'$.
\end{proof}
\begin{lemma}[Guarantee of \apxspanner for $\textsc{ComputeSpanner}_{\texttt{fqi}}$]
  \label{lem:robust_spanner_guarantee}
  Let $\veps > 0$ and $\delta \in (0,1)$ be the accuracy and confidence parameters input to \cref{alg:cover}, and consider the call to \apxspanner at Step $h$.
  Let $\cS_{1:h-1}$ and $\Psi_{1:h-1}$ be the subspaces and policy collections computed by \cref{alg:cover} prior to Step $h$.
  Let $\cX\ind{h}$ be defined as in \eqref{eq:thm_Xdef} (in terms of $\cS_{1:h-1}$), and let $\delta'$, $N_{\texttt{iter}}$, $\veps_{\rob}$ be as in \eqref{eq:parameters}.
  Suppose \cref{ass:features,ass:bellman_complete} hold.
  
  Conditioned on the coverage event $\cE^{\texttt{cov}}_{h}$ (defined in \eqref{eq:coverage_event}), if $\P[\x_1 \in \cX\ind{h}] \ge 1 - \veps_{\texttt{conc}}$ for some $\veps_{\texttt{conc}} \ge 0$, then with probability at least $1-\frac{\delta}{4H}$ over the samples used by \apxspanner, its output $\Gamma_h = (\pi_1,\dots,\pi_d)$ satisfies the following: for any $\pi \in \Pi$, there exist $\beta_{1},\dots,\beta_d\in[-2,2]$ such that
  \begin{align}
    \nrm*{\E[\phi_h^{\pi}(\x_1)] - \sum_{j=1}^d \beta_j \E[\phi_h^{\pi_j}(\x_1)]} \leq 24 d \veps_{\texttt{conc}} + \frac{\veps}{10 H^2}. \label{eq:local_spanner}
  \end{align}
\end{lemma}
\begin{proof}
  Condition on $\cE^{\texttt{cov}}_{h}$ throughout.
  We verify that the oracles $\apx$ and $\est$ satisfy the requirements of \cref{ass:spanner} (the preconditions of \cref{prop:spanner}, the guarantee of \apxspanner).
  
  By \cref{prop:spanner}, \apxspanner terminates after at most $N_{\texttt{iter}}$ iterations and makes at most $2N_{\texttt{iter}}$ calls to each of $\apx$ and $\est$.
  
  \emph{Verifying approximate linear optimization (first condition of \cref{ass:spanner}).}
  Since we condition on $\cE^{\texttt{cov}}_{h}$, the preconditions of \cref{lem:fqi_as_apx} are satisfied. Thus, for each call to $\apx$ with input $\theta \neq 0$, \cref{lem:fqi_as_apx} implies that with probability $1-\delta'$ over the samples used:
  \[ \E[\phi_h^{\pihat}(\x_1)^\top \theta] \ge \max_{\pi \in \Pi} \E[\phi_h^{\pi}(\x_1)^\top \theta] - 2 \|\theta\|\cdot  \P[\x_1 \notin \cX\ind{h}]. \]
  Given $\P[\x_1 \in \cX\ind{h}] \ge 1 - \veps_{\texttt{conc}}$, the right-hand side is at least $\max_{\pi} \E[\phi_h^{\pi}(\x_1)^\top \theta] - 2\|\theta\| \veps_{\texttt{conc}}$, so $\apx$ satisfies the approximate linear optimization condition of \cref{ass:spanner} with error $\veps_{\rob} = 2 \veps_{\texttt{conc}}$.
  
  \emph{Verifying vector estimation (second condition of \cref{ass:spanner}).}
  For each call to $\est$ with input $\pi$, \cref{lem:veceval_as_est} guarantees that with probability $1-\delta'$ over the samples used, $\|\est(\pi) - \E[\phi_h^\pi(\x_1)]\| \le \veps_{\rob}$ (as defined in \eqref{eq:parameters}), so $\est$ satisfies the vector estimation condition of \cref{ass:spanner} with error $\veps_{\texttt{est}} = \veps_{\rob}$.
  
  \emph{Union bound over oracle calls.}
  Let $\cE^{\texttt{orc}}$ denote the event that all $2N_{\texttt{iter}}$ calls to $\apx$ satisfy the first condition of \cref{ass:spanner} (approximate linear optimization) and all $2N_{\texttt{iter}}$ calls to $\est$ satisfy the second condition of \cref{ass:spanner} (vector estimation).
  Each individual oracle call succeeds with probability at least $1-\delta'$ (conditioned on $\cE^{\texttt{cov}}_{h}$). By a union bound over the $4N_{\texttt{iter}}$ total oracle calls:
  \begin{align}
    \P[\cE^{\texttt{orc}} \mid \cE^{\texttt{cov}}_{h}] \ge 1-4N_{\texttt{iter}}\delta' = 1 - \frac{\delta}{2 H d^2},
  \end{align}
  where the equality uses $\delta' = \delta/(8 H d^2 N_{\texttt{iter}})$ from \eqref{eq:parameters}.
  
  \emph{Applying \cref{prop:spanner}.}
  On the event $\cE^{\texttt{orc}}$, \cref{ass:spanner} holds with combined error $\veps' = \veps_{\rob} + \veps_{\texttt{est}} \le 2\veps_{\texttt{conc}} + \veps_{\rob}$.
  \cref{alg:cover} invokes \apxspanner with inputs $(\apx, \est, C=2, \veps_{\rob})$.
  Applying \cref{prop:spanner} with these parameters (setting the proposition's accuracy parameter to $\veps_{\rob}$) and oracle error $\veps'$, the output $\Gamma_h$ is a 2-approximate barycentric spanner with error:
  \[ 3Cd\left(\veps_{\rob} + 2\veps'\right) \le 6d \left(\veps_{\rob} + 2\left(2\veps_{\texttt{conc}} + \veps_{\rob}\right)\right) = 6d \left(3\veps_{\rob} + 4\veps_{\texttt{conc}}\right). \]
  Simplifying: $6d \cdot 3\veps_{\rob} = 18d \cdot \frac{\veps}{192 d H^2} = \frac{18 \veps}{192 H^2} \le \frac{\veps}{10 H^2}$ and $6d \cdot 4\veps_{\texttt{conc}} = 24 d \veps_{\texttt{conc}}$.
  This gives the bound in \eqref{eq:local_spanner}.
\end{proof}

\begin{proof}[Proof of \cref{thm:cover_guarantee}]
  For each $h\in[H]$, let $\cE_h^{\texttt{lazy}}$ denote the success event of the call
  \[
    (\cS_h,\Psi_h)\gets \lazyspanner\left(h,\Psi_{1:h-1},\veps,\tfrac{\delta}{4H}\right),
  \]
  within \cref{alg:cover}; namely the event on which the conclusions of \cref{lem:notspanner_guarantee} (guarantee of \lazyspanner) hold for that call.
  Let $\cE_h^{\texttt{cov}}$ denote the event that for i.i.d.\ samples $\x_1\ind{1},\dots,\x_1\ind{\ncover}$ from the initial state distribution, the following holds simultaneously for all $\ell \in [h-1]$:
  \begin{align}
    \cS_\ell \subseteq \Span\left\{\phi^\pi_\ell(\x_{1}\ind{m}): m\in[\ncover], \pi \in \Psi_{\ell}\right\}. \label{eq:cov_event_proof}
  \end{align}
  Let $\cE_h^{\texttt{rob}}$ denote the event that the conclusions of \cref{lem:robust_spanner_guarantee} (guarantee of \apxspanner) hold for the call to \apxspanner at Step $h$.
  Define $\cE_h \coloneqq \cE_h^{\texttt{lazy}} \cap \cE_h^{\texttt{cov}} \cap \cE_h^{\texttt{rob}}$.
  Let $\cE_{<h}\coloneqq \bigcap_{\ell=1}^{h-1}\cE_\ell$, with $\cE_{<1}$ being the sure event.

  For any $h \in [H]$, let $\cH_h$ denote the hypothesis that \cref{alg:cover}'s variables $(\cS_h, \Psi_h, \Gamma_h)$ satisfy:
  \begin{enumerate}[leftmargin=*]
      \item \label{item:coverage_h} \emph{Coverage at $h$:} For i.i.d.~$\x_1, \x_1\ind{1},\dots,\x_1\ind{\ncover}$, $\P[\cS_h \subseteq \Span\{\phi^\pi_h(\x_{1}\ind{m}) : m\in[\ncover], \pi\in\Psi_h\}] \ge 1-\delta/H^2$.
      \item \label{item:state_concentration_h} \emph{State concentration at $h+1$:} $\P[\x_1 \in \cX\ind{h+1}] \ge 1 - h\veps/H$, where $\cX\ind{h}$ is defined in \eqref{eq:thm_Xdef}.
      \item \label{item:spanner_h} \emph{Spanner property at $h$:} The property in \eqref{eq:cover_spanner_prop} holds for $\Gamma_h$.
  \end{enumerate}

  We proceed by induction to show that for all $h \in [H]$, the event $\cE_h$ occurs with high probability and implies $\cH_h$.
  The key to the induction is that success at previous steps ($\cH_1, \dots, \cH_{h-1}$) ensures that the inputs passed to the subroutines \lazyspanner and \apxspanner at step $h$ satisfy the conditions of \cref{lem:notspanner_guarantee} (guarantee of \lazyspanner) and \cref{lem:robust_spanner_guarantee} (guarantee of \apxspanner), respectively.
  Specifically:
  \begin{itemize}[leftmargin=*]
      \item The \emph{Coverage} properties (\cref{item:coverage_h}) established in steps $1, \dots, h-1$ ensure that the preconditions for invoking \cref{lem:notspanner_guarantee} at step $h$ are met.
      \item The \emph{State concentration} property (\cref{item:state_concentration_h}) established in step $h-1$ ensures that the preconditions for invoking \cref{lem:robust_spanner_guarantee} at step $h$ are met.
  \end{itemize}

  \paragraph{Base case ($h=1$).}
  At step $h=1$, the inputs $(\cS_0, \Psi_0)$ are empty, so the coverage requirements for \lazyspanner are vacuously satisfied, and $\cE_1^{\texttt{cov}}$ is the sure event.
  Also, $\cX\ind{1} = \cX$, so $\P[\x_1 \in \cX\ind{1}] = 1$, satisfying the state concentration requirement in \cref{lem:robust_spanner_guarantee}.
  By \cref{lem:notspanner_guarantee}, $\P[(\cE_1^{\texttt{lazy}})^c] \le \delta/(4H)$.
  Since $\cE_1^{\texttt{cov}}$ is the sure event, \cref{lem:robust_spanner_guarantee} gives $\P[(\cE_1^{\texttt{rob}})^c] = \P[(\cE_1^{\texttt{rob}})^c \mid \cE_1^{\texttt{cov}}] \le \delta/(4H)$.
  By a union bound, $\P[(\cE_1)^c] \le \delta/(2H)$.
  On the event $\cE_1$, the subroutines succeed, which directly implies $\cH_1$: \cref{item:coverage_h,item:state_concentration_h} hold by \cref{lem:notspanner_guarantee}, and \cref{item:spanner_h} holds by \cref{lem:robust_spanner_guarantee}.

  \paragraph{Inductive step.}
  Fix $h \in \{2, \dots, H\}$. Assume that on the event $\cE_{<h} = \bigcap_{\ell=1}^{h-1} \cE_\ell$, the hypotheses $\cH_1, \dots, \cH_{h-1}$ hold.
  We verify that the algorithm can successfully execute step $h$:
  \begin{itemize}[leftmargin=*]
      \item \emph{Valid inputs for \lazyspanner:} To invoke \cref{lem:notspanner_guarantee} for \lazyspanner, we require that the previous outputs $(\cS_\ell, \Psi_\ell)_{\ell<h}$ satisfy the coverage property in \cref{item:coverage_h}. This is exactly established by the inductive hypotheses $\cH_1, \dots, \cH_{h-1}$.
      \item \emph{Valid distribution for \apxspanner:} To invoke \cref{lem:robust_spanner_guarantee} for \apxspanner, we require that the state distribution is concentrated on $\cX\ind{h}$. \cref{item:state_concentration_h} of $\cH_{h-1}$ gives exactly this: $\P[\x_1 \in \cX\ind{h}] \ge 1 - (h-1)\veps/H$.
      \item \emph{Coverage event for \apxspanner:} \cref{lem:robust_spanner_guarantee} requires conditioning on $\cE_h^{\texttt{cov}}$. On $\cE_{<h}$, by \cref{item:coverage_h} of $\cH_1, \dots, \cH_{h-1}$, each layer $\ell \in [h-1]$ satisfies $\P[\text{\eqref{eq:cov_event_proof} holds for layer } \ell] \ge 1-\delta/H^2$. By a union bound over the $(h-1)$ layers:
      \begin{align}
        \P[(\cE_h^{\texttt{cov}})^c \mid \cE_{<h}] \le (h-1)\delta/H^2. \label{eq:cov_bound}
      \end{align}
  \end{itemize}
  
  Since the requirements are met (on $\cE_{<h}$), the subroutine guarantees apply:
  \begin{itemize}[leftmargin=*]
    \item By \cref{lem:notspanner_guarantee}: $\P[(\cE_h^{\texttt{lazy}})^c \mid \cE_{<h}] \le \delta/(4H)$.
    \item By \eqref{eq:cov_bound}: $\P[(\cE_h^{\texttt{cov}})^c \mid \cE_{<h}] \le (h-1)\delta/H^2$.
    \item By \cref{lem:robust_spanner_guarantee}: $\P[(\cE_h^{\texttt{rob}})^c \mid \cE_h^{\texttt{cov}} \cap \cE_{<h}] \le \delta/(4H)$.
  \end{itemize}
  By a union bound:
  \begin{align}
    \P\left[(\cE_h)^c \mid \cE_{<h}\right] &\le \P[(\cE_h^{\texttt{lazy}})^c \mid \cE_{<h}] + \P[(\cE_h^{\texttt{cov}})^c \mid \cE_{<h}] + \P[(\cE_h^{\texttt{rob}})^c \cap \cE_h^{\texttt{cov}} \mid \cE_{<h}] \notag \\
    &\le \P[(\cE_h^{\texttt{lazy}})^c \mid \cE_{<h}] + \P[(\cE_h^{\texttt{cov}})^c \mid \cE_{<h}] + \P[(\cE_h^{\texttt{rob}})^c \mid \cE_h^{\texttt{cov}} \cap \cE_{<h}] \notag \\
    &\le \frac{\delta}{4H} + \frac{(h-1)\delta}{H^2} + \frac{\delta}{4H}. \label{eq:per_step_failure}
  \end{align}
  
  Conditioned on $\cE_h$ occurring, we establish $\cH_h$:
  \begin{itemize}[leftmargin=*]
      \item \emph{Coverage at $h$:} \cref{lem:notspanner_guarantee} guarantees that for any $u \in \cS_h$, $\P[u \in \Span\{\phi^\pi_h(\x_{1}\ind{m}) : m\in[\nspan], \pi\in\Psi_h\}] \ge 1/2$.
      Since $\ncover = \nspan \cdot \lceil \log_2(H^2 d/\delta') \rceil$, \cref{lem:uniform-amplify-subspace} implies that $\P[\cS_h \subseteq \Span\{\phi^\pi_h(\x_{1}\ind{m}) : m\in[\ncover], \pi\in\Psi_h\}] \ge 1-\delta/H^2$, establishing \cref{item:coverage_h}.
      \item \emph{State concentration at $h+1$:} \cref{lem:notspanner_guarantee} ensures that
      \[ \P\left[ \Span \left\{ \phi^{\pi}_h(\x_1) : \pi \in \Pi\right\} \subseteq \cS_{h},\   \x_1 \in \cX\ind{h}  \right]  \geq  1-\frac{h\veps}{H}. \]
      By definition, $\cX\ind{h+1} = \cX\ind{h} \cap \{x : \Span \left\{ \phi^{\pi}_h(x) : \pi \in \Pi\right\} \subseteq \cS_h\}$.
      Thus, the left-hand side equals $\P[\x_1 \in \cX\ind{h+1}]$, establishing \cref{item:state_concentration_h}.
      \item \emph{Spanner property at $h$:} By \cref{lem:robust_spanner_guarantee}, the output satisfies the property with error at most $24 d \veps_{\texttt{conc}} + \frac{\veps}{10 H^2}$, where $\veps_{\texttt{conc}} = (h-1)\veps/H \le \veps$. Thus, the total error is bounded by $24 d \veps + \veps \le 25 d \veps$, establishing \cref{item:spanner_h}.
  \end{itemize}

  \paragraph{Global success probability.}
  Summing \eqref{eq:per_step_failure} over all $h \in [H]$:
  \begin{align}
    \P\left[\left(\bigcap_{h=1}^H \cE_h\right)^c\right]
    &\le \sum_{h=1}^H \P[(\cE_h)^c \mid \cE_{<h}] \notag \\
    &\le \sum_{h=1}^H \left(\frac{\delta}{4H} + \frac{(h-1)\delta}{H^2} + \frac{\delta}{4H}\right) \notag \\
    &= H \cdot \frac{\delta}{2H} + \frac{\delta}{H^2} \sum_{h=1}^H (h-1) \notag \\
    &= \frac{\delta}{2} + \frac{\delta}{H^2} \cdot \frac{H(H-1)}{2} \notag \\
    &= \frac{\delta}{2} + \frac{(H-1)\delta}{2H} < \delta.
    \label{eq:cover_global_prob}
  \end{align}
  Thus $\P[\bigcap_{h=1}^H \cE_h] \ge 1-\delta$.
  
  On this event, $\cH_h$ holds for all $h \in [H]$, and \cref{item:thm_coverage,item:thm_concentration,item:thm_spanner} of the lemma follow from \cref{item:coverage_h,item:state_concentration_h,item:spanner_h} of $\cH_h$ respectively.
\end{proof}
\section{Guarantee of \cref{alg:fqi-stochastic-appendix} ($\textsc{PolicyOpt}_{\texttt{fqi}}$)}
\label{sec:policyopt-fqi}
\begin{lemma}[Guarantee of $\textsc{PolicyOpt}_{\texttt{fqi}}$]
	\label{lem:policy-opt}
	Fix $\delta \in (0,1)$ and let $\veps_{\texttt{span}}>0$ be given. Consider a call to $\textsc{PolicyOpt}_{\texttt{fqi}}(\Psi_{1:H-1}, \Gamma_{1:H}, n, \lambda)$ (\cref{alg:fqi-stochastic-appendix}).
    Further, let $\cS_{1:H}\subseteq \reals^d$ be a collection of subspaces and define $
      \cX\ind{H} \coloneqq \bigcap_{\ell \in[H-1]}\left\{x\in \cX \mid  \Span \left\{ \phi^{\pi}_\ell(x) : \pi \in \Pi\right\} \subseteq \cS_{\ell}\right\}.$
    Suppose:
    \begin{enumerate}[leftmargin=*]
        \item \cref{ass:features,ass:bellman_complete} hold.
        \item The collection of policies $\Psi_{1:H-1}$ satisfies the condition: for all $\ell\in[H-1]$ and i.i.d.~$\x_1, \x_1\ind{1},\dots,\x_1\ind{n}$:
          \begin{align}
            \P[\cS_\ell \subseteq \Span\left\{\phi^\pi_\ell(\x_{1}\ind{m}): m\in[n], \pi \in \Psi_{\ell}\right\}] \geq 1-\frac{\delta}{H^2}.
            \label{eq:policyopt-psi-cov}
          \end{align}
        \item \label{itemLspanners} The collection $\Gamma_{1:H}$ satisfies: for all $\ell\in[H]$ and $\rho\in\Pi$, there exists a coefficient map
    	$\beta_{\ell}: \Gamma_\ell \to [-2,2]$ such that
    	\begin{align}
    		\left\|
    		\E^{\rho}[\phi_\ell(\x_\ell,\a_\ell)]
    		-
    		\sum_{\pi\in\Gamma_\ell} \beta_{\ell}(\pi)\cdot  \E^{\pi}[\phi_\ell(\x_\ell,\a_\ell)]
    		\right\|
    		\le
    		\veps_{\texttt{span}}.
    		\label{eq:fqi-sto-mean-span}
    	\end{align}
        \item The collection of policies $\Psi_{1:H-1}$ and $\Gamma_{1:H}$ satisfy $|\Psi_h|\leq d$ for all $h\in[H-1]$ and $|\Gamma_h|\leq d$ for all $h\in[H]$.
        \item We have $\P[\x_1 \in \cX\ind{H}] \ge 1 - \veps_{\texttt{conc}}$ for some $\veps_{\texttt{conc}} \ge 0$.
    \end{enumerate}

	Then, with probability at least $1-3\delta$, for every $\pi\in\Pi$, the output policy $\pihat_{1:H}$ satisfies:
	\begin{align}
		\E\left[V_1^\pi(\x_1)-V_1^{\pihat}(\x_1)\right]
		\le
		\sum_{\ell=1}^H \left( 2 \sqrt{\frac{d}{n}} \cdot \zeta(\delta) + 2d \veps_{\texttt{stat}} + \left(\sqrt{d}+C\right)\veps_{\texttt{span}} \right)
        + H \cdot C \cdot \veps_{\texttt{conc}},
		\label{eq:fqi-sto-final}
	\end{align}
    where $\zeta(\delta) \coloneqq \sqrt{\lambda d} + \sqrt{2\log(H/\delta) + d\log(1 + n/\lambda)}$, $C \coloneqq \sqrt{d} + \lambda^{-1/2}\zeta(\delta)$, and
    \begin{align}
        \veps_{\texttt{stat}} \coloneqq C\cdot \sqrt{\frac{8 \left(d \log(1+2nC) + \log(2Hd/\delta)\right)}{n}} + \frac{2}{n}.
        \label{eq:stat-error}
    \end{align}
\end{lemma}

\begin{proof}
    We write $V_1^\pi(\cdot;r^\star_{1:H})$ for the value under the MDP's rewards and $V_1^\pi(\cdot; \wh r_{1:H})$ for the value under the deterministic proxy rewards $\wh r_{1:H}$ in \cref{alg:fqi-stochastic-appendix} (see \cref{sec:prelims} for the notation).
    By linearity of expectation, for any $\pi \in \Pi$,
    \begin{align}
        \E[V_1^\pi(\x_1;r^\star_{1:H}) - V_1^\pi(\x_1;\wh r_{1:H})]
        = \sum_{\ell=1}^H \E^\pi[\br_\ell - \wh r_\ell(\x_\ell,\a_\ell)].
        \label{eq:policyopt-decomp}
    \end{align}

    \paragraph{Reward estimation bound.}
    The algorithm computes $\wh w_\ell = (\lambda I + \Sigma_\ell)^{-1} \sum_{(x,a,r) \in \cD_\ell} \phi_\ell(x,a) \cdot r$, where
    \[
        \Sigma_\ell = \sum_{(x,a) \in \cD_\ell} \phi_\ell(x,a) \phi_\ell(x,a)^\top.
    \]
    Further, let $V_\ell = \lambda I + \Sigma_\ell$.
    By \cref{ass:features} (reward linearity), for any $\ell \in [H]$ and $(x,a)\in \cX \times \cA$, there exists $w_\ell \in \reals^d$ such that
    $\E[\bm{r}_\ell\mid \x_\ell =x, \a_\ell =a] = \phi_\ell(x,a)^\top w_\ell$.
  We now apply \cref{lem:ridge-ellipsoid} (ridge regression guarantee) by identifying the sequence $(A_k, X_k)$ with the samples $(\phi_\ell(x,a), r)$ in $\cD_\ell$ (ordered arbitrarily), and $\theta_\star = w_\ell$.
Let $\eta_\ell = \bm{r}_\ell - \E[\bm{r}_\ell \mid \x_\ell, \a_\ell] = \bm{r}_\ell - \phi_\ell(\x_\ell, \a_\ell)^\top w_\ell$.
Since $\bm{r}_\ell \in [0,1]$ (by \cref{ass:features}) and $\E[\bm{r}_\ell \mid \x_\ell, \a_\ell] \in [0,1]$, we have $\eta_\ell \in [-1,1]$ almost surely. By Hoeffding's lemma, $\eta_\ell$ is $1$-sub-Gaussian, satisfying the requirement of \cref{lem:ridge-ellipsoid}.
Let $\cE_\ell^{\texttt{ridge}}$ be the event that
\begin{align}
    \|\wh w_\ell - w_\ell\|_{V_\ell}
    &\le \sqrt{\lambda} \|w_\ell\| + \sqrt{2\log(H/\delta) + d\log(1 + n|\Gamma_\ell|/(\lambda d))} \nn\\
    &\le \zeta(\delta),
    \label{eq:ellipsoid-bound-final}
\end{align}
where the last inequality follows from $|\Gamma_\ell| \le d$ (so $n|\Gamma_\ell|/(\lambda d) \le n/\lambda$) and $\|w_\ell\| \le \sqrt{d}$ (by \cref{ass:features}).
By \cref{lem:ridge-ellipsoid},
\[
  \P[\cE_\ell^{\texttt{ridge}}]\ge 1-\frac{\delta}{H}.
\]

On the same event $\cE_\ell^{\texttt{ridge}}$, by \cref{cor:ridge-norm} (which is a direct consequence of \cref{lem:ridge-ellipsoid}), we also have
\[
    \|\wh w_\ell\| \le \|w_\ell\| + \frac{1}{\sqrt{\lambda}}\zeta(\delta) \le \sqrt{d} + \lambda^{-1/2}\zeta(\delta) = C.
\]
    \paragraph{Empirical and concentration bounds along $\Gamma_\ell$.}
Fix $\ell\in[H]$ and $\pi'\in\Gamma_\ell$.
Let $\cE_{\ell,\pi'}^{\texttt{stat}}$ be the event that for every $\theta\in\Theta_C \coloneqq \{\theta\in\reals^d:\|\theta\|\le C\}$,
\[
    \left| \E^{\pi'}[\phi_\ell(\x_\ell,\a_\ell)^\top \theta] - \frac{1}{n}\sum_{m=1}^n \phi_\ell(x_\ell\ind{m,\pi'},a_\ell\ind{m,\pi'})^\top \theta \right|
    \le \veps_{\texttt{stat}},
\]
where the expectation is under the distribution of $(\x_\ell,\a_\ell)$ induced by $\pi'$ at layer $\ell$.
By \cref{lem:uniform-convergence},
\[
    \P\big[\cE_{\ell,\pi'}^{\texttt{stat}}\big]\ge 1-\frac{\delta}{Hd}.
\]
Define $\cE_\ell^{\texttt{stat}} \coloneqq \bigcap_{\pi'\in\Gamma_\ell} \cE_{\ell,\pi'}^{\texttt{stat}}$. Since $|\Gamma_\ell|\le d$, a union bound yields
\[
    \P[\cE_\ell^{\texttt{stat}}]\ge 1-\frac{\delta}{H}.
\]

\paragraph{A uniform success event.}
Define
\[
    \cE^{\texttt{good}} \coloneqq \bigcap_{\ell=1}^H \left(\cE_\ell^{\texttt{ridge}}\cap \cE_\ell^{\texttt{stat}}\right).
\]
By the bounds on $\P[\cE_\ell^{\texttt{ridge}}]$ and $\P[\cE_\ell^{\texttt{stat}}]$ and a union bound over $\ell\in[H]$, we have
\[
    \P[\cE^{\texttt{good}}]\ge 1-2\delta.
\]
Condition on $\cE^{\texttt{good}}$ for the remainder of the proof. In particular, for every $\ell\in[H]$ we have
\[
    \|\wh w_\ell\|\le C
    \qquad\text{and}\qquad
    \|w_\ell-\wh w_\ell\|_{V_\ell}\le \zeta(\delta),
\]
where $\zeta$ is as in the lemma's statement.

\paragraph{Empirical bound along $\Gamma_\ell$.}
For each $\pi'\in\Gamma_\ell$, let $\{(x_\ell\ind{m,\pi'}, a_\ell\ind{m, \pi'})\}_{m=1}^n$ denote the independent samples in $\cD_\ell$ collected from $\pi'$ at layer $\ell$.
Using Cauchy-Schwarz, we have
\begin{align}
    \sum_{\pi' \in \Gamma_\ell} \sum_{m=1}^n \left|\phi_\ell(x_\ell\ind{m,\pi'}, a_\ell\ind{m, \pi'})^\top (w_\ell-\wh w_\ell)\right|
    &\le \sqrt{|\Gamma_\ell| n} \cdot \sqrt{\sum_{(x,a,r)\in \cD_\ell} \left(\phi_\ell(x,a)^\top (w_\ell-\wh w_\ell)\right)^2}, \nn\\
    &\le \sqrt{|\Gamma_\ell| n}\cdot \|w_\ell-\wh w_\ell\|_{V_\ell}
    \le \sqrt{d n}\cdot \zeta(\delta),
    \label{eq:empirical-term}
\end{align}
where the last step uses $|\Gamma_\ell|\le d$ and the bound $\|w_\ell-\wh w_\ell\|_{V_\ell}\le \zeta(\delta)$ on $\cE^{\texttt{good}}$.

\paragraph{Bounding $\E^\rho[\phi_\ell^\top (w_\ell - \wh w_\ell)]$ for arbitrary $\rho$.}
Fix an arbitrary policy $\rho\in\Pi$.
Using \cref{itemLspanners} in the lemma's statement and linearity, we have
\begin{align}
   & \left| \E^\rho[\phi_\ell(\x_\ell,\a_\ell)]^\top (w_\ell - \wh w_\ell) \right| \nn\\
    &\quad \le \left| \sum_{\pi'\in\Gamma_\ell} \beta_\ell(\pi') \E^{\pi'}[\phi_\ell(\x_\ell,\a_\ell)]^\top (w_\ell - \wh w_\ell) \right|
    + \veps_{\texttt{span}}\|w_\ell - \wh w_\ell\|, \nn\\
    &\quad \le \sum_{\pi'\in\Gamma_\ell} |\beta_\ell(\pi')|
    \left(
    \frac{1}{n}\sum_{m=1}^n \left|\phi_\ell(x_\ell\ind{m,\pi'},a_\ell\ind{m,\pi'})^\top (w_\ell - \wh w_\ell)\right|
    + \veps_{\texttt{stat}}
    \right)
    + \veps_{\texttt{span}}\|w_\ell - \wh w_\ell\| ,\nn\\
    &\quad \le 2\cdot \frac{1}{n}\sum_{\pi'\in\Gamma_\ell}\sum_{m=1}^n \left|\phi_\ell(x_\ell\ind{m,\pi'},a_\ell\ind{m,\pi'})^\top (w_\ell - \wh w_\ell)\right|
    + 2d\,\veps_{\texttt{stat}}
    + \veps_{\texttt{span}}\|w_\ell - \wh w_\ell\| ,\nn\\
    &\quad \le 2\sqrt{\frac{d}{n}}\cdot \zeta(\delta) + 2d\,\veps_{\texttt{stat}} + \veps_{\texttt{span}}(\|w_\ell\|+\|\wh w_\ell\|), \nn\\
    &\quad \le 2\sqrt{\frac{d}{n}}\cdot \zeta(\delta) + 2d\,\veps_{\texttt{stat}} + \left(\sqrt{d}+C\right)\veps_{\texttt{span}},
    \label{eq:phiDelta-bound}
\end{align}
where we used \eqref{eq:empirical-term}, $|\beta_\ell(\pi')|\le 2$, $|\Gamma_\ell|\le d$, and $\|w_\ell\|+\|\wh w_\ell\|\le \sqrt{d} + C$ on $\cE^{\texttt{good}}$.

\paragraph{From feature error to reward error.}
By \cref{ass:features} (reward linearity), we have $\E[\br_\ell\mid \x_\ell,\a_\ell]=\phi_\ell(\x_\ell,\a_\ell)^\top w_\ell$ and $\wh r_\ell(\x_\ell,\a_\ell)=\phi_\ell(\x_\ell,\a_\ell)^\top \wh w_\ell$, hence for any $\rho\in\Pi$,
\[
    \E^\rho[\br_\ell - \wh r_\ell(\x_\ell,\a_\ell)]
    = \E^\rho[\phi_\ell(\x_\ell,\a_\ell)]^\top (w_\ell-\wh w_\ell).
\]
Therefore, \eqref{eq:phiDelta-bound} bounds $\left|\E^\rho[\br_\ell-\wh r_\ell(\x_\ell,\a_\ell)]\right|$ for every $\rho\in\Pi$.

\paragraph{Optimization guarantee.}
The algorithm calls $\textsc{FQI}(H, \wh r_{1:H}, \Psi_{1:H-1}, n)$ with deterministic rewards $\wh r_{1:H}$.
Let $\cE^{\texttt{fqi}}$ be the event that for all $x\in \cX\ind{H}$,
\[
    V_1^{\pihat}(x;\wh r_{1:H}) \ge V_1^\pi(x;\wh r_{1:H}).
\]
Since $\Psi_{1:H}$ satisfies \eqref{eq:policyopt-psi-cov}, by \cref{lem:fqi} we have
\[
    \P[\cE^{\texttt{fqi}}]\ge 1-\delta.
\]
Condition on $\cE^{\texttt{fqi}}$ for the remainder of this paragraph. Splitting on the event $\x_1 \in \cX\ind{H}$, we obtain
\begin{align}
    \E[V_1^\pi(\x_1;\wh r_{1:H}) - V_1^{\pihat}(\x_1;\wh r_{1:H})]
    &\le \E\left[(V_1^\pi(\x_1;\wh r_{1:H}) - V_1^{\pihat}(\x_1;\wh r_{1:H}))\cdot \mathbb{I}\{\x_1 \in \cX\ind{H}\}\right] \nn\\
    &\quad + \E\left[(V_1^\pi(\x_1;\wh r_{1:H}) - V_1^{\pihat}(\x_1;\wh r_{1:H}))\cdot \mathbb{I}\{\x_1 \notin \cX\ind{H}\}\right] \nn\\
    &\le 0 + \E\left[\sum_{\ell=1}^H |\wh r_\ell(\x_\ell,\a_\ell)|\cdot \mathbb{I}\{\x_1 \notin \cX\ind{H}\}\right] \nn\\
    &\le H\cdot C \cdot \P[\x_1 \notin \cX\ind{H}]
    \le H\cdot C \cdot \veps_{\texttt{conc}},
    \label{eq:opt-conc}
\end{align}
where we used $|\wh r_\ell(x,a)|=|\phi_\ell(x,a)^\top \wh w_\ell|\le \|\wh w_\ell\|\le C$ on $\cE^{\texttt{good}}$.

\paragraph{Conclusion.}
Using \eqref{eq:policyopt-decomp} for $\pi$ and for $\pihat$, together with \eqref{eq:phiDelta-bound} (applied to $\rho=\pi$ and $\rho=\pihat$) and \eqref{eq:opt-conc}, we obtain
\begin{align}
    \E[V_1^\pi(\x_1) - V_1^{\pihat}(\x_1)]
    &\le
    \sum_{\ell=1}^H \left( 2 \sqrt{\frac{d}{n}} \cdot \zeta(\delta) + 2d \veps_{\texttt{stat}} + \left(\sqrt{d}+C\right)\veps_{\texttt{span}} \right)
    + H\cdot C \cdot \veps_{\texttt{conc}}.
\end{align}
Finally, by a union bound,
\[
    \P[\cE^{\texttt{good}}\cap \cE^{\texttt{fqi}}]\ge 1-3\delta.
\]
This completes the proof.
\end{proof} %
\section{Proof of \cref{thm:main} (Main Guarantee)} 
\label{sec:mainproof}
\begin{proof}
  We verify that the outputs of \cref{alg:cover} satisfy the preconditions of \cref{lem:policy-opt} (guarantee of $\textsc{PolicyOpt}_{\texttt{fqi}}$), and then combine the guarantees.

  \paragraph{Step 1: Applying \cref{thm:cover_guarantee}.}
  Let $\cE^{\texttt{cov}}$ denote the event (over the randomness in \cref{alg:cover}) that the outputs $(\cS_{1:H}, \Psi_{1:H}, \Gamma_{1:H})$ satisfy:
  \begin{itemize}[leftmargin=*]
    \item \emph{Coverage (distributional):} For all $h \in [H]$, the conditional probability (over $n$ fresh i.i.d.\ samples from the initial distribution) satisfies
    \[
      \P\left[\cS_h \subseteq \Span\left\{\phi_h^\pi(\x_1\ind{m}) : m \in [n], \pi \in \Psi_h\right\} \,\Big|\, \text{outputs of \cref{alg:cover}}\right] \ge 1 - \frac{\delta}{4H^2}.
    \]
    \item \emph{State concentration:} $\P[\x_1 \in \cX^{(H)}] \ge 1 - \veps'$.
    \item \emph{Spanner property:} For all $h \in [H]$ and $\pi \in \Pi$, there exist $\beta_1, \ldots, \beta_d \in [-2,2]$ such that
    \[
      \left\| \E^\pi[\phi_h(\x_h, \a_h)] - \sum_{j=1}^d \beta_j \E^{\pi_j}[\phi_h(\x_h, \a_h)] \right\| \le 25 d \veps'.
    \]
    \item \emph{Cardinality bounds:} For all $h \in [H]$, $|\Psi_h| \le d$ and $|\Gamma_h| \le d$.
  \end{itemize}
  
  We now show $\P[\cE^{\texttt{cov}}] \ge 1 - \delta/4$.
  Let $\cE^{\texttt{thm}}$ denote the event that the outputs of \cref{alg:cover} satisfy all three conclusions of \cref{thm:cover_guarantee} simultaneously for all $h \in [H]$: concentration, spanner, and for each $h$ the conditional inclusion probability (over fresh samples, given the outputs) is $\ge 1-\delta/(4H^2)$ for $\ncover$ i.i.d.\ samples.
  By \cref{thm:cover_guarantee}, $\P[\cE^{\texttt{thm}}] \ge 1 - \delta/4$.
  We now argue that $\cE^{\texttt{thm}} \subseteq \cE^{\texttt{cov}}$.
  The concentration and spanner properties in $\cE^{\texttt{cov}}$ match those in $\cE^{\texttt{thm}}$.
  For coverage: on $\cE^{\texttt{thm}}$, the conditional inclusion probability is $\ge 1-\delta/(4H^2)$ for $\ncover$ samples; to upgrade to $n$ samples, draw $n$ i.i.d.\ samples and consider the first $\ncover$; if inclusion holds for the first $\ncover$, it holds for all $n$ (since the span can only grow); hence the conditional probability for $n$ samples is at least that for $\ncover$ samples.
  The cardinality bounds hold deterministically by construction: \cref{alg:cover} constructs $\Gamma_h$ via \apxspanner which returns $d$ policies, and $\Psi_h$ is built iteratively with at most $d$ policies added per layer.
  Thus $\cE^{\texttt{thm}} \subseteq \cE^{\texttt{cov}}$, so $\P[\cE^{\texttt{cov}}] \ge \P[\cE^{\texttt{thm}}] \ge 1 - \delta/4$.

  \paragraph{Step 2: Applying \cref{lem:policy-opt}.}
  On $\cE^{\texttt{cov}}$, all preconditions of \cref{lem:policy-opt} are satisfied with sample size $n$, $\veps_{\texttt{span}} = 25 d \veps' = 25\veps/(c_0 d H^2 L)$, and $\veps_{\texttt{conc}} = \veps' = \veps/(c_0 d^2 H^2 L)$.
  Define (matching the notation in \cref{lem:policy-opt} with $\lambda = 1$ and $\delta$ replaced by $\delta/4$):
  \[
    \zeta \coloneqq \sqrt{d} + \sqrt{2\log(4H/\delta) + d\log(1 + n)}, \quad C \coloneqq \sqrt{d} + \zeta,
  \]
  and let $\veps_{\texttt{stat}}$ be as in \eqref{eq:stat-error} with $\delta$ replaced by $\delta/4$.
  
  Let $\cE^{\texttt{opt}}$ denote the event that the output policy $\pihat$ of $\textsc{PolicyOpt}_{\texttt{fqi}}$ satisfies: for all $\pi \in \Pi$,
  \begin{align}
    \E\left[V_1^\pi(\x_1) - V_1^{\pihat}(\x_1)\right]
    &\le H \left( 2\sqrt{\frac{d}{n}} \cdot \zeta + 2d \veps_{\texttt{stat}} + (\sqrt{d}+C) \cdot 25 d \veps' \right) + H \cdot C \cdot \veps'. \label{eq:main-bound-intermediate}
  \end{align}
  By \cref{lem:policy-opt} (with $\lambda = 1$), $\P[\cE^{\texttt{opt}} \mid \cE^{\texttt{cov}}] \ge 1 - 3\delta/4$.

  \paragraph{Step 3: Bounding the suboptimality.}
  Since \cref{alg:cover} is called with accuracy parameter $\veps'$, the parameter $\ncover$ from \eqref{eq:parameters} is instantiated with $\veps'$ in place of $\veps$.
  From \eqref{eq:n-policyopt}, we have $n \ge \ncover$ and $n \ge c_2 d^4 H^2 \veps^{-2} L^2$.

  Recall from Step 2 that $\zeta = \sqrt{d} + \sqrt{2\log(4H/\delta) + d\log(1+n)}$ and $C = \sqrt{d} + \zeta$.
  Since $n \ge c_2 d^4 H^2 \veps^{-2} L^2$, we have
  \[
    \log(1+n) = O(\log(d) + \log(H) + \log(1/\veps) + \log(L)) = O(L).
  \]
  For the last equality: if $\log(dH/(\delta\veps)) \ge 1$, then $L = \log(dH/(\delta\veps)) \ge \log(d), \log(H), \log(1/\veps)$ (since $dH/(\delta\veps) \ge d, H, 1/\veps$ respectively, using $dH/\delta \ge 1$), and $\log(L) \le L$ since $L \ge 1$.
  If $\log(dH/(\delta\veps)) < 1$, then $L = 1$ and all the logarithms $\log(d), \log(H), \log(1/\veps), \log(L)$ are $O(1)$, so the bound holds trivially.
  Throughout the remainder of this step, we use that $\log(1+n) = O(L)$, hence $\log(n) = O(L)$ and $\log(1+2nC) = O(L)$.
  From the above, $\zeta = O(\sqrt{dL})$, so $C = O(\sqrt{dL})$.

  We bound each term in \eqref{eq:main-bound-intermediate}:
  \begin{itemize}[leftmargin=*]
    \item \emph{Ridge regression term:} $H \cdot 2\sqrt{d/n} \cdot \zeta$.
    
    With $n \ge c_2 d^4 H^2 \veps^{-2} L^2$, we have $\sqrt{d/n} \le \veps/(c_2^{1/2} d^{3/2} H L)$.
    Thus this term is at most
    \[
      H \cdot \frac{\veps}{c_2^{1/2} d^{3/2} H L} \cdot O(\sqrt{dL}) = O\left(\frac{\veps}{c_2^{1/2} d \sqrt{L}}\right) \le \frac{\veps}{8}
    \]
    for $c_2$ sufficiently large (using $d \ge 1$ and $L \ge 1$).
    
    \item \emph{Statistical error term:} $H \cdot 2d \veps_{\texttt{stat}}$.
    
    From \eqref{eq:stat-error}, $\veps_{\texttt{stat}} = C \cdot \sqrt{\frac{8(d\log(1+2nC) + \log(2Hd/\delta))}{n}} + \frac{2}{n}$.
    Since $C = O(\sqrt{dL})$, we have $\log(C) = O(\log(d) + \log(L)) = O(L)$ (using $\log(d) \le L$ and $\log(L) \le L$ since $L \ge 1$).
    Thus $\log(1+2nC) = O(\log(n) + \log(C)) = O(L)$.
    Also, $\log(2Hd/\delta) = O(L)$ since $L \ge \log(dH/(\delta\veps))$ and $\veps \le 1$.
    The numerator is thus $d\log(1+2nC) + \log(2Hd/\delta) = O(dL)$.
    Thus
    \[
      \veps_{\texttt{stat}} = O\left(C \cdot \sqrt{\frac{dL}{n}}\right) = O\left(\sqrt{dL} \cdot \sqrt{\frac{dL}{c_2 d^4 H^2 \veps^{-2} L^2}}\right) = O\left(\frac{\veps}{c_2^{1/2} d H}\right).
    \]
    Thus this term is at most
    \[
      H \cdot 2d \cdot O\left(\frac{\veps}{c_2^{1/2} d H}\right) = O\left(\frac{\veps}{c_2^{1/2}}\right) \le \frac{\veps}{8},
    \]
    for $c_2$ sufficiently large.
    
    \item \emph{Span error term:} $H \cdot (\sqrt{d}+C) \cdot 25 d \veps'$.
    
    With $\veps' = \veps/(c_0 d^2 H^2 L)$ and $\sqrt{d}+C = O(\sqrt{dL})$, this is
    \[
      O\left(H \cdot \sqrt{dL} \cdot d \cdot \frac{\veps}{c_0 d^2 H^2 L}\right) = O\left(\frac{\veps\sqrt{dL}}{c_0 d H L}\right) = O\left(\frac{\veps}{c_0 \sqrt{d} H \sqrt{L}}\right) \le \frac{\veps}{8}
    \]
    for $c_0$ sufficiently large (using $d, H, L \ge 1$).
    
    \item \emph{Concentration error term:} $H \cdot C \cdot \veps'$.
    
    This is $O\left(\frac{\sqrt{dL} \cdot \veps}{c_0 d^2 H L}\right) = O\left(\frac{\veps}{c_0 d^{3/2} H \sqrt{L}}\right) \le \frac{\veps}{8}$ for $c_0$ sufficiently large (using $d, H, L \ge 1$).
  \end{itemize}
  Summing all four contributions yields $\E[V_1^\pi(\x_1) - V_1^{\pihat}(\x_1)] \le \veps$ for all $\pi \in \Pi$.

  \paragraph{Step 4: Combining the events.}
  By the chain rule for conditional probabilities,
  \[
    \P[\cE^{\texttt{cov}} \cap \cE^{\texttt{opt}}] = \P[\cE^{\texttt{cov}}] \cdot \P[\cE^{\texttt{opt}} \mid \cE^{\texttt{cov}}] \ge (1 - \delta/4)(1 - 3\delta/4) \ge 1 - \delta.
  \]
  On the event $\cE^{\texttt{cov}} \cap \cE^{\texttt{opt}}$, the bound \eqref{eq:main-bound-intermediate} holds, and the analysis in Step 3 gives the stated suboptimality guarantee.
\end{proof} \clearpage
\part{Auxiliary Results}
\label{part:auxiliary}

\section{Guarantee of \textsc{ComputeOutlierDirection}}

\begin{algorithm}[ht]
  \caption{$\textsc{ComputeOutlierDirection}(P, n, m)$.}
  \label{alg:getdirection}
  \begin{algorithmic}[1]
    \Require{$P \text{ (distribution over $\reals^d$)},n, m$.}
    \State Draw $n$ samples $U = \{\u_1, \dots,\u_n\}$ from $P$.
    \For{$i = 1, \dots, n$}
    \State Initialize counter $S_i \gets 0$.
    \For{$j = 1, \dots, m$}
    \State Draw a fresh set of $n$ samples $W\ind{j} = \{\w\ind{j}_1, \dots, \w\ind{j}_n\}$ from $P$.
    \If{$\u_i \in \Span (W\ind{j})$}
    \State $S_i \gets S_i + 1$.
    \EndIf
    \EndFor
    \State Compute the empirical estimate: $\hat{p}_i = S_i/m$.
    \EndFor
    \State Select index $i^{\star} = \argmax_{i \in \{1, \dots, n\}} \hat{p}_i$.
    \State \textbf{Return:} $v = \u_{i^{\star}}$.
  \end{algorithmic}
\end{algorithm}

\begin{lemma}
  \label{lem:compute_outlier_direction}
  Let $P$ be a distribution over $\mathbb{R}^d$ and consider a call to $\textsc{ComputeOutlierDirection}(P, n, m)$ for some given $n, m \in \mathbb{N}$. Let $\w_1, \dots, \w_n$ be a set of i.i.d.~samples from $P$. Then, for any $\delta\in(0,1)$, with probability at least $1-\delta$, the output vector $v$ of $\textsc{ComputeOutlierDirection}$ satisfies:
  \begin{align}
    \P[v \in \Span\{\w_1, \dots, \w_n\} \mid v] \ge  1 - \frac{d}{n+1} - 2\sqrt{\frac{\log(4n/\delta)}{2m}} - \sqrt{\frac{\log(4/\delta)}{2n}}.
  \end{align}
  Further, the algorithm samples from $P$ a total of $K=n+m n^2$ points.
\end{lemma}
\begin{proof}
  Let $\w_1, \dots, \w_n$ be i.i.d.~samples from $P$ and let $\u_1, \dots, \u_n$ be as \cref{alg:getdirection}. Let $\bm{S}_n$ denote the random subspace $\Span\{\w_1,\dots,\w_n\}$. For any fixed vector $x \in \mathbb{R}^d$, we define the ``spannability'' probability:
  \[
    p(x) = \P[x \in \bm{S}_n].
  \]

  \paragraph{Step 1: Lower bound on the population mean.} Fix $i\in [n]$ and consider the expectation of $p(\u_i)$:
  \[
    \mu \coloneqq \mathbb{E}[p(\u_i)] =  \mathbb{E}[\P[ \u_i \in \bm{S}_n \mid \u_i]] = \P[\u_i \in \bm{S}_n].
  \]
  Since $\u_i, \w_1, \dots, \w_n$ are i.i.d.~samples from $P$, by \cref{thm:span}, we have
  \[
    \mathbb{P}[\u_i \notin \bm{S}_n] \le \frac{d}{n+1}.
  \]
  Therefore, the expected success probability is lower bounded by:
  \begin{align}
    \mu \ge 1 - \frac{d}{n+1}.  \label{eq:mu}
  \end{align}
  \paragraph{Step 2: Concentration of the empirical average.} Fix $\delta \in(0,1)$. By Hoeffding's Inequality, there is an event $\cE$ of measure at least $1-\delta/2$ such that under $\cE$, we have that
  \begin{align}
    \frac{1}{n} \sum_{i=1}^n p(\u_i) \geq \mu - \sqrt{\frac{\log(4/\delta)}{2n}}. \label{eq:hoeffding1}
  \end{align}
  Now, consider the quantity $\hat{p}_i$ in \cref{alg:getdirection}. By Hoeffding's inequality, there are events $\cE_1,\dots, \cE_n$ such that for each $i$, $\P[\cE_i] \ge 1-\frac{\delta}{2 n}$ and under $\cE_i$, we have that
  \begin{align}
    p(\u_i) + \sqrt{\frac{\log(4n/\delta)}{2m}}	\geq \hat{p}_i \ge p(\u_i) - \sqrt{\frac{\log(4n/\delta)}{2m}}. \label{eq:hoeffding2}
  \end{align}
  Let $i^\star$ be as in \cref{alg:getdirection}. By combining \eqref{eq:hoeffding1} and \eqref{eq:hoeffding2} and using the union bound (so that \eqref{eq:hoeffding1} and \eqref{eq:hoeffding2}, for all $i$, hold simultaneously), we have that with probability at least $1-\delta$,
  \begin{align}
    \P[v \in \Span\{\w_1, \dots, \w_n\} \mid v] & = p(\u_{i^\star}), \nn                                                                          \\                                    & \geq \max_{i\in [n]} \hat{p}_i - \sqrt{\frac{\log(4n/\delta)}{2m}},  \quad (\text{by \eqref{eq:hoeffding2} and the definition of $i^\star$})\nn             \\
                                                & \geq \frac{1}{n} \sum_{i=1}^n \hat{p}_i - \sqrt{\frac{\log(4n/\delta)}{2m}}, \nn                \\                                 & \ge \frac{1}{n} \sum_{i=1}^n p(\u_i) - 2\sqrt{\frac{\log(4n/\delta)}{2m}},  \nn\\  & \geq \mu - 2 \sqrt{\frac{\log(4n/\delta)}{2m}} - \sqrt{\frac{\log(4/\delta)}{2n}},  \quad \text{(by \eqref{eq:hoeffding1})}\nn \\
                                                & \geq 1 - \frac{d}{n+1} - 2\sqrt{\frac{\log(4n/\delta)}{2m}} - \sqrt{\frac{\log(4/\delta)}{2n}},
  \end{align}
  where the last inequality follows from \eqref{eq:mu}.
\end{proof}

\begin{lemma}
  \label{thm:span}
  Let $P$ be an arbitrary probability distribution over $\mathbb{R}^d$. Let $\bm{u}_1, \dots, \bm{u}_{n+1}$ be i.i.d.~random vectors drawn from $P$. Then, the expected probability that the $(n+1)$-th sample falls outside the span of the first $n$ samples is bounded by:
  \[
    \P\left[ \bm{u}_{n+1} \notin \operatorname{Span}(\bm{u}_1, \dots, \bm{u}_n) \right] \le \frac{d}{n+1}.
  \]
\end{lemma}

\begin{proof}
  Let $S = \{\u_1, \dots, \u_{n+1}\}$ be the set of $n+1$ i.i.d. samples. We utilize a symmetry (exchangeability) argument based on the concept of ``critical'' samples.

  \paragraph{Critical samples.}
  For a fixed set of samples $S$, we say that the sample $\u_i \in S$ is critical with respect to $S$ if it does not belong to the span of the other samples:
  \[
    \u_i \notin \operatorname{Span}(S \setminus \{\u_i\}).
  \]
  If $\u_i$ is critical, it implies that removing $\u_i$ strictly decreases the dimension of the subspace spanned by the set. That is:
  \[
    \dim(\operatorname{Span}(S \setminus \{\u_i\})) < \dim(\operatorname{Span}(S)).
  \]

  \paragraph{Bounding the number of critical samples.}
  Let $C \subseteq S$ be the set of all critical samples in $S$. We claim that the vectors in $C$ are linearly independent.
  \begin{itemize}[leftmargin=*]
    \item Suppose for the sake of contradiction that the vectors in $C$ are linearly dependent. Then, there exists a non-trivial linear combination:
          \[
            \sum_{\u_j \in C} \alpha_j \u_j = 0,
          \]
          where not all $\alpha_j$ are zero.
    \item Let $\u_k \in C$ be a vector with a non-zero coefficient $\alpha_k \neq 0$. We can rewrite the equation as:
          \[
            \u_k = - \sum_{\u_j \in C, j \neq k} \frac{\alpha_j}{\alpha_k} \u_j.
          \]
    \item This equation expresses $\u_k$ as a linear combination of other vectors in $C$. Since $C \subseteq S$, the other vectors are in $S \setminus \{\u_k\}$. Thus, $\u_k \in \operatorname{Span}(S \setminus \{\u_k\})$.
    \item This contradicts the definition of $\u_k$ being a critical sample.
  \end{itemize}
  Therefore, the set of critical samples $C$ must be linearly independent. Since the vectors lie in $\mathbb{R}^d$, the size of any linearly independent set is bounded by the dimension $d$. Consequently, the number of critical samples is at most $d$:
  \[
    |C| \le d.
  \]

  \paragraph{Applying exchangeability.}
  Define the indicator variable $\bm{I}_i$ for the event that $\u_i$ is critical in the set of $n+1$ samples:
  \[
    \bm{I}_i = \mathbb{I}\{\u_i \notin \operatorname{Span}(\{\u_1, \dots, \u_{n+1}\} \setminus \{\u_i\})\}.
  \]
  The quantity we wish to bound is the probability that the specific sample $\u_{n+1}$ is not in the span of the previous $n$. By the definition above, this is exactly the probability that $\u_{n+1}$ is critical:
  \[
    \P[\u_{n+1} \notin \operatorname{Span}(\u_1, \dots, \u_n)] = \P[\bm{I}_{n+1} = 1].	\]
  Since the samples $\u_1, \dots, \u_{n+1}$ are i.i.d.~the joint distribution is invariant under permutations. Thus, the probability of being critical is identical for every index $i$:
  \[
    \P[\bm{I}_1 = 1] = \P[\bm{I}_2 = 1] = \dots = \P[\bm{I}_{n+1} = 1].
  \]
  Let this common probability be $p_{\texttt{fail}}$. The expected number of critical samples is the sum of these probabilities:
  \begin{align}
    \mathbb{E}[\text{number of critical samples}] & = \mathbb{E}\left[\sum_{i=1}^{n+1} \bm{I}_i \right], \nonumber      \\
                                                  & =	\sum_{i=1}^{n+1} \P[\bm{I}_i = 1] = (n+1)\cdot p_{\texttt{fail}}.
  \end{align}
  However, we established deterministically that the number of critical samples is never greater than $d$. Therefore,
  \[
    (n+1)\cdot p_{\texttt{fail}} \le d \implies p_{\texttt{fail}} \le \frac{d}{n+1}.
  \]
  Substituting $p_{\texttt{fail}} = \P[\u_{n+1} \notin \operatorname{Span}(\u_1, \dots, \u_n)]$, we obtain:
  \[
    \P[\u_{n+1} \notin \operatorname{Span}(\u_1, \dots, \u_n)] \le \frac{d}{n+1}.
  \]
\end{proof}

\section{Guarantee of \textsc{FQI} with Deterministic Rewards and Transitions}
\label{sec:fqi-deterministic}

\begin{algorithm}[ht]
  \caption{$\textsc{FQI}(h, r_{1:h},\Psi_{1:h-1},
      n)$:
    Fitted $Q$-iteration \citep{munos2008finite}.}
  \label{alg:fqi}
  \begin{algorithmic}[1]
    \Require{$h, r_{1:h},\Psi_{1:h-1},
        n$.}
    \State Define $\Qhat_{h}(\cdot,\cdot)= r_h(\cdot,\cdot)$ and set $\pihat_h(\cdot)\in \argmax_{a\in \cA} \Qhat_{h}(\cdot, a)$.
    \For{$\ell = h-1, \ldots, 1$}
    \State Set $\cD_\ell\gets \emptyset$.
    \For{$\pi\in\Psi_\ell$}
    \For{$n$ times}
    \State Sample $(\x_1, \a_1, \dots,\x_\ell,\a_\ell, \x_{\ell+1}) \sim \P^{\pi}$.
    \State Set $\bm{y}_\ell \gets   r_\ell(\x_\ell, \a_\ell) + \max_{a\in \cA} \Qhat_{\ell+1}(\x_{\ell+1},a)$.
    \State Update $\cD_{\ell}\gets \cD_{\ell}\cup \{(\x_\ell,\a_\ell, \bm{y}_\ell)\}$.
    \EndFor
    \EndFor
    \State Set $\Sigma_\ell \gets \sum_{(x,a,y)\in \cD_\ell} \phi_\ell(x,a)\phi_\ell(x,a)^\top$.
    \State Set $\wh\theta_\ell \gets \Sigma_{\ell}^{\dagger} \cdot \sum_{(x,a,y)\in \cD_{\ell}}
      \phi_\ell(x, a) \cdot y$.  \label{line:define-w-hat}
    \State Define $\Qhat_\ell(\cdot,\cdot) \coloneqq \phi_\ell(\cdot,\cdot)^\top \wh\theta_\ell$ and set $\pihat_\ell(\cdot) \coloneqq  \argmax_{a \in
        \A} \Qhat_\ell(\cdot,a)$.
    \EndFor
    \State \textbf{Return:} $\pihat =\pihat_{1:h}$.
  \end{algorithmic}
\end{algorithm}

\begin{lemma}
  \label{lem:fqi}
  Fix $\delta \in(0,1)$. Consider a call to \textsc{FQI} (\cref{alg:fqi}) with input $h \geq 1$, rewards $r_{1:h}$, collection of policies $\Psi_{1:h-1}$, and $n \in \mathbb{N}$. Further, let $\cS_{1:h-1}\subseteq \reals^d$ be a collection of subspaces and define $\cX\ind{\ell} \coloneqq \bigcap_{j \in[\ell-1]}\{x\in \cX \mid \Span\{\phi^{\pi}_j(x) : \pi \in \Pi\} \subseteq \cS_{j}\}$ for $\ell \in [h]$. Suppose the following:
  \begin{enumerate}[leftmargin=*]
    \item \cref{ass:features,ass:bellman_complete} hold.
    \item The rewards $r_{1:h}$ are such that for all $\ell \in [h]$, there exists a vector $w_\ell \in \mathbb{R}^d$ such that $r_\ell(\cdot, \cdot)= w_\ell^\top \phi_\ell(\cdot,\cdot)$.
    \item  The collections of policies $\Psi_{1:h-1}$ satisfy: for all $\ell\in[h-1]$ and i.i.d.~$\x_1, \x_1\ind{1},\dots,\x_1\ind{n}$: \label{item:span}
          \begin{align}
            \P[\cS_\ell \subseteq \Span\left\{\phi^\pi_\ell(\x_{1}\ind{m}): m\in[n], \pi \in \Psi_{\ell}\right\}] \geq 1-\frac{\delta}{H^2}.
          \end{align}
  \end{enumerate}
  Then, with probability at least $1-\delta$, the output policies $\pihat_{1:h}$ of \textsc{FQI} satisfy,
  \begin{align}
    \forall x\in \cX\ind{h},\quad 	\max_{\pi \in \Pi} V^{\pi}_1(x;r_{1:h})\leq V^{\pihat}_1(x;r_{1:h}).\label{eq:fqi}
  \end{align}
\end{lemma}

\begin{corollary}
  \label{cor:fqi}
  Fix $\delta \in (0,1)$. Consider the setting of \cref{lem:fqi} with $r_{1:h-1}\equiv 0$ and $r_h(\cdot, \cdot) = w_h^\top \phi_h(\cdot,\cdot)$ for some $w_h\in \mathbb{R}^d$. Then, with probability at least $1-\delta$, the output policies $\pihat_{1:h}$ of \textsc{FQI} satisfy,
  \begin{align}
    \forall x\in \cX\ind{h}, \quad \max_{\pi \in \Pi} \phi_h^\pi(x)^\top w_h \leq \phi_h^{\pihat}(x)^\top w_h. \label{eq:corfqi}
  \end{align}
\end{corollary}

\begin{proof}[Proof of \cref{lem:fqi}]
  Let $\cT_h^r[Q](x,a) \coloneqq r_h(x,a) + \E\left[\max_{a'\in \cA}Q(\x_{h+1}, a')  \mid \x_h =x, \a_h =a\right]$ be the Bellman operator under rewards $r_{1:H}$. Fix $x\in \cX\ind{h}$ and $\pi \in \Pi$. By \cref{cor:decomp-diff}, we have
  \begin{align}
  V_1^{\pi}(x;r_{1:h}) - V_1^{\pihat}(x;r_{1:h}) & \leq \sum_{\ell=1}^h \E^{\pi} \left[ \left(\cT_\ell^r[\Qhat_{\ell+1}] - \Qhat_{\ell}\right)(\x_\ell,\a_\ell) \mid \x_1 = x\right] \nn \\
  & \quad  -  \sum_{\ell=1}^h \E^{\pihat} \left[ \left(\cT_\ell^r[\Qhat_{\ell+1}] - \Qhat_{\ell}\right)(\x_\ell,\a_\ell) \mid \x_1 =x\right], \nn \\
  & \leq \sum_{\ell=1}^{h-1} \E^{\pi} \left[ \left(\cT_\ell^r[\Qhat_{\ell+1}] - \Qhat_{\ell}\right)(\x_\ell,\a_\ell) \mid \x_1 = x\right] \nn \\
  & \quad  -  \sum_{\ell=1}^{h-1} \E^{\pihat} \left[ \left(\cT_\ell^r[\Qhat_{\ell+1}] - \Qhat_{\ell}\right)(\x_\ell,\a_\ell) \mid \x_1 =x\right], \label{eq:fqi-decomp}
  \end{align}
  where the last step follows from the fact that for $\ell=h$, \(\cT_h^r[\Qhat_{h+1}]-\Qhat_h\equiv 0\).
  Using \cref{ass:bellman_complete} (linear Bellman completeness), we have that there exists $\tilde\theta_\ell \in \reals^d$ such that for all $(x',a)\in \cX\times\cA$:
  \begin{align}
  r_\ell(x', a) + \E\left[ \max_{a'\in \cA} \Qhat_{\ell+1}(\x_{\ell+1},a')\mid \x_\ell =x',\a_\ell =a  \right] = \phi_{\ell}(x',a)^\top \tilde\theta_\ell. \nn 
  \end{align} 
  Using this, the definition of the regression targets $\bm{y}_\ell$ in \cref{alg:fqi}, and the deterministic assumption (which implies that $\bm{y}_\ell  = \E[\bm{y}_\ell \mid \x_\ell, \a_\ell]$), we have 
  \begin{align}
    \bm{y}_\ell 
    & = r_\ell(\x_\ell, \a_\ell) + \E\left[ \max_{a'\in \cA} \Qhat_{\ell+1}(\x_{\ell+1},a')\mid \x_\ell, \a_\ell\right] = \phi_\ell(\x_\ell, \a_\ell)^\top \tilde\theta_\ell.
  \end{align}
Thus, since $\Qhat_\ell(\cdot,\cdot) = \phi_\ell(\cdot,\cdot)^\top \wh\theta_\ell$, the Bellman residual is linear:
  \begin{align}
    (\cT_\ell^r[\Qhat_{\ell+1}] - \Qhat_{\ell})(\cdot, \cdot) = \iprod{\tilde\theta_\ell - \wh\theta_\ell}{\phi_\ell(\cdot, \cdot)}.
  \end{align}
  Substituting this into \eqref{eq:fqi-decomp}, we get
  \begin{align}
     V_1^{\pi}(x) - V_1^{\pihat}(x) & \leq \sum_{\ell=1}^{h-1} \iprod{\tilde\theta_\ell - \wh\theta_\ell}{\phi^\pi_\ell(x) - \phi^{\pihat}_\ell(x)}. \label{eq:fqi-linear}
  \end{align}
  Note that by definition of $\cX\ind{h}$ and the fact that $x\in \cX\ind{h}$ and $\cS_\ell$ is a subspace, we have that
  \begin{align}
    \phi^{\pi}_\ell(x) - \phi^{\pihat}_\ell(x) \in \cS_\ell, \quad \forall \ell\in[h-1]. \label{eq:setincl-fqi}
  \end{align}

  Now, fix $\ell\in[h-1]$. Consider the dataset $\cD_\ell$ constructed in \cref{alg:fqi}. For every $\pi' \in \Psi_\ell$, $\cD_\ell$ contains $n$ i.i.d.~tuples $(\x_\ell,\a_\ell,\bm{y}_\ell)$, where \((\x_\ell,\a_\ell)\) are generated by executing $\pi'$ up to time $\ell$. Define $u_{\ell}(x,\pi)\coloneqq \phi^{\pi}_\ell(x)-\phi^{\pihat}_\ell(x)$.
   By \eqref{eq:setincl-fqi} we have $u_{\ell}(x,\pi)\in\cS_\ell$. 
   Applying \cref{item:span}, we have that with probability at least $1-\delta/H$,
  \begin{align}
    \cE_\ell \coloneqq \left\{ \cS_\ell \subseteq \Span\{\phi_\ell(x',a):(x',a,y)\in \cD_\ell\} \right\}
  \end{align}
  holds. Condition on $\cE_\ell$ for the remaining argument. Since $u_{\ell}(x,\pi) \in \cS_\ell$, we have
  \begin{align}
    u_{\ell}(x,\pi)\in \Span\{\phi_\ell(x',a):(x',a,y)\in \cD_\ell\}.\label{eq:linearcomb-fqi}
  \end{align}
  On the other hand, by \cref{lem:noiseless-regression-orthogonality} applied to the minimum-norm least-squares estimator $\wh\theta_\ell$ with targets $\bm{y}_\ell$ (which are linear in $\phi_\ell$ as shown above), we have
  \begin{align}
    \Sigma_\ell(\wh\theta_\ell-\tilde\theta_\ell)= 0, \label{eq:regression-fqi}
  \end{align}
  where $\Sigma_\ell = \sum_{(x',a,y)\in \cD_\ell} \phi_\ell(x',a) \phi_\ell(x',a)^\top$.
  Since $\mathrm{Col}(\Sigma_\ell)=\Span\{\phi_\ell(x',a):(x',a,y)\in\cD_\ell\}$, \eqref{eq:regression-fqi} implies that $\wh\theta_\ell-\tilde\theta_\ell$ is orthogonal to $\Span\{\phi_\ell(x',a):(x',a,y)\in\cD_\ell\}$, and so by \eqref{eq:linearcomb-fqi}, we have
  \[
    \iprod{\wh\theta_\ell-\tilde\theta_\ell}{\phi^{\pi}_\ell(x)-\phi^{\pihat}_\ell(x)}=0.
  \]
  By a union bound over $\ell\in[h-1]$, with probability at least $1-\delta$,
  $\cE_\ell$ holds for all $\ell\in[h-1]$. Further, under $\bigcap_{\ell=1}^{h-1} \cE_\ell$, we have
  \begin{align}
    \forall \ell\in[h-1], \quad |\iprod{\tilde\theta_\ell - \wh\theta_\ell}{\phi^{\pi}_\ell(x) - \phi^{\pihat}_\ell(x)}| =0. \label{eq:plug-fqi}
  \end{align}
  Plugging this into \eqref{eq:fqi-linear} yields $V_1^\pi(x;r_{1:h}) - V_1^{\pihat}(x;r_{1:h}) \le 0$.
\end{proof}

\begin{algorithm}[!htbp]
  \caption{\apxspanner: Approximate barycentric spanner via approximate linear optimization (adapted from \textsc{RobustSpanner} in \cite{mhammedi2023efficient})}
  \label{alg:spanner}
  \begin{algorithmic}[1]\onehalfspacing
    \Require~
    \begin{itemize}[leftmargin=*]
      \item Approximate linear optimization subroutine
            $\apx:\reals^d\to \cZ$. 
      \item Approximate index-to-vector subroutine
            $\est:\cZ\rightarrow  \reals^d$.
      \item Parameters $C,\veps>0$.
    \end{itemize}
    \State Set $W =(w_1,\dots,w_d)= (e_1,\dots,e_d)$.
    \For{$i=1,\dots, d$} \label{line:firstfor}
    \State Set $\theta_i =(\det(e_j,W_{-i}))_{j\in[d]}\in
      \reals^d$. \hfill\algcommentlight{$W_{-i}$ is defined to be
      $W$ without the $i$th column}
    \State Set $z_i^+ = \apx(\theta_i/\|\theta_i\|)$ and $w_i^+= \est(z_i^+)$.
    \State Set $z_i^- = \apx(-\theta_i/\|\theta_i\|)$ and $w_i^-= \est(z_i^-)$.
    \If{$\theta_i^\top w^+_i\geq - \theta_i^\top w^-_i$} \label{line:if}
    \State Set $\wtilde w_i= w^+_i$, $z_i = z_i^+$, and $w_i = \wtilde w_i + \veps \theta_i/\|\theta_i\|$.
    \Else
    \State Set $\wtilde w_i= w^-_i$, $z_i = z_i^-$, and $w_i = \wtilde w_i - \veps \theta_i/\|\theta_i\|$.
    \EndIf \label{line:endif}
    \EndFor
    \For{$n=1,2,\dots$} \label{line:for}
    \State Set $i=1$.
    \While{$i\leq d$}
    \State Set $\theta_i =(\det(e_j,W_{-i}))_{j\in[d]}\in \reals^d$.
    \State Set $z_i^+ = \apx(\theta_i/\|\theta_i\|)$ and $w_i^+= \est(z_i^+)$.
    \State Set $z_i^- = \apx(-\theta_i/\|\theta_i\|)$ and $w_i^-= \est(z_i^-)$.
    \If{$\theta_i^\top w_i^+ +\veps \cdot \|\theta_i\|  \geq C \cdot |\det(w_i, W_{-i})|$}
    \State Set $\wtilde w_i = w_i^+$, $z_i = z_i^+$, and $w_i = \wtilde w_i  + \veps \cdot \theta_i/\|\theta_i\|$.
    \State \textbf{break}
    \ElsIf{$-\theta_i^\top w_i^- +\veps \cdot \|\theta_i\|  \geq C \cdot |\det(w_i, W_{-i})|$}
    \State  Set $\wtilde w_i = w_i^-$, $z_i = z_i^-$, and $w_i = \wtilde w_i  - \veps \cdot \theta_i/\|\theta_i\|$.
    \State \textbf{break}
    \EndIf
    \State Set $i = i+1$.
    \EndWhile
    \If{$i=d+1$}
    \State  \textbf{break}
    \EndIf
    \EndFor
    \State \textbf{Return:} $(z_1, \dots, z_d)$.
  \end{algorithmic}
\end{algorithm}

\begin{algorithm}[tp]
  \caption{$\veceval(h, F, \pi,n)$: Estimate
    $\E^{\pi}[F(\x_h,\a_h)]$ for policy $\pi$ and function
    $F:\cX\times\cA\rightarrow\reals^d$.}
  \label{alg:veceval}
  \begin{algorithmic}[1]\onehalfspacing
    \Require~
    \begin{itemize}[leftmargin=*]
      \item Target layer $h\in[H]$.
      \item Vector-valued function $F:\cX \times \cA\rightarrow \reals^d$.
      \item Policy $\pi\in \Pim$.
      \item Number of samples $n\in \mathbb{N}$.
    \end{itemize}
    \State $\cD \gets\emptyset$.
    \For{$n$ times}
    \State Sample $(\x_h, \a_h)\sim
      \pi$.
    \State Update dataset: $\cD \gets \cD \cup \left\{\left(\x_h, \a_h\right)\right\}$.
    \EndFor
    \State \textbf{Return:} $\bar F = \frac{1}{n}\sum_{(x, a)\in\cD} F(x,a)$.
  \end{algorithmic}
\end{algorithm}

\section{Guarantee of \veceval{}}
\begin{lemma}[Guarantee of \veceval{}]
  \label{lem:veceval}
  Let $\delta \in(0,1)$, $h\in[H]$, $\phi\in \Phi$, $\theta \in \reals^d$, $\pi \in \Pim$, and $n\in \mathbb{N}$ be given. The output $\phibar_h= \veceval(h,\phi,\pi, n)$ (\cref{alg:veceval}) satisfies, with probability at least $1-\delta$,
  \begin{align}
    \| \phibar_h -  \E^{\pi}[\phi(\x_h,\a_h) ] \| \leq c \cdot \sqrt{n^{-1} \cdot \log (1/\delta)}, \label{eq:est}
  \end{align}
  where $c>0$ is a sufficiently large absolute constant.
\end{lemma}
\begin{proof}[\pfref{lem:veceval}]
  By a standard vector-valued concentration bound in Euclidean space (see for example \citet{pinelis1994optimum}) and the fact that $\norm{\phi(x, a)} \leq 1$ for all $x \in \cX$ and $a \in \cA$, there exists an absolute constant $c>0$ such that with probability at least $1 - \delta$,
  \begin{align}
    \nrm*{\phibar_h - \ee^\pi\left[ \phi(\x_h, \a_h) \right]} \leq c \cdot \sqrt{\frac{\log(1/\delta)}{n}}.
  \end{align}
\end{proof}

\section{Guarantee of \RS{}}

\begin{algorithm}[t]
  \caption{$\textsc{RejectionSampling}(Q,A,K)$: Rejection sampling (to approximate $Q(\cdot\mid A)$ with budget $K$)}
  \label{alg:rejection-sampling-budgeted}
  \begin{algorithmic}[1]
    \For{$j=1,2,\dots,K$}
    \State Draw $X_j \sim Q$ independently
    \If{$X_j \in A$}
    \State \Return $Y \gets X_j$
    \EndIf
    \EndFor
    \State \Return $Y \gets \mathbf{0}$
  \end{algorithmic}
\end{algorithm}

\begin{lemma}[Transfer under finite-budget rejection sampling on $\R^d$]
  \label{lem:transfer-rejection-Rd}
  Let $Q$ be a probability distribution on $\R^d$ and let $A \subseteq \R^d$ be measurable with $p \coloneqq Q(A) \in (0,1].$
  Define the conditional distribution $P \coloneqq Q(\cdot \mid A)$ on $\R^d$. Consider a (possibly randomized) algorithm $\Alg$ which, when run with access to i.i.d.\ samples from a distribution $D$ on $\R^d$,
  draws exactly $N \in \mathbb{N}$ samples from $D$ (possibly adaptively) and outputs $\Out(D)$.
  Assume that when $\Alg$ is run with i.i.d.\ samples from $P$, its output satisfies an event $\cE$ with probability at least $1-\delta$:
  \begin{equation}
    \label{eq:assumption-P-Rd}
    \P\big[\cE \text{ holds for } \Out(P)\big] \;\ge\; 1-\delta,
  \end{equation}
  where the probability is over the i.i.d.\ $P$-samples used by $\Alg$ and any internal randomness of $\Alg$.

  Define $\Alg^{\texttt{RS}}$ as the implementation of $\Alg$ in which each time $\Alg$ requests a sample from $P$,
  we instead invoke $\RS(Q,A,K)$; the value returned by $\RS(Q,A,K)$ is provided to $\Alg$ as the requested sample
  Let $\Out^{\texttt{RS}}$ denote the resulting output. Then the following statements hold:

  \[
    \P\big[\cE \text{ holds for } \Out^{\texttt{RS}}\big]
    \;\ge\;
    1 - \delta - N(1-p)^K
    \;\ge\;
    1 - \delta - N e^{-pK}.
  \]
\end{lemma}

\begin{proof} \textbf{Step 1:} Fix one invocation of $\RS(Q,A,K)$ and let $X_1,\dots,X_K \sim Q$ be the i.i.d.\ samples drawn within that invocation.
  Let $J$ be the (random) index of the first accepted sample, with the convention $J=\infty$ if no accept occurs.
  On the event $\{J<\infty\}$, the returned value is $X_J$, while on $\{J=\infty\}$ the algorithm returns $\mathbf{0}$.

  Fix any measurable $B \subseteq \R^d$ and any $1 \le j \le K$. By independence,
  \[
    \P[X_J \in B,\ J=j]
    =
    \P\big[X_1\notin A,\dots,X_{j-1}\notin A,\ X_j \in (B\cap A)\big]
    =
    (1-p)^{j-1} Q(B\cap A),
  \]
  where $p=Q(A)$. Summing over $j$ yields
  \[
    \P[X_J \in B,\ J<\infty]
    =
    \sum_{j=1}^K (1-p)^{j-1} Q(B\cap A)
    =
    Q(B\cap A)\cdot \frac{1-(1-p)^K}{p}.
  \]
  Also, $\P[J<\infty]=1-(1-p)^K$. Therefore,
  \[
    \P[X_J \in B \mid J<\infty]
    =
    \frac{\P[X_J \in B,\ J<\infty]}{\P[J<\infty]}
    =
    \frac{Q(B\cap A)}{p}
    =
    \frac{Q(B\cap A)}{Q(A)}
    =
    Q(B\mid A)
    =
    P(B).
  \]
  Thus, conditional on not returning $\mathbf{0}$, a single call to $\RS(Q,A,K)$ returns an exact draw from $P$.

  \paragraph{Step 2: coupling for $N$ calls (including adaptive requests).}
  Run $\Alg^{\texttt{RS}}$ and let $G$ be the event that all $N$ calls to $\RS(Q,A,K)$ return a nonzero vector.
  On $G$, each call returns an exact draw from $P$ by Step 1, and the internal randomness used across calls
  is independent because each call uses fresh i.i.d.\ draws from $Q$.
  Therefore, the sequence of $N$ values provided to $\Alg$ on $G$ is i.i.d.\ with law $P$, even if $\Alg$ requests them adaptively.
  It follows that the conditional distribution of the entire execution of $\Alg^{\texttt{RS}}$ given $G$
  matches that of $\Alg$ run with true i.i.d.\ $P$-samples, which implies
  \[
    \Law(\Out^{\texttt{RS}} \mid G) = \Law(\Out(P)).
  \]

  \paragraph{Step 3: lower bound on $\P[G]$.}
  Let $F_i$ be the event that the $i$-th call to $\RS(Q,A,K)$ returns $\mathbf{0}$.
  Then $F_i$ occurs exactly when $K$ consecutive draws from $Q$ miss $A$, hence
  \[
    \P[F_i] = (1-p)^K.
  \]
  By the union bound,
  \[
    \P[G^c] = \P\left[\bigcup_{i=1}^N F_i\right] \le \sum_{i=1}^N \P[F_i] = N(1-p)^K,
  \]
  so $\P[G] \ge 1 - N(1-p)^K$.
  Using $1-p \le e^{-p}$ yields $(1-p)^K \le e^{-pK}$, hence also $\P[G] \ge 1 - N e^{-pK}$.

  \paragraph{Step 4: transferring the guarantee.}
  By Step 2 and the assumption \eqref{eq:assumption-P-Rd},
  \[
    \P[\cE \mid G] = \P\left[\cE \text{ holds for } \Out(P)\right] \ge 1-\delta.
  \]
  Therefore,
  \[
    \P[\cE]
    \ge
    \P[\cE \cap G]
    =
    \P[\cE \mid G]\cdot \P[G]
    \ge
    (1-\delta)\P[G]
    \ge
    1-\delta-\P[G^c]
    \ge
    1-\delta - N(1-p)^K,
  \]
  which gives the first inequality in the last paragraph of the lemma.
  The second inequality follows from $(1-p)^K \le e^{-pK}$.
\end{proof}

\paragraph{Choosing $K$.}
In \cref{lem:transfer-rejection-Rd} we obtained
\[
  \P[\cE \text{ holds for } \Out^{\texttt{RS}}] \;\ge\; 1-\delta - N(1-p)^K,
\]
where $p=Q(A)$. To ensure
\[
  \P[\cE \text{ holds for } \Out^{\texttt{RS}}] \;\ge\; 1-2\delta,
\]
it suffices to make the rejection-sampling failure term at most $\delta$, i.e.
\[
  N(1-p)^K \le \delta.
\]
Equivalently,
\[
  (1-p)^K \le \frac{\delta}{N}.
\]
Taking logs (noting $\log(1-p)<0$) gives the exact requirement
\[
  K \;\ge\; \frac{\log(N/\delta)}{-\log(1-p)}.
\]

\paragraph{Convenient sufficient choice.}
Using $-\log(1-p) \ge p$ (valid for $p\in(0,1)$), a simpler sufficient condition is
\[
  K \;\ge\; \frac{1}{p}\log\!\left(\frac{N}{\delta}\right),
\]
which guarantees $(1-p)^K \le e^{-pK} \le \delta/N$ and hence $N(1-p)^K \le \delta$.

\section{Generic Guarantee for \apxspanner}
\label{sec:robustspanner}

This section recalls the \apxspanner algorithm and its guarantee from \citep{mhammedi2023efficient}.
The algorithm computes an approximate barycentric spanner using oracles $\apx$ and $\est$ satisfying the following assumption.
\begin{assumption}[$\apx$ and $\est$ as approximate Linear Optimization Oracles]
  \label{ass:spanner}
  For some abstract set $\cZ$ and a collection of vectors $\{w^{z}\in \reals^d \mid z\in \cZ\}$ indexed by elements in $\cZ$, there exists $\veps'>0$ such that for any $\theta \in \reals^d\setminus\{0\}$ and $z\in \cZ$, the outputs $\hat z_{\theta} \coloneqq \apx(\theta/\|\theta\|)$ and $\hat w_z \coloneqq \est(z)$ satisfy
  \begin{align}
    \sup_{z\in \cZ} \theta^\top w^{z}
    \leq 	\theta^\top w^{\hat z_\theta} +\veps' \cdot \|\theta\|,\quad \text{and} \quad
    \|\hat w_z - w^{z}\| \leq \veps' .
  \end{align}
\end{assumption}
Let $\cW \coloneqq\{w^z \mid z\in \cZ\}$ and assume $\cW\subseteq \cB(1)$. The following proposition, adapted from \citet[Proposition E.1]{mhammedi2023efficient}, bounds the number of iterations of $\apxspanner$ and shows that the output is an approximate barycentric spanner for $\cW$ (see \cref{sec:framework}).
\begin{proposition}[{\citet[Proposition E.1]{mhammedi2023efficient}}]
  \label{prop:spanner}
  Fix $C>1$ and $\veps\in(0,1)$ and suppose that $\{w^z \mid z \in \cZ\}\subseteq \cB(1)$. If \apxspanner{} (\Cref{alg:spanner}) is run with parameters $C, \veps>0$ and oracles $\apx$, $\est$ satisfying \cref{ass:spanner} with parameter $\veps'$, then it terminates after $d + \ceil{\frac{d}{2} \log_C\frac{100 d}{\veps^2}}$ iterations, and requires at most twice that many calls to each of $\apx$ and $\est$. Furthermore, the output $z_{1:d}$ has the property that for all $z\in \cZ$, there exist $\beta_{1},\dots,\beta_d\in[-C,C]$, such that
  \begin{align}
    \nrm*{w^z - \sum_{i=1}^d\beta_i w^{z_i}}\leq 3Cd(\veps + 2\veps'). \label{eq:approxspanner}
  \end{align}
\end{proposition}

\section{Helper Lemmas}

\begin{lemma}[Uniform amplification for a subspace target]
  \label{lem:uniform-amplify-subspace}
  Let $\delta\in(0,1)$ and let $\bm{u}_1,\bm{u}_2,\dots$ be i.i.d.\ random vectors in $\mathbb{R}^d$.
  Fix a linear subspace $S\subseteq \mathbb{R}^d$ and an integer $n\ge 1$.
  Let $r \coloneqq \dim(S)$ and assume $r\ge 1$.
  Assume that for every $v\in S$,
  \[
    \P\left[v \in \Span(\bm{u}_1,\dots,\bm{u}_n)\right] \ge \frac{1}{2}.
  \]
  Let $k \coloneqq \left\lceil \log_2\left(\frac{r}{\delta}\right)\right\rceil$ and $m \coloneqq kn$.
  Then \[\P\left[S \subseteq \Span(\bm{u}_1,\dots,\bm{u}_m)\right] \ge 1-\delta.\]
\end{lemma}

\begin{proof}
  Fix a basis $b_1,\dots,b_r$ of $S$.
  For each $i\in[r]$ and each $j\in\{1,\dots,k\}$ define the block event
  \[
    A_j(b_i) \coloneqq \left\{b_i \in \Span\left\{\bm{u}_{(j-1)n+1},\dots,\bm{u}_{jn}\right\}\right\}.
  \]
  By the i.i.d.\ assumption, for every fixed $i\in[r]$, the events $A_1(b_i),\dots,A_k(b_i)$ are independent and identically distributed, and
  \[
    \P\left[A_j(b_i)\right]
    = \P\left[b_i \in \Span\left\{\bm{u}_1,\dots,\bm{u}_n\right\}\right]
    \ge \frac{1}{2}
    \qquad\text{for all } j\in\{1,\dots,k\}.
  \]

  Let $W \coloneqq \Span\left\{\bm{u}_1,\dots,\bm{u}_m\right\}$. If $A_j(b_i)$ occurs for some $j$, then
  $b_i \in \Span\left\{\bm{u}_{(j-1)n+1},\dots,\bm{u}_{jn}\right\} \subseteq W$, hence
  \[
    \bigcup_{j=1}^k A_j(b_i) \subseteq \{b_i\in W\}.
  \]
  Therefore,
  \[
    \P[b_i\notin W]
    \le \P\left[\bigcap_{j=1}^k A_j(b_i)^c\right]
    = \prod_{j=1}^k \P\left[A_j(b_i)^c\right]
    \le \left(1-\frac{1}{2}\right)^k
    = 2^{-k}.
  \]
  By the choice $k=\left\lceil \log_2\left(r/\delta\right)\right\rceil$, we have $2^{-k}\le \delta/r$, hence
  \[
    \P[b_i\in W] \ge 1-\frac{\delta}{r}.
  \]

  Using a union bound over $i\in[r]$, we obtain
  \[
    \P\left[\forall i\in[r],\ b_i\in W\right]
    \ge 1-\sum_{i=1}^r \P[b_i\notin W]
    \ge 1-\delta.
  \]
  On the event $\{\forall i\in[r],\ b_i\in W\}$ we have
  $S=\Span(b_1,\dots,b_r)\subseteq W$. This gives
  \[
    \P\left[S \subseteq \Span\left\{\bm{u}_1,\dots,\bm{u}_m\right\}\right] \ge 1-\delta.
  \]
\end{proof}

\begin{lemma}[Noiseless linear regression]
  \label{lem:noiseless-regression-orthogonality}
  Fix $\ell\in[h]$ and let $\cD_\ell$ be a finite dataset of triples $(x,a,R)$.
  For each $(x,a,R)\in\cD_\ell$ define
  \[
    \Sigma_\ell \coloneqq \sum_{(x,a,R)\in\cD_\ell} \phi_\ell(x,a)\phi_\ell(x,a)^\top,
    \qquad
    b_\ell \coloneqq \sum_{(x,a,R)\in\cD_\ell} \phi_\ell(x,a) R.
  \]
  Assume linear realizability with deterministic targets: there exists $\theta_\ell^\star\in\reals^d$ such that
  \[
    R = \phi_\ell(x,a)^\top \theta_\ell^\star \qquad \text{for all } (x,a,R)\in\cD_\ell.
  \]
  Define the (minimum-norm) least-squares estimator $\wh\theta_\ell \coloneqq \Sigma_\ell^\dagger  b_\ell$.
  Then
  \begin{align}
    \Sigma_\ell(\wh\theta_\ell-\theta_\ell^\star)=0. \label{eq:noiseless-normal-eq}
  \end{align}
  In particular, for every $u\in \mathrm{Col}(\Sigma_\ell)=\Span\{\phi_\ell(x,a):(x,a,R)\in\cD_\ell\}$,
  \begin{align}
    \iprod{\wh\theta_\ell-\theta_\ell^\star}{u}=0, \label{eq:noiseless-orth}
  \end{align}
  and hence for every $(x,a,R)\in\cD_\ell$,
  \begin{align}
    \phi_\ell(x,a)^\top \wh\theta_\ell = \phi_\ell(x,a)^\top \theta_\ell^\star = R. \label{eq:noiseless-interp}
  \end{align}
\end{lemma}

\begin{proof}
  By the realizability assumption,
  \[
    b_\ell
    = \sum_{(x,a,R)\in\cD_\ell} \phi_\ell(x,a) \phi_\ell(x,a)^\top \theta_\ell^\star
    = \Sigma_\ell \theta_\ell^\star.
  \]
  Therefore,
  \[
    \wh\theta_\ell
    = \Sigma_\ell^\dagger b_\ell
    = \Sigma_\ell^\dagger \Sigma_\ell \theta_\ell^\star.
  \]
  Multiplying by $\Sigma_\ell$ on the left and using $\Sigma_\ell\Sigma_\ell^\dagger\Sigma_\ell=\Sigma_\ell$ gives
  \[
    \Sigma_\ell \wh\theta_\ell
    = \Sigma_\ell\Sigma_\ell^\dagger\Sigma_\ell \theta_\ell^\star
    = \Sigma_\ell \theta_\ell^\star,
  \]
  which proves \eqref{eq:noiseless-normal-eq}. For \eqref{eq:noiseless-orth}, fix any $u\in\mathrm{Col}(\Sigma_\ell)$. Then, there exists $v\in\reals^d$ such that
  $u=\Sigma_\ell v$. Since $\Sigma_\ell$ is symmetric, we have
  \[
    \iprod{\wh\theta_\ell-\theta_\ell^\star}{u}
    = \iprod{\wh\theta_\ell-\theta_\ell^\star}{\Sigma_\ell v}
    = \iprod{\Sigma_\ell(\wh\theta_\ell-\theta_\ell^\star)}{v}
    = 0,
  \]
  where the last equality uses \eqref{eq:noiseless-normal-eq}. Finally, for any $(x,a,R)\in\cD_\ell$, we have $\phi_\ell(x,a)\in\mathrm{Col}(\Sigma_\ell)$, so applying \eqref{eq:noiseless-orth}
  with $u=\phi_\ell(x,a)$ yields $\phi_\ell(x,a)^\top(\wh\theta_\ell-\theta_\ell^\star)=0$, which is exactly \eqref{eq:noiseless-interp}.
\end{proof}

\begin{lemma}
\label{lem:decomp}
Let $r_{1:H}$ be any collection of reward functions with $r_h:\cX\times\cA\to\reals$. For any policy $\pi\in \Pi$ and $\Qhat_{1:H} \subseteq \{ f\colon \mathcal{X}\times\mathcal{A}\to\reals \}$, we have 
\begin{align}
 & \E\left[ V_1^{\pi}(\x_1;r_{1:H})\right] - \E \left[ \max_{a\in \cA} \Qhat_1(\x_1,a)\right] \nn \\
 & \quad =  \sum_{h=1}^H \E^{\pi} \left[ \left(\cT_h^r[\Qhat_{h+1}] - \Qhat_{h}\right)(\x_h,\a_h)\right] +  \sum_{h=1}^H \E^{\pi}\left[ \Qhat_{h}(\x_h,\a_h)- \wh V_h(\x_h)\right],
\end{align}
where $\cT_h^r[Q](x,a) \coloneqq r_h(x,a) + \E\left[\max_{a'\in \cA}Q(\x_{h+1}, a')  \mid \x_h =x, \a_h =a\right]$ is the Bellman operator under rewards $r_{1:H}$.
\end{lemma}

\begin{proof}
 Fix $\pi\in \Pi$. Let $\wh V_h(\cdot) \coloneqq \max_{a\in\cA} \Qhat_h(\cdot,a)$.
  By definition of the operators, we have
  \begin{align*}
    \E^{\pi} \left[ \left(\cT_h^r[\Qhat_{h+1}] - \Qhat_{h}\right)(\x_h,\a_h)\right]
     & = \E^{\pi} \left[ r_h(\x_h,\a_h) + \wh V_{h+1}(\x_{h+1}) - \Qhat_{h}(\x_h,\a_h)\right].
  \end{align*}
  Summing over $h=1,\dots,H$ and using linearity of expectation:
  \begin{align*}
    \text{RHS} & = \E^{\pi}\left[\sum_{h=1}^H r_h(\x_h,\a_h)\right] + \E^{\pi}\left[\sum_{h=1}^H \left(\wh V_{h+1}(\x_{h+1}) - \Qhat_{h}(\x_h,\a_h)\right)\right]
    \\
               & = \E\left[V_1^\pi(\x_1;r_{1:H})\right] + \E^{\pi}\left[\sum_{h=1}^H \left(\wh V_{h+1}(\x_{h+1}) - \wh V_h(\x_h) + \wh V_h(\x_h) - \Qhat_{h}(\x_h,\a_h)\right)\right].
  \end{align*}
  The term $\sum_{h=1}^H (\wh V_{h+1}(\x_{h+1}) - \wh V_h(\x_h))$ telescopes to $\wh V_{H+1} - \wh V_1 = -\wh V_1$ (since $\wh V_{H+1} \equiv 0$).
  Thus,
  \begin{align*}
    \text{RHS} = \E\left[V_1^\pi(\x_1;r_{1:H})\right] - \E\left[\wh V_1(\x_1)\right] + \sum_{h=1}^H \E^{\pi}\left[\wh V_h(\x_h) - \Qhat_{h}(\x_h,\a_h)\right].
  \end{align*}
  Rearranging the terms yields the result.
\end{proof}

\begin{corollary}[Bellman residual decomposition]
\label{cor:decomp-diff}
Let $r_{1:H}$ be any collection of reward functions with $r_h:\cX\times\cA\to\reals$. For any policy $\pi\in \Pi$ and $\Qhat_{1:H} \subseteq \{ f\colon \mathcal{X}\times\mathcal{A}\to\reals \}$. Further, let $\pihat_h(\cdot) \in \argmax_{a\in \cA} \Qhat_h(\cdot ,a)$, for all $h\in[H]$. Then, we have 
\begin{align}
 & \E\left[ V_1^{\pi}(\x_1;r_{1:H})\right] - \E \left[V_1^{\pihat}(\x_1;r_{1:H})\right] \nn \\
 &\quad  \leq \sum_{h=1}^H \E^{\pi} \left[ \left(\cT_h^r[\Qhat_{h+1}] - \Qhat_{h}\right)(\x_h,\a_h)\right]  -  \sum_{h=1}^H \E^{\pihat} \left[ \left(\cT_h^r[\Qhat_{h+1}] - \Qhat_{h}\right)(\x_h,\a_h)\right].
\end{align}
\end{corollary}

\begin{proof}
  We decompose the difference as
  \begin{align*}
   & \E\left[ V_1^{\pi}(\x_1;r_{1:H})\right] - \E \left[V_1^{\pihat}(\x_1;r_{1:H})\right] \nn \\
   & \quad  =
    \underbrace{\E\left[ V_1^{\pi}(\x_1;r_{1:H})\right] - \E \left[ \wh V_1(\x_1)\right]}_{(\mathrm{i})}
    +
    \underbrace{\E\left[ \wh V_1(\x_1)\right] - \E \left[V_1^{\pihat}(\x_1;r_{1:H})\right]}_{(\mathrm{ii})},
  \end{align*}
  where $\wh V_1(x) \coloneqq \max_{a\in\cA} \Qhat_1(x,a)$.
  For the first term $(\mathrm{i})$, \cref{lem:decomp} implies
  \begin{align*}
    (\mathrm{i}) = \sum_{h=1}^H \E^{\pi} \left[ \left(\cT_h^r[\Qhat_{h+1}] - \Qhat_{h}\right)(\x_h,\a_h)\right] + \sum_{h=1}^H \E^{\pi}\left[ \Qhat_{h}(\x_h,\a_h)- \wh V_h(\x_h)\right].
  \end{align*}
  Since $\wh V_h(x) = \max_a \Qhat_h(x,a)$, we have $\Qhat_h(x,a) - \wh V_h(x) \le 0$ for all $(x,a)$. Thus, the second sum is non-positive, so
  \begin{align*}
    (\mathrm{i}) \le \sum_{h=1}^H \E^{\pi} \left[ \left(\cT_h^r[\Qhat_{h+1}] - \Qhat_{h}\right)(\x_h,\a_h)\right].
  \end{align*}
  For the second term $(\mathrm{ii})$, applying \cref{lem:decomp} to $\pihat$ gives
  \begin{align*}
     - (\mathrm{ii})&  = \E \left[V_1^{\pihat}(\x_1;r_{1:H})\right] - \E\left[ \wh V_1(\x_1)\right], \nn \\
     & = \sum_{h=1}^H \E^{\pihat} \left[ \left(\cT_h^r[\Qhat_{h+1}] - \Qhat_{h}\right)(\x_h,\a_h)\right] + \sum_{h=1}^H \E^{\pihat}\left[ \Qhat_{h}(\x_h,\a_h)- \wh V_h(\x_h)\right].
  \end{align*}
  Since $\pihat$ is greedy with respect to $\Qhat$, we have $\Qhat_h(x,\pihat_h(x)) = \max_a \Qhat_h(x,a) = \wh V_h(x)$, so the second sum is exactly zero.
  Thus,
  \begin{align*}
    (\mathrm{ii}) = - \sum_{h=1}^H \E^{\pihat} \left[ \left(\cT_h^r[\Qhat_{h+1}] - \Qhat_{h}\right)(\x_h,\a_h)\right].
  \end{align*}
  Combining the bounds for $(\mathrm{i})$ and $(\mathrm{ii})$ yields the result.
\end{proof}

\begin{lemma}[Sequential union bound with random stopping]
\label{lem:sequnionbound-stopping}
Let $T\in\mathbb{N}$ and $\delta\in(0,1)$. Consider an algorithm $\mathcal{A}$ with state space $\StateSpace$.
During its execution, $\mathcal{A}$ may call a subroutine $\mathcal{B}$ repeatedly, for at most $T$ calls in total.
Index these \emph{potential} calls by $i\in\{1,2,\dots,T\}$, and let $\tau \in \{0,1,\dots,T\}$ be the (random) total number of calls that are actually executed.

For each $i\in\{1,\dots,T\}$, let $\mathbf{S}_{i,-}\in\StateSpace$ be the random state of $\mathcal{A}$ immediately \emph{before} the $i$-th call to $\mathcal{B}$,
and let $\mathbf{S}_{i,+}\in\StateSpace$ be the random state of $\mathcal{A}$ immediately \emph{after} that call (when executed).

Let $\mathcal{F}_{i-1}$ be the $\sigma$-algebra generated by the history of $\mathcal{A}$ up to (and including) $\mathbf{S}_{i,-}$,
and assume that $\tau$ is a stopping time with respect to $(\mathcal{F}_i)_{i\ge 0}$.

Assume further that for each $i$ and each $s\in\StateSpace$, there exists a measurable set $\mathcal{E}_i(s)\subseteq\StateSpace$
such that the event $\{\mathbf{S}_{i,+}\in\mathcal{E}_i(\mathbf{S}_{i,-})\}$ is measurable and
\[
\P\left[\mathbf{S}_{i,+}\in\mathcal{E}_i(s)\mid \mathbf{S}_{i,-}=s\right]\ge 1-\delta.
\]
Then
\[
\P\left[\forall i\in\{1,\dots,\tau\}:\ \mathbf{S}_{i,+}\in\mathcal{E}_i(\mathbf{S}_{i,-})\right]
\ge 1-\delta\,\E[\tau].
\]
In particular, since $\tau\le T$ almost surely, the right-hand side is at least $1-T\delta$.
\end{lemma}

\begin{proof}
For each $i\in\{1,\dots,T\}$ define the failure event
\[
F_i \coloneqq \left\{\mathbf{S}_{i,+}\notin\mathcal{E}_i(\mathbf{S}_{i,-})\right\}.
\]
We use the pointwise bound
\[
\mathbb{I}\left\{\bigcup_{i=1}^{\tau} F_i\right\} \le \sum_{i=1}^{\tau} \mathbb{I}\{F_i\}.
\]
Taking expectations and expanding the random sum,
\[
\P\left[\bigcup_{i=1}^{\tau} F_i\right]
\le \E\left[\sum_{i=1}^{\tau} \mathbb{I}\{F_i\}\right]
= \sum_{i=1}^{T} \E\left[\mathbb{I}\{i\le\tau\}\cdot \mathbb{I}\{F_i\}\right].
\]
By the tower property,
\[
\E\left[\mathbb{I}\{i\le\tau\}\cdot \mathbb{I}\{F_i\}\right]
= \E\left[\mathbb{I}\{i\le\tau\}\cdot \E\left[\mathbb{I}\{F_i\}\mid \mathcal{F}_{i-1}\right]\right].
\]
Since $\tau$ is a stopping time, the event $\{i\le\tau\}$ is $\mathcal{F}_{i-1}$-measurable. Moreover,
\[
\E\left[\mathbb{I}\{F_i\}\mid \mathcal{F}_{i-1}\right]
= \P\left[F_i\mid \mathcal{F}_{i-1}\right]
= \E\left[\P\left[F_i\mid \mathbf{S}_{i,-}\right]\mid \mathcal{F}_{i-1}\right].
\]
By assumption, for every realization $s\in\StateSpace$ we have
\[
\P\left[F_i\mid \mathbf{S}_{i,-}=s\right]
= 1-\P\left[\mathbf{S}_{i,+}\in\mathcal{E}_i(s)\mid \mathbf{S}_{i,-}=s\right]
\le \delta,
\]
so $\P[F_i\mid \mathcal{F}_{i-1}]\le\delta$ almost surely. Hence
\[
\E\left[\mathbb{I}\{i\le\tau\}\cdot \mathbb{I}\{F_i\}\right]
\le \E\left[\mathbb{I}\{i\le\tau\}\cdot \delta\right]
= \delta\,\P[i\le\tau].
\]
Summing over $i$ gives
\[
\P\left[\bigcup_{i=1}^{\tau} F_i\right]
\le \delta \sum_{i=1}^{T} \P[i\le\tau]
= \delta\,\E[\tau],
\]
where the last identity uses $\E[\tau]=\sum_{i=1}^{T}\P[i\le\tau]$ for $\tau\in\{0,1,\dots,T\}$.
Taking complements yields the stated success bound.
\end{proof}

\begin{lemma}[Boosting a conditional spanning guarantee by oversampling]
\label{lem:boost-conditional-span}
Let $S\subseteq \R^d$ be a linear subspace, and let $\bm{x}$ be an $\R^d$-valued random variable with
\[
p \coloneqq \P[\bm{x}\notin S] \ge \veps.
\]
Let $Q$ denote the conditional law $Q \coloneqq \Law(\bm{x}\mid \bm{x}\notin S)$, and let $u\in S^\perp$ be fixed.
Assume that for some $n\in\mathbb{N}$ and $\delta\in(0,1)$,
\[
Q^{\otimes n}\left[u \in \Span \left\{\bm{y}_1^\perp,\dots,\bm{y}_n^\perp\right\}\right] \ge 1-\delta,
\]
where $\bm{y}_1,\dots,\bm{y}_n$ are i.i.d.\ from $Q$ and $(\cdot)^\perp$ denotes orthogonal projection onto $S^\perp$.

Set $m \coloneqq \left\lceil \frac{n}{\veps} + \frac{2 \sqrt{\log (1/\delta) n}}{\veps} + \frac{4 \log(1/\delta)}{\veps}\right\rceil$ and let $\bm{x}_1,\dots,\bm{x}_m$ be i.i.d.\ copies of $\bm{x}$. Then
\[
\P\left[u \in \Span\left\{\bm{x}_1^\perp,\dots,\bm{x}_m^\perp\right\}\right]
\ge 1-2\delta.
\]
\end{lemma}

\begin{proof}
Define the indicators $I_i \coloneqq \mathbb{I}\{\bm{x}_i\notin S\}$ and the count
\[
N \coloneqq \sum_{i=1}^m I_i.
\]
Then $N\sim \mathrm{Bin}(m,p)$ with mean $\mu \coloneqq \E[N]=mp$. Since $p\ge \veps$ and $m=\left\lceil (1+c)n/\veps\right\rceil$,
\[
\mu = mp \ge m\veps \ge (1+c)n.
\]
Let $J \coloneqq \{i\in\{1,\dots,m\} : I_i=1\}$ be the random index set of samples that fall outside $S$.
On the event $\{N\ge n\}$, define $i_1<\dots<i_n$ to be the first $n$ indices in $J$.

We first relate the target event to these $n$ selected indices. Since
\[
\Span\left\{\bm{x}_{i_1}^\perp,\dots,\bm{x}_{i_n}^\perp\right\} \subseteq \Span\left\{\bm{x}_1^\perp,\dots,\bm{x}_m^\perp\right\},
\]
we have
\[
\P\left[u \in \Span\left\{\bm{x}_1^\perp,\dots,\bm{x}_m^\perp\right\}\right]
\ge \P\left[\{N\ge n\}\cap \left\{u \in \Span\left\{\bm{x}_{i_1}^\perp,\dots,\bm{x}_{i_n}^\perp\right\}\right\}\right].
\]
By conditioning on $(I_1,\dots,I_m)$,
\[
\P\left[\{N\ge n\}\cap \left\{u \in \Span\left\{\bm{x}_{i_1}^\perp,\dots,\bm{x}_{i_n}^\perp\right\}\right\}\right]
= \E\left[\mathbb{I}\{N\ge n\}\,
\P\left[u \in \Span\left\{\bm{x}_{i_1}^\perp,\dots,\bm{x}_{i_n}^\perp\right\}\mid I_1,\dots,I_m\right]\right].
\]

We now analyze the conditional probability inside the expectation.
Fix any realization of $(I_1,\dots,I_m)$ with $N\ge n$. Conditional on $(I_1,\dots,I_m)$,
the variables $\{\bm{x}_i : I_i=1\}$ are independent and each has law $Q$.
Therefore the selected tuple $(\bm{x}_{i_1},\dots,\bm{x}_{i_n})$ is i.i.d.\ from $Q$, and hence
\[
\P\left[u \in \Span\left\{\bm{x}_{i_1}^\perp,\dots,\bm{x}_{i_n}^\perp\right\}\mid I_1,\dots,I_m\right]
= Q^{\otimes n}\left[u \in \Span\left\{\bm{y}_1^\perp,\dots,\bm{y}_n^\perp\right\}\right]
\ge 1-\delta.
\]
Plugging this back gives
\[
\P\left[u \in \Span\left\{\bm{x}_1^\perp,\dots,\bm{x}_m^\perp\right\}\right]
\ge (1-\delta)\P[N\ge n].
\]

It remains to lower bound $\P[N\ge n]$.
Since $\mu\ge (1+c)n$, we have $n \le \mu/(1+c)$, so
\[
\{N<n\}\subseteq \left\{N \le \left(1-\alpha\right)\mu\right\}
\quad\text{where}\quad
\alpha \coloneqq 1-\frac{1}{1+c} = \frac{c}{1+c}.
\]
By the multiplicative Chernoff bound for a binomial random variable,
\[
\P\left[N \le \left(1-\alpha\right)\mu\right] \le \exp\left(-\frac{\alpha^2}{2}\mu\right).
\]
Therefore,
\[
\P[N<n] \le \exp\left(-\frac{1}{2}\left(\frac{c}{1+c}\right)^2 \mu\right)
\le \exp\left(-\frac{1}{2}\left(\frac{c}{1+c}\right)^2 (1+c)n\right)
= \exp\left(-\frac{c^2}{2(1+c)}\,n\right),
\]
and hence
\[
\P[N\ge n] \ge 1-\exp\left(-\frac{c^2}{2(1+c)}\,n\right).
\]
Combining the bounds and setting $c \geq \max\left\{2\sqrt{\frac{L}{n}},\,\frac{4L}{n}\right\}$ completes the proof.
\end{proof}

\begin{lemma}[Ridge regression confidence ellipsoid]
  \label{lem:ridge-ellipsoid}
  Let $\lambda>0$ and let $(\cF_k)_{k\ge 0}$ be a filtration.
  Let $A_k\in\reals^d$ be $\cF_{k-1}$-measurable and let
  \[
    X_k = \iprod{A_k}{\theta_\star} + \eta_k,
  \]
  where $\theta_\star\in\reals^d$ is fixed and $(\eta_k)$ is conditionally $1$-subgaussian:
  for all $\alpha\in\reals$,
  \[
    \E\left[\exp\left(\alpha \eta_k\right)\mid \cF_{k-1}\right] \le \exp\left(\alpha^2/2\right)
    \qquad \text{a.s.}
  \]
  Define
  \[
    V_k(\lambda)\coloneqq \lambda I + \sum_{s=1}^k A_sA_s^\top,
    \qquad
    \wh\theta_k \coloneqq V_k(\lambda)^{-1}\sum_{s=1}^k X_sA_s.
  \]
  Then for any $\zeta\in(0,1)$, with probability at least $1-\zeta$, it holds that for all $k\in\mathbb{N}$,
  \[
    \|\wh\theta_k-\theta_\star\|_{V_k(\lambda)}
    \le
    \sqrt{\lambda}\,\|\theta_\star\|_2
    +
    \sqrt{2\log\left(1/\zeta\right)+\log\left(\det\left(V_k(\lambda)\right)/\lambda^d\right)}.
  \]
\end{lemma}

\begin{corollary}[Directional ridge bound]
  \label{cor:ridge-direction}
  On the event of \cref{lem:ridge-ellipsoid}, for any fixed $u\in\reals^d$ and any $k\in\mathbb{N}$,
  \[
    \left|\iprod{\wh\theta_k-\theta_\star}{u}\right|
    \le
    \beta_k(\zeta)\cdot \|u\|_{V_k(\lambda)^{-1}},
  \]
  where
  \[
    \beta_k(\zeta)
    \coloneqq
    \sqrt{\lambda}\,\|\theta_\star\|_2
    +
    \sqrt{2\log\left(1/\zeta\right)+\log\left(\det\left(V_k(\lambda)\right)/\lambda^d\right)}.
  \]
\end{corollary}

\begin{corollary}[Estimator Norm Bound]
  \label{cor:ridge-norm}
  On the event of \cref{lem:ridge-ellipsoid}, for all $k \in \mathbb{N}$,
  \begin{align}
      \|\wh\theta_k\|_2 \le \|\theta_\star\|_2 + \frac{1}{\sqrt{\lambda}} \beta_k(\zeta),
  \end{align}
  where $\beta_k(\zeta)$ is defined in \cref{cor:ridge-direction}.
\end{corollary}
\begin{proof}
    By triangle inequality, $\|\wh\theta_k\|_2 \le \|\theta_\star\|_2 + \|\wh\theta_k - \theta_\star\|_2$.
    Since $V_k(\lambda) \succeq \lambda I$, we have $\|u\|_{V_k(\lambda)}^2 = u^\top V_k(\lambda) u \ge \lambda \|u\|_2^2$, which implies $\|u\|_2 \le \frac{1}{\sqrt{\lambda}} \|u\|_{V_k(\lambda)}$.
    Applying this to the error term: $\|\wh\theta_k - \theta_\star\|_2 \le \frac{1}{\sqrt{\lambda}} \|\wh\theta_k - \theta_\star\|_{V_k(\lambda)}$.
    Using the bound from \cref{lem:ridge-ellipsoid}, $\|\wh\theta_k - \theta_\star\|_{V_k(\lambda)} \le \beta_k(\zeta)$.
    Combining these gives the result.
\end{proof}

\paragraphi{A convenient determinant bound.}
Under $\|A_s\|_2\le 1$ for all $s$ we have
\begin{align}
  \log\left(\frac{\det\left(V_k(\lambda)\right)}{\lambda^d}\right)
  \le d\log\left(1+\frac{k}{\lambda d}\right).
  \label{eq:det_bound_simple}
\end{align}
(Apply $\det\left(\lambda I+M\right)\le \left(\lambda+\tfrac{1}{d}\mathrm{tr}(M)\right)^d$ with $M=\sum_{s=1}^k A_sA_s^\top$.)

\begin{lemma}[Uniform convergence for linear classes]
  \label{lem:uniform-convergence}
  Let $\Theta \subset \reals^d$ be a class of parameter vectors with $\|\theta\| \le C$ for all $\theta \in \Theta$.
   Let $(\x,\a), (\x\ind{1}, \a\ind{1}), \dots, (\x\ind{n}, \a\ind{n}) \in \cX \times \cA$ be i.i.d.~samples from some distribution. Further, let $\phi:\cX\times\cA\to\reals^d$ be a feature map such that $\sup_{x,a}\|\phi(x,a)\|_2\le 1$.
  For any fixed $\delta \in (0,1)$, with probability at least $1-\delta$, we have:
  \begin{align}
      \sup_{\theta \in \Theta} \left| \E[\phi(\x,\a)^\top \theta] - \frac{1}{n} \sum_{i=1}^n \phi(\x\ind{i}, \a\ind{i})^\top \theta \right| \le C \sqrt{\frac{8 (d \log(1+2nC) + \log(2/\delta))}{n}} + \frac{2}{n}.
  \end{align}
\end{lemma}
\begin{proof}
  We apply a covering number argument. Let $\eps = 1/n$. Let $\cN_\eps$ be an $\eps$-cover of $\Theta$ in the $\ell_2$-norm.
  Standard results for Euclidean balls imply $|\cN_\eps| \le (1 + 2C/\eps)^d = (1 + 2nC)^d$.
  For any fixed $\theta \in \cN_\eps$, let $\bm{Z}\ind{i} = \phi(\x\ind{i}, \a\ind{i})^\top \theta$. Since $\sup_{x,a}\|\phi(x,a)\|_2 \le 1$ and $\|\theta\| \le C$, we have $|\bm{Z}\ind{i}| \le C$ almost surely.
  We apply Lemma C.1 (Azuma-Hoeffding) of \cite{foster2025goodfoundationnecessaryefficient} to the sequence $\bm{Z}\ind{i}$.
  With probability $1-\delta'$, 
  \[ \left| \sum_{i=1}^n \bm{Z}\ind{i} - n\E[\bm{Z}\ind{i}] \right| \le C \sqrt{8n \log(2/\delta')}. \]
  Dividing by $n$, we get $|\frac{1}{n}\sum_{i=1}^n \bm{Z}\ind{i} - \E[\bm{Z}\ind{i}]| \le C \sqrt{\frac{8 \log(2/\delta')}{n}}$.
  
  We take a union bound over all $\theta \in \cN_\eps$. Setting $\delta' = \delta / |\cN_\eps|$, the bound holds simultaneously for all $\theta \in \cN_\eps$ with probability at least $1-\delta$.
  Substituting $|\cN_\eps| \le (1+2nC)^d$, the log term becomes $\log(2/\delta') = \log(2/\delta) + \log(|\cN_\eps|) \le \log(2/\delta) + d \log(1+2nC)$.
  Thus, under the success event, we have:
  \[ \left| \E[\phi(\x,\a)^\top \theta] - \frac{1}{n}\sum_{i=1}^n \phi(\x\ind{i}, \a\ind{i})^\top \theta \right| \le C \sqrt{\frac{8(d \log(1+2nC) + \log(2/\delta))}{n}}. \]
  We condition on this success event for the remainder of the proof.
  
  Finally, for any $\theta \in \Theta$, let $\theta' \in \cN_\eps$ be such that $\|\theta - \theta'\| \le \eps$.
  The error decomposes as:
  \begin{align}
      &\left| \E[\phi(\x,\a)^\top \theta] - \frac{1}{n} \sum_{i=1}^n \phi(\x\ind{i}, \a\ind{i})^\top \theta \right| \\
      &\le \left| \E[\phi(\x,\a)^\top \theta'] - \frac{1}{n} \sum_{i=1}^n \phi(\x\ind{i}, \a\ind{i})^\top \theta' \right| + \left| \E[\phi(\x,\a)^\top (\theta-\theta')] \right| + \left| \frac{1}{n} \sum_{i=1}^n \phi(\x\ind{i}, \a\ind{i})^\top (\theta-\theta') \right| \\
      &\le C \sqrt{\frac{8(d \log(1+2nC) + \log(2/\delta))}{n}} + \|\theta-\theta'\| + \|\theta-\theta'\| \\
      &\le C \sqrt{\frac{8(d \log(1+2nC) + \log(2/\delta))}{n}} + 2\eps.
  \end{align}
  Substituting $\eps = 1/n$ yields the result.
\end{proof}

\neurips{
\section*{NeurIPS Paper Checklist}

\begin{enumerate}

\item {\bf Claims}
    \item[] Question: Do the main claims made in the abstract and introduction accurately reflect the paper's contributions and scope?
    \item[] Answer: \answerYes{}
    \item[] Justification: The abstract and introduction clearly state our main contribution: a computationally efficient algorithm for linear Bellman complete MDPs with deterministic transitions that requires only an argmax oracle over actions, with polynomial sample and computational complexity. The formal guarantee is stated in \cref{thm:main}, which matches the claims. The key assumptions (deterministic transitions, linear Bellman completeness, bounded reward parameters) are explicitly stated in \cref{ass:bellman_complete,ass:features}.
    \item[] Guidelines:
    \begin{itemize}
        \item The answer \answerNA{} means that the abstract and introduction do not include the claims made in the paper.
        \item The abstract and/or introduction should clearly state the claims made, including the contributions made in the paper and important assumptions and limitations. A \answerNo{} or \answerNA{} answer to this question will not be perceived well by the reviewers. 
        \item The claims made should match theoretical and experimental results, and reflect how much the results can be expected to generalize to other settings. 
        \item It is fine to include aspirational goals as motivation as long as it is clear that these goals are not attained by the paper. 
    \end{itemize}

\item {\bf Limitations}
    \item[] Question: Does the paper discuss the limitations of the work performed by the authors?
    \item[] Answer: \answerYes{}
    \item[] Justification: The paper discusses its limitations throughout. The restriction to deterministic transitions is stated explicitly as a key assumption (\cref{ass:bellman_complete}) and motivated in the introduction (\cref{sec:prelims}). The need for an argmax oracle over actions for large/infinite action spaces is discussed in \cref{sec:prelims}. The open question of extending to stochastic transitions is highlighted as an important direction. The bounded reward parameter assumption is contrasted with stronger assumptions in prior work (\cref{sec:prelims}).
    \item[] Guidelines:
    \begin{itemize}
        \item The answer \answerNA{} means that the paper has no limitation while the answer \answerNo{} means that the paper has limitations, but those are not discussed in the paper. 
        \item The authors are encouraged to create a separate ``Limitations'' section in their paper.
        \item The paper should point out any strong assumptions and how robust the results are to violations of these assumptions (e.g., independence assumptions, noiseless settings, model well-specification, asymptotic approximations only holding locally). The authors should reflect on how these assumptions might be violated in practice and what the implications would be.
        \item The authors should reflect on the scope of the claims made, e.g., if the approach was only tested on a few datasets or with a few runs. In general, empirical results often depend on implicit assumptions, which should be articulated.
        \item The authors should reflect on the factors that influence the performance of the approach. For example, a facial recognition algorithm may perform poorly when image resolution is low or images are taken in low lighting. Or a speech-to-text system might not be used reliably to provide closed captions for online lectures because it fails to handle technical jargon.
        \item The authors should discuss the computational efficiency of the proposed algorithms and how they scale with dataset size.
        \item If applicable, the authors should discuss possible limitations of their approach to address problems of privacy and fairness.
        \item While the authors might fear that complete honesty about limitations might be used by reviewers as grounds for rejection, a worse outcome might be that reviewers discover limitations that aren't acknowledged in the paper. The authors should use their best judgment and recognize that individual actions in favor of transparency play an important role in developing norms that preserve the integrity of the community. Reviewers will be specifically instructed to not penalize honesty concerning limitations.
    \end{itemize}

\item {\bf Theory assumptions and proofs}
    \item[] Question: For each theoretical result, does the paper provide the full set of assumptions and a complete (and correct) proof?
    \item[] Answer: \answerYes{}
    \item[] Justification: All theoretical results (theorems, lemmas, corollaries) are formally stated with their full assumptions. The two main assumptions are stated in \cref{ass:bellman_complete,ass:features}. The main result (\cref{thm:main}) and intermediate results (\cref{thm:cover_guarantee,lem:notspanner_guarantee}) are stated with complete assumptions. A proof sketch is provided in the main text, and complete proofs are provided in the appendix (\cref{sec:lazyspanner,sec:spanner_guarantee_main,sec:policyopt-fqi,sec:mainproof}).
    \item[] Guidelines:
    \begin{itemize}
        \item The answer \answerNA{} means that the paper does not include theoretical results. 
        \item All the theorems, formulas, and proofs in the paper should be numbered and cross-referenced.
        \item All assumptions should be clearly stated or referenced in the statement of any theorems.
        \item The proofs can either appear in the main paper or the supplemental material, but if they appear in the supplemental material, the authors are encouraged to provide a short proof sketch to provide intuition. 
        \item Inversely, any informal proof provided in the core of the paper should be complemented by formal proofs provided in appendix or supplemental material.
        \item Theorems and Lemmas that the proof relies upon should be properly referenced. 
    \end{itemize}

    \item {\bf Experimental result reproducibility}
    \item[] Question: Does the paper fully disclose all the information needed to reproduce the main experimental results of the paper to the extent that it affects the main claims and/or conclusions of the paper (regardless of whether the code and data are provided or not)?
    \item[] Answer: \answerNA{}
    \item[] Justification: This paper is a theoretical contribution and does not include experiments.
    \item[] Guidelines:
    \begin{itemize}
        \item The answer \answerNA{} means that the paper does not include experiments.
        \item If the paper includes experiments, a \answerNo{} answer to this question will not be perceived well by the reviewers: Making the paper reproducible is important, regardless of whether the code and data are provided or not.
        \item If the contribution is a dataset and\slash or model, the authors should describe the steps taken to make their results reproducible or verifiable. 
        \item Depending on the contribution, reproducibility can be accomplished in various ways. For example, if the contribution is a novel architecture, describing the architecture fully might suffice, or if the contribution is a specific model and empirical evaluation, it may be necessary to either make it possible for others to replicate the model with the same dataset, or provide access to the model. In general. releasing code and data is often one good way to accomplish this, but reproducibility can also be provided via detailed instructions for how to replicate the results, access to a hosted model (e.g., in the case of a large language model), releasing of a model checkpoint, or other means that are appropriate to the research performed.
        \item While NeurIPS does not require releasing code, the conference does require all submissions to provide some reasonable avenue for reproducibility, which may depend on the nature of the contribution. For example
        \begin{enumerate}
            \item If the contribution is primarily a new algorithm, the paper should make it clear how to reproduce that algorithm.
            \item If the contribution is primarily a new model architecture, the paper should describe the architecture clearly and fully.
            \item If the contribution is a new model (e.g., a large language model), then there should either be a way to access this model for reproducing the results or a way to reproduce the model (e.g., with an open-source dataset or instructions for how to construct the dataset).
            \item We recognize that reproducibility may be tricky in some cases, in which case authors are welcome to describe the particular way they provide for reproducibility. In the case of closed-source models, it may be that access to the model is limited in some way (e.g., to registered users), but it should be possible for other researchers to have some path to reproducing or verifying the results.
        \end{enumerate}
    \end{itemize}

\item {\bf Open access to data and code}
    \item[] Question: Does the paper provide open access to the data and code, with sufficient instructions to faithfully reproduce the main experimental results, as described in supplemental material?
    \item[] Answer: \answerNA{}
    \item[] Justification: This paper is a theoretical contribution and does not include experiments or code.
    \item[] Guidelines:
    \begin{itemize}
        \item The answer \answerNA{} means that paper does not include experiments requiring code.
        \item Please see the NeurIPS code and data submission guidelines (\url{https://neurips.cc/public/guides/CodeSubmissionPolicy}) for more details.
        \item While we encourage the release of code and data, we understand that this might not be possible, so \answerNo{} is an acceptable answer. Papers cannot be rejected simply for not including code, unless this is central to the contribution (e.g., for a new open-source benchmark).
        \item The instructions should contain the exact command and environment needed to run to reproduce the results. See the NeurIPS code and data submission guidelines (\url{https://neurips.cc/public/guides/CodeSubmissionPolicy}) for more details.
        \item The authors should provide instructions on data access and preparation, including how to access the raw data, preprocessed data, intermediate data, and generated data, etc.
        \item The authors should provide scripts to reproduce all experimental results for the new proposed method and baselines. If only a subset of experiments are reproducible, they should state which ones are omitted from the script and why.
        \item At submission time, to preserve anonymity, the authors should release anonymized versions (if applicable).
        \item Providing as much information as possible in supplemental material (appended to the paper) is recommended, but including URLs to data and code is permitted.
    \end{itemize}

\item {\bf Experimental setting/details}
    \item[] Question: Does the paper specify all the training and test details (e.g., data splits, hyperparameters, how they were chosen, type of optimizer) necessary to understand the results?
    \item[] Answer: \answerNA{}
    \item[] Justification: This paper is a theoretical contribution and does not include experiments.
    \item[] Guidelines:
    \begin{itemize}
        \item The answer \answerNA{} means that the paper does not include experiments.
        \item The experimental setting should be presented in the core of the paper to a level of detail that is necessary to appreciate the results and make sense of them.
        \item The full details can be provided either with the code, in appendix, or as supplemental material.
    \end{itemize}

\item {\bf Experiment statistical significance}
    \item[] Question: Does the paper report error bars suitably and correctly defined or other appropriate information about the statistical significance of the experiments?
    \item[] Answer: \answerNA{}
    \item[] Justification: This paper is a theoretical contribution and does not include experiments.
    \item[] Guidelines:
    \begin{itemize}
        \item The answer \answerNA{} means that the paper does not include experiments.
        \item The authors should answer \answerYes{} if the results are accompanied by error bars, confidence intervals, or statistical significance tests, at least for the experiments that support the main claims of the paper.
        \item The factors of variability that the error bars are capturing should be clearly stated (for example, train/test split, initialization, random drawing of some parameter, or overall run with given experimental conditions).
        \item The method for calculating the error bars should be explained (closed form formula, call to a library function, bootstrap, etc.)
        \item The assumptions made should be given (e.g., Normally distributed errors).
        \item It should be clear whether the error bar is the standard deviation or the standard error of the mean.
        \item It is OK to report 1-sigma error bars, but one should state it. The authors should preferably report a 2-sigma error bar than state that they have a 96\% CI, if the hypothesis of Normality of errors is not verified.
        \item For asymmetric distributions, the authors should be careful not to show in tables or figures symmetric error bars that would yield results that are out of range (e.g., negative error rates).
        \item If error bars are reported in tables or plots, the authors should explain in the text how they were calculated and reference the corresponding figures or tables in the text.
    \end{itemize}

\item {\bf Experiments compute resources}
    \item[] Question: For each experiment, does the paper provide sufficient information on the computer resources (type of compute workers, memory, time of execution) needed to reproduce the experiments?
    \item[] Answer: \answerNA{}
    \item[] Justification: This paper is a theoretical contribution and does not include experiments.
    \item[] Guidelines:
    \begin{itemize}
        \item The answer \answerNA{} means that the paper does not include experiments.
        \item The paper should indicate the type of compute workers CPU or GPU, internal cluster, or cloud provider, including relevant memory and storage.
        \item The paper should provide the amount of compute required for each of the individual experimental runs as well as estimate the total compute. 
        \item The paper should disclose whether the full research project required more compute than the experiments reported in the paper (e.g., preliminary or failed experiments that didn't make it into the paper). 
    \end{itemize}
    
\item {\bf Code of ethics}
    \item[] Question: Does the research conducted in the paper conform, in every respect, with the NeurIPS Code of Ethics \url{https://neurips.cc/public/EthicsGuidelines}?
    \item[] Answer: \answerYes{}
    \item[] Justification: This work is purely theoretical and conforms with the NeurIPS Code of Ethics. It does not involve human subjects, personal data, or any ethically sensitive components.
    \item[] Guidelines:
    \begin{itemize}
        \item The answer \answerNA{} means that the authors have not reviewed the NeurIPS Code of Ethics.
        \item If the authors answer \answerNo, they should explain the special circumstances that require a deviation from the Code of Ethics.
        \item The authors should make sure to preserve anonymity (e.g., if there is a special consideration due to laws or regulations in their jurisdiction).
    \end{itemize}

\item {\bf Broader impacts}
    \item[] Question: Does the paper discuss both potential positive societal impacts and negative societal impacts of the work performed?
    \item[] Answer: \answerNA{}
    \item[] Justification: This paper is a foundational theoretical contribution to reinforcement learning theory. It does not have direct societal impacts, as it addresses fundamental algorithmic and mathematical questions about the computational tractability of learning in structured MDPs.
    \item[] Guidelines:
    \begin{itemize}
        \item The answer \answerNA{} means that there is no societal impact of the work performed.
        \item If the authors answer \answerNA{} or \answerNo, they should explain why their work has no societal impact or why the paper does not address societal impact.
        \item Examples of negative societal impacts include potential malicious or unintended uses (e.g., disinformation, generating fake profiles, surveillance), fairness considerations (e.g., deployment of technologies that could make decisions that unfairly impact specific groups), privacy considerations, and security considerations.
        \item The conference expects that many papers will be foundational research and not tied to particular applications, let alone deployments. However, if there is a direct path to any negative applications, the authors should point it out. For example, it is legitimate to point out that an improvement in the quality of generative models could be used to generate Deepfakes for disinformation. On the other hand, it is not needed to point out that a generic algorithm for optimizing neural networks could enable people to train models that generate Deepfakes faster.
        \item The authors should consider possible harms that could arise when the technology is being used as intended and functioning correctly, harms that could arise when the technology is being used as intended but gives incorrect results, and harms following from (intentional or unintentional) misuse of the technology.
        \item If there are negative societal impacts, the authors could also discuss possible mitigation strategies (e.g., gated release of models, providing defenses in addition to attacks, mechanisms for monitoring misuse, mechanisms to monitor how a system learns from feedback over time, improving the efficiency and accessibility of ML).
    \end{itemize}
    
\item {\bf Safeguards}
    \item[] Question: Does the paper describe safeguards that have been put in place for responsible release of data or models that have a high risk for misuse (e.g., pre-trained language models, image generators, or scraped datasets)?
    \item[] Answer: \answerNA{}
    \item[] Justification: This paper does not release any data, models, or code. It is a purely theoretical contribution.
    \item[] Guidelines:
    \begin{itemize}
        \item The answer \answerNA{} means that the paper poses no such risks.
        \item Released models that have a high risk for misuse or dual-use should be released with necessary safeguards to allow for controlled use of the model, for example by requiring that users adhere to usage guidelines or restrictions to access the model or implementing safety filters. 
        \item Datasets that have been scraped from the Internet could pose safety risks. The authors should describe how they avoided releasing unsafe images.
        \item We recognize that providing effective safeguards is challenging, and many papers do not require this, but we encourage authors to take this into account and make a best faith effort.
    \end{itemize}

\item {\bf Licenses for existing assets}
    \item[] Question: Are the creators or original owners of assets (e.g., code, data, models), used in the paper, properly credited and are the license and terms of use explicitly mentioned and properly respected?
    \item[] Answer: \answerNA{}
    \item[] Justification: This paper does not use any existing code, data, or model assets. All prior work is properly cited.
    \item[] Guidelines:
    \begin{itemize}
        \item The answer \answerNA{} means that the paper does not use existing assets.
        \item The authors should cite the original paper that produced the code package or dataset.
        \item The authors should state which version of the asset is used and, if possible, include a URL.
        \item The name of the license (e.g., CC-BY 4.0) should be included for each asset.
        \item For scraped data from a particular source (e.g., website), the copyright and terms of service of that source should be provided.
        \item If assets are released, the license, copyright information, and terms of use in the package should be provided. For popular datasets, \url{paperswithcode.com/datasets} has curated licenses for some datasets. Their licensing guide can help determine the license of a dataset.
        \item For existing datasets that are re-packaged, both the original license and the license of the derived asset (if it has changed) should be provided.
        \item If this information is not available online, the authors are encouraged to reach out to the asset's creators.
    \end{itemize}

\item {\bf New assets}
    \item[] Question: Are new assets introduced in the paper well documented and is the documentation provided alongside the assets?
    \item[] Answer: \answerNA{}
    \item[] Justification: This paper does not release any new assets (code, data, or models).
    \item[] Guidelines:
    \begin{itemize}
        \item The answer \answerNA{} means that the paper does not release new assets.
        \item Researchers should communicate the details of the dataset\slash code\slash model as part of their submissions via structured templates. This includes details about training, license, limitations, etc. 
        \item The paper should discuss whether and how consent was obtained from people whose asset is used.
        \item At submission time, remember to anonymize your assets (if applicable). You can either create an anonymized URL or include an anonymized zip file.
    \end{itemize}

\item {\bf Crowdsourcing and research with human subjects}
    \item[] Question: For crowdsourcing experiments and research with human subjects, does the paper include the full text of instructions given to participants and screenshots, if applicable, as well as details about compensation (if any)? 
    \item[] Answer: \answerNA{}
    \item[] Justification: This paper does not involve crowdsourcing or research with human subjects.
    \item[] Guidelines:
    \begin{itemize}
        \item The answer \answerNA{} means that the paper does not involve crowdsourcing nor research with human subjects.
        \item Including this information in the supplemental material is fine, but if the main contribution of the paper involves human subjects, then as much detail as possible should be included in the main paper. 
        \item According to the NeurIPS Code of Ethics, workers involved in data collection, curation, or other labor should be paid at least the minimum wage in the country of the data collector. 
    \end{itemize}

\item {\bf Institutional review board (IRB) approvals or equivalent for research with human subjects}
    \item[] Question: Does the paper describe potential risks incurred by study participants, whether such risks were disclosed to the subjects, and whether Institutional Review Board (IRB) approvals (or an equivalent approval/review based on the requirements of your country or institution) were obtained?
    \item[] Answer: \answerNA{}
    \item[] Justification: This paper does not involve research with human subjects.
    \item[] Guidelines:
    \begin{itemize}
        \item The answer \answerNA{} means that the paper does not involve crowdsourcing nor research with human subjects.
        \item Depending on the country in which research is conducted, IRB approval (or equivalent) may be required for any human subjects research. If you obtained IRB approval, you should clearly state this in the paper. 
        \item We recognize that the procedures for this may vary significantly between institutions and locations, and we expect authors to adhere to the NeurIPS Code of Ethics and the guidelines for their institution. 
        \item For initial submissions, do not include any information that would break anonymity (if applicable), such as the institution conducting the review.
    \end{itemize}

\item {\bf Declaration of LLM usage}
    \item[] Question: Does the paper describe the usage of LLMs if it is an important, original, or non-standard component of the core methods in this research? Note that if the LLM is used only for writing, editing, or formatting purposes and does \emph{not} impact the core methodology, scientific rigor, or originality of the research, declaration is not required.
    \item[] Answer: \answerNA{}
    \item[] Justification: LLMs were not used as a component of the core methods in this research. Any LLM usage was limited to writing and editing assistance, which does not require declaration per the NeurIPS policy.
    \item[] Guidelines:
    \begin{itemize}
        \item The answer \answerNA{} means that the core method development in this research does not involve LLMs as any important, original, or non-standard components.
        \item Please refer to our LLM policy in the NeurIPS handbook for what should or should not be described.
    \end{itemize}

\end{enumerate} }

\end{document}